%% file: paper.tex
\documentclass[final,12pt]{colt2023} 

\title[Active Coverage for PAC Reinforcement Learning]{Active Coverage for PAC Reinforcement Learning}

\usepackage{times}
\usepackage{macros}
\usepackage{pifont}

\usepackage{algorithm,algorithmic}
\usepackage{dsfont}

\newtheorem{assumption}{Assumption}

\coltauthor{%
 \Name{Aymen Al-Marjani} \Email{aymen.al$\_$marjani@ens-lyon.fr}\\
 \addr UMPA, ENS Lyon, Lyon, France
 \AND
 \Name{Andrea Tirinzoni} \Email{tirinzoni@meta.com}\\
 \addr Meta AI, Paris, France
 \AND
 \Name{Emilie Kaufmann} \Email{emilie.kaufmann@univ-lille.fr}\\
 \addr Univ. Lille, CNRS, Inria, Centrale Lille, UMR 9189 - CRIStAL, Lille, France
 }

\begin{document}

\maketitle

\begin{abstract}%
  Collecting and leveraging data with good coverage properties plays a crucial role in different aspects of reinforcement learning (RL), including reward-free exploration and offline learning. 
  However, the notion of ``good coverage'' really depends on the application at hand, as data suitable for one context may not be so for another. 
  In this paper, we formalize the problem of \emph{active coverage} in episodic Markov decision processes (MDPs), where the goal is to interact with the environment so as to fulfill given sampling requirements. 
  This framework is sufficiently flexible to specify any desired coverage property, making it applicable to any problem that involves online exploration. 
  Our main contribution is an \emph{instance-dependent} lower bound on the sample complexity of active coverage and a simple game-theoretic algorithm, \covalg{}, that nearly matches it. We then show that \covalg{} can be used as a building block to solve different PAC RL tasks. In particular, we obtain a simple algorithm for PAC reward-free exploration with an instance-dependent sample complexity that, in certain MDPs which are ``easy to explore'', is lower than the minimax one. 
  By further coupling this exploration algorithm with a new technique to do implicit eliminations in policy space, we obtain a computationally-efficient algorithm for best-policy identification whose instance-dependent sample complexity scales with gaps between policy values. 
\end{abstract}


\input{sections/introduction}

\input{sections/preliminaries.tex}

\input{sections/coverage}
\input{sections/algorithm}

\input{sections/BPI}

\input{sections/discussion_BPI}
\input{sections/conclusion}

\acks{Aymen Al-Marjani ackowledges the support of the Chaire SeqALO (ANR-20-CHIA-0020). Emilie Kaufmann acknoweldges the support of the French National Research Agency under the BOLD project (ANR-19-CE23-0026-04).
}

\bibliography{biblio_bpi}

\newpage

\appendix

\tableofcontents

\newpage

\input{appendix/app_flows.tex}
\input{appendix/app_coverage}

\input{appendix/app_related}

\input{appendix/app_ucbvi}

\input{appendix/app_concentration_final}

\input{appendix/app_RFE_final}

\input{appendix/app_benign}
\input{appendix/app_BPI}

\input{appendix/app_visitations}

\end{document}

%% file: sections/introduction.tex
\section{Introduction}

The quality of the available data, whether it is actively gathered through \emph{online} interactions with the environment or provided as a fixed \emph{offline} dataset, plays a fundamental role in characterizing the performance of any reinforcement learning \citep[RL,][]{sutton2018reinforcement} agent. An important concept to quantify such quality is \emph{coverage}, a property measuring the extent to which data spreads across the state-action space. The notion of coverage, through the so-called \emph{concentrability coefficients}, is ubiquitous in the vast literature on offline RL \citep[e.g.,][]{munos2003error,munos2008finite,massoud2009regularized,farahmand2010error,chen2019information,xie2020q,xie2021batch,jin2021pessimism,foster2022offline}. Intuitively, the better data covers the state space, the better performance one can expect from an offline RL method. Recently, \cite{xie2022role} showed that a similar phenomenon also occurs in online RL: the sole existence of a good covering data distribution implies sample-efficient online RL with non-linear function approximation, even if such a distribution is unknown and inaccessible by the agent.

While these works treat coverage as a property of some \emph{given} data or environment, a large body of literature focuses on \emph{actively} collecting good covering data. This falls under the umbrella of \emph{reward-free exploration} \citep[RFE,][]{Jin20RewardFree}, a setting where the agent interacts with an unknown environment without any reward feedback. The objective is typically to collect sufficient data to enable the computation of a near-optimal policy for any reward function provided at downstream, e.g., by planning on top of an estimated model of the environment or by running any off-the-shelf offline RL method. Many provably-efficient algorithms exist for this problem that mostly differ in their exploration strategy. Some try to gather a minimum number of samples from each reachable state \citep{Jin20RewardFree,zhang2021near}, while others adaptively optimize a reward function proportional to their uncertainty over the environment \citep{Kaufmann21RFE,Menard21RFE} or more simply a zero reward \citep{chenstatistical}. All these approaches provably guarantee that the collected data is sufficient to learn any reward function provided at test time. Another popular technique is to seek data distributions that maximize the entropy over the state-space \citep{hazan2019provably,cheung2019exploration,Zahavy2021RewardIsEnough,mutti2022importance}. Finally, there is a long recent line of empirical works focusing on RFE, where the problem is often called \emph{unsupervised RL} \citep[e.g.,][]{laskin2urlb,eysenbachdiversity,burdaexploration,yarats2021reinforcement}.

The RFE literature mostly focuses on collecting data with the \emph{specific} properties needed for the task under consideration (e.g., achieving zero-shot RL at test time). Motivated by the crucial role of coverage in RL, in this paper we treat the problem at a higher level of generality. We formulate and study the problem of \emph{active coverage} in episodic MDPs, where the goal is to interact online with the environment so as to collect data that satisfies some given coverage constraints. Following \cite{tarbouriech2021provably} who considered a similar problem in reset-free MDPs, we formalize such constraints as a set of sampling requirements that the learner must fulfill during learning. This gives our framework a high flexibility, as one can require different notions of coverage simply by changing the sampling requirements. Moreover, the applications are numerous, as any active coverage algorithm yields an exploration strategy that can be readily plugged in to tackle different problems. In our specific case, we shall see how to apply it to design PAC algorithms for both RFE and best-policy identification \citep[BPI,][]{Fiechter94,dann15PAC,Dann2019certificates,wagenmaker21IDPAC,Wagenmaker22linearMDP,tirinzoni2022NearIP,tirinzoni23optimistic}.

\paragraph{Contributions} First, we derive an \emph{instance-dependent complexity measure} for the active coverage problem as a lower bound on the number of episodes that any algorithm must play in order to fulfill the sampling requirements on an MDP. We show interesting connections with existing coverage measures, especially the concentrability coefficients used in offline RL \citep[e.g.,][]{munos2003error}.

Then, we propose \covalg, a novel approach for active coverage. \covalg is based on a simple game-theoretic view of the problem, where an RL agent tries to optimize a sequence of rewards produced by an adversary that constantly challenges it to reach uncovered states. We show that the sample complexity of \covalg scales with our complexity measure plus some lower order learning cost, hence making our approach near-optimal.

Finally, we show how active coverage can be readily applied to get PAC algorithms with \emph{instance-dependent} sample complexity for both RFE and BPI. In particular, we show that an almost plug-and-play version of \covalg solves RFE using a number of samples scaling with our \emph{instance-dependent} coverage complexity, i.e., adapting to the complexity for navigating the underlying MDP. We show that this sample complexity can be smaller than the minimax one \citep{Menard21RFE,zhang2021near}, a perhaps surprising result given the worst-case nature of the problem (i.e., the agent aims at optimizing for \emph{all} possible rewards). For BPI, we show how \covalg can be sequentially applied to estimate the value function of all policies, while gradually focusing on policies with better performance. Notably, we obtain an instance-dependent sample complexity scaling with \emph{policy gaps} \citep{tirinzoni2021fully,dann21ReturnGap} which is in line with the recent results of \cite{Wagenmaker22linearMDP} and \cite{tirinzoni2022NearIP} (the latter for the special case of deterministic MDPs). A key advantage is that our algorithm, as opposed to the one of \cite{Wagenmaker22linearMDP}, is computationally-efficient and does not need to enumerate all policies to perform explicit eliminations. This is obtained thanks to a novel scheme which instead sequentially constrains the set of state-action distributions corresponding to high-return and well-covered policies, a technique that we believe to be of broader interest. An important technical tool for both RFE and BPI is a novel concentration inequality for value functions 
\revision{(see Appendix~\ref{sec:app_concentration}).}

%% file: sections/preliminaries.tex
\section{Active Coverage and its Complexity}\label{sec:prelim}

We suppose that the learner interacts with an environment modeled as a tabular finite-horizon \emph{Markov decision process} (MDP) $\cM := (\cS, \cA, \{p_h\}_{h\in[H]}, s_1, H)$, where $\cS$ is a finite set of $S$ states, $\cA$ is a finite set of $A$ actions, $p_h : \cS\times\cA \rightarrow \cP(\cS)$\footnote{We use $\cP(\cX)$ to denote the set of probability measures over a set $\cX$.} denotes the transition function at stage $h\in[H]$, $s_1\in\cS$ is the initial state, and $H$ is the horizon. The interaction with $\cM$ proceeds through episodes of length $H$. In each episode, starting from the initial state $s_1\in\cS$, at each stage $h\in[H]$, the learner takes an action $a_h\in\cA$ based on the current state $s_h\in\cS$ and it observes a stochastic transition to a new state $s_{h+1} \sim p_h(s_h,a_h)$. \revision{We denote by $p_h(s'|s,a)$ the probability that the new state is $s'$ when selecting action $a$ in state $s$ at step $h$ of the episode.}

The actions are chosen by a (possibly stochastic) policy $\pi = \{\pi_h\}_{h\in[H]}$, i.e., a sequence of mappings $\pi_h : \cS \rightarrow \cP(\cA)$, where $\pi_h(a|s)$ denotes the probability that the learner takes action $a$ in state $s$ at stage $h$. With some abuse of notation, we shall use $\pi_h : \cS \rightarrow \cA$ to denote a deterministic policy, where $\pi_h(s)$ directly returns the action taken in state $s$ at stage $h$. We denote by $\PiS$ (resp. $\PiD$) the set of all stochastic (resp. deterministic policies).

Denoting by $\bP^\pi$ (resp. $\bE^\pi$) the probability (resp. expectation) operator induced by the execution of a policy $\pi\in\PiS$ for an episode on $\cM$, we define, for each $(h,s,a)$, $p_h^\pi(s,a) := \bP^\pi(s_h=s,a_h=a)$ and $p_h^\pi(s) := \bP^\pi(s_h=s)$. We let $\Omega := \{p^\pi : \pi\in\PiS\}$ denote the set of all valid state-action distributions. It is well known \citep[e.g.,][]{puterman1994markov} that any distribution $\rho\in\Omega$ satisfies $\rho_h \in \cP(\cS\times\cA)$ for all $h$ and $\sum_a\rho_h(s,a) = \sum_{s',a'}\rho_{h-1}(s',a')p_{h-1}(s|s',a')$ for all $s,a$ and $h > 1$.
We make the following assumption to ensure that the whole state-space can be navigated.
\begin{assumption}[Reachability]\label{asm:reachability}
    Each state $s\in\cS$ is reachable at any stage $h \in \{2,\dots,H\}$ by some policy, i.e., $\max_{\pi\in\PiS} p_h^\pi(s) > 0$.
\end{assumption}
Reachability conditions like Assumption \ref{asm:reachability} are standard in prior work. In non-episodic reset-free MDPs \citep[e.g.,][]{jaksch2010near}, the MDP is often required to be communicating to ensure learnability, i.e., any two states are reachable from each other by some policy. Assumption \ref{asm:reachability} is the analogous for episodic MDPs, where we only need reachability from the initial state. In episodic MDPs, reachability conditions have been used in different settings, including model-free learning \citep{modi2021model} and reward-free exploration \citep{zanette2020provably}.

\paragraph{Notation}

Throughout the paper, we shall use $\ind_\cX$ to denote an indicator function over some set $\cX$, i.e., $\ind_{\cX}(h,s,a) := \ind\{(h,s,a)\in\cX\}$ for all $h,s,a$. We shall hide $\cX$ whenever $\cX = [H] \times \cS \times \cA$.

\subsection{Learning problem}

The learner interacts with an MDP $\cM$ with unknown transition probabilities in order to fulfill some given \emph{sampling requirements}. In particular, it is given a \emph{target function} $c : [H] \times \cS \times \cA \rightarrow \mathbb{R}$, where $c_h(s,a)$ denotes the minimum number of samples that must be gathered from $(s,a)$ at stage $h$. In each episode of interaction $t\in\bN^*$, the learner plays a policy $\pi^t$ and observes a corresponding trajectory $\{(s_h^t,a_h^t)\}_{h\in[H]}$. Let $n_h^{t}(s,a) := \sum_{j=1}^t \ind(s_h^j=s,a_h^j=a)$ denote the number of times $(s,a)$ has been visited at stage $h$ up to episode $t$. The goal is to \emph{minimize} the number of episodes required to collect at least $c_h(s,a)$ samples from each $h,s,a$ with high probability.
\begin{definition}[$\delta$-correct ${c}$-coverage algorithm]\label{def:cov-alg}
    Fix $\delta \in (0,1)$ and a target function ${c}$. An algorithm is called $\delta$-correct ${c}$-coverage if, with probability at least $1-\delta$, it stops after interacting with $\cM$ for $\tau$ episodes and returns a dataset of transitions with visitation counts guaranteeing
    \begin{align*}\label{eq:opt-cover-stopping-2}
        \forall (h,s,a), \ n_h^\tau(s,a) \geq {c}_h(s,a).
    \end{align*}
\end{definition}

\paragraph{Examples}

While the definition of the active coverage problem gives complete freedom in choosing the target function $c$, for our applications we shall mostly be interested in two specific instances. In \emph{uniform coverage}, we have $c_h(s,a) = N \indi{(h,s,a) \in \cX}$ for some given set $\cX$ and $N\in\bN$. Intuitively, this requires collecting at least $N$ samples from each state-action-stage triplet in $\cX$, and the name suggests that the learner should explore $\cX$ as uniformly as possible. Possible applications include estimating the transition model uniformly well across the state-action space \citep{tarbouriech2020active} and discovering sparse rewards. In our applications to PAC RL, we will further explore the benefits of performing \emph{proportional coverage}, which corresponds to setting $c_h(s,a) = N \max_\pi p_h^\pi(s,a) \indi{(h,s,a) \in \cX}$
\footnote{To cope with unknown transitions, we will use an upper bound of $p_h^\pi(s,a)$ in the definition of proportional coverage.}.
This requires collecting a number of samples from each $(h,s,a)\in\cX$ that scales proportionally to its reachability. 


%% file: sections/coverage.tex
\subsection{The complexity of active coverage}\label{sec:coverage}
Minimizing the sample complexity required to solve the active coverage problem requires the learner to properly plan how to distribute its exploration throughout the state-action space, hence accounting for the complex interplay between the MDP dynamics $p$ and the target function $c$. The following theorem gives a precise characterization of the complexity of this problem.
\begin{theorem}\label{th:lb-coverage}
    For any target function ${c}$ and $\delta \in (0,1)$, the stopping time $\tau$ of any $\delta$-correct ${c}$-coverage algorithm satisfies $\bE[\tau] \geq (1-\delta)\varphi^\star({c}),$ where 
    \[\varphi^\star({c}) = \inf_{\rho \in \Omega} \max_{(s,a,h) \in \cX} \frac{{c}_h(s,a)}{\rho_h(s,a)}\;,\]
    with $\cX := \{(h,s,a) : c_h(s,a) > 0\}$.
\end{theorem}
The quantity $\varphi^\star({c})$ of Theorem \ref{th:lb-coverage} provides an \emph{instance-dependent} complexity measure for the active coverage problem. In particular, it depends on both the MDP $\cM$ through the set of valid state-action distributions $\Omega$ and on the target function $c$. It can be interpreted as follows. Imagine that a learner repeatedly plays a policy which induces a state-action distribution $\rho\in\Omega$. Then, for any $(h,s,a)$, the quantity $1/\rho_h(s,a)$ is roughly the expected number of episodes the learner takes to collect a single sample from $(h,s,a)$. This implies that $\max_{(s,a,h) \in \cX} \frac{{c}_h(s,a)}{\rho_h(s,a)}$ is roughly the expected number of episodes needed to satisfy the sampling requirements across all $(h,s,a)$ when playing distribution $\omega$. Then, the complexity measure is intuitively the minimum of this quantity across all possible state-action distributions. In other words, any distribution $\rho^\star$ attaining the minimum in $\varphi^\star(c)$ denotes an \emph{optimal} $c$-coverage distribution, i.e., generating data from $\rho^\star$ provably minimizes the time to satisfy all sampling requirements, in expectation. 

We remark that the lower bound of Theorem \ref{th:lb-coverage} holds for any $\delta$-correct algorithm, even for an oracle that knows the transition probabilities. In general, we do not believe it to be exactly matchable since (i) any algorithm must work with sample counts rather the expectations, (ii) the transition probabilities are unknown. However, $\varphi^\star(c)$ will appear as the leading order terms in our sample complexity, while these learning costs will be absorbed into lower order terms.




\subsection{Links to existing measures of coverage}

In Appendix \ref{app:coverage}, we show that $\varphi^\star({c})$ can be reformulated as a \emph{stochastic minimum flow}, a generalization of the minimum flow for directed acyclic graphs (DAGs), as used by \cite{tirinzoni2022NearIP} in deterministic MDPs, to stochastic environments. In this reformulation, $\varphi^\star(c)$ is written as a linear program seeking the minimal allocation of visits to each $(h,s,a)$ (i.e., a flow) that satisfies the sampling requirements while complying with the MDP dynamics. 

In Appendix \ref{app:flows}, we prove that the complexity $\varphi^\star(c)$ satisfies the following inequalities
\begin{align}\label{eq:cov-bounds}
    \underbrace{\max_{h}\sum_{s,a} {c}_h(s,a)}_{\text{\ding{182}}} \leq \varphi^\star({c}) \leq \underbrace{\sum_{h} \underset{\rho\in\Omega}{\inf }\max_{s,a} \frac{c_h(s,a)}{\rho_h(s,a)} }_{\text{\ding{183}}} \leq \underbrace{\sum_{h,s,a} \frac{{c}_h(s,a)}{\max_{\pi} p_h^\pi(s,a)}}_{\text{\ding{184}}}.
\end{align}
Interestingly, each of these terms relates to a complexity measure that appeared in previous works. Term \ding{182} is the complexity for covering a tree-based deterministic MDP \citep{tirinzoni2022NearIP}, perhaps the easiest MDP topology to navigate. As $\varphi^\star(c)$ reduces to the complexity of \cite{tirinzoni2022NearIP} in deterministic MDPs, we attain the equality $\varphi^\star(c) = \text{\ding{182}}$ in this specific tree structure. For a specific choice of $c$, \ding{183} can be shown to be exactly the ``gap visitation'' complexity measure introduced by \cite{wagenmaker21IDPAC} for BPI. As a component of their BPI algorithm MOCA, \cite{wagenmaker21IDPAC} introduced Learn2Explore, a strategy that learns policies to reach all states in the MDP. While it may be possible to adapt Learn2Explore for our active coverage problem, one limitation is that it learns how to reach each layer independently, and this is reflected on the fact that \ding{183} is only a loose upper bound (up to a factor $H$ larger) to the optimal complexity $\varphi^\star(c)$. Finally, \ding{184} can be related to the sample complexity for active coverage obtained by the GOSPRL algorithm of \cite{tarbouriech2021provably}\footnote{Since \cite{tarbouriech2021provably} consider reset-free MDPs, their complexity actually scales as $\sum_{s,a} D_{s,a}c(s,a)$, where $D_{s,a}$ is the minimum expected time to reach $s,a$ from any state. In episodic MDPs, the minimum expected number of episodes to reach some $(h,s,a)$ is exactly $1/\max_\pi p_h^\pi(s,a)$, hence yielding \ding{184}.}. It can be interpreted as the complexity for learning how to reach each $h,s,a$ independently, which makes it an even looser upper bound to $\varphi^\star(c)$. 


\paragraph{Concentrability and coverability} 

A definition of \emph{concentrability coefficient} for data distribution $\rho$ is $C_{\mathrm{conc}}(\rho) := \max_{s,a,h}\frac{\max_\pi p_h^\pi(s,a)}{\rho_h(s,a)}$. This plays a fundamental role in characterizing the efficiency of offline RL methods (see, e.g., \citep{chen2019information,xie2022role} and references therein). It is easy to see that $\varphi^\star(c) = \inf_{\rho\in\Omega} C_{\mathrm{conc}}(\rho)$ for the target function $c$ of proportional coverage. That is, our coverage complexity is equivalent to the minimum concentrability coefficient achievable by any distribution generated by some stochastic policy. Under a similar perspective, \cite{xie2022role} introduced the \emph{coverability coefficient} $C_{\mathrm{cov}} := \inf_{\rho_1,\dots,\rho_H \in \cP(\cX\times\cA)} \max_{s,a,h} \frac{\max_\pi p_h^\pi(s,a)}{\rho_h(s,a)}$ to characterize to what extent the best data distribution covers all policies. Noting that the infimum is taken across all probability distributions rather than valid state-action distributions, the optimal data distribution in $C_{\mathrm{cov}}$ may not be attained by the execution of any stochastic policy. This means that $C_{\mathrm{cov}}$ is not a valid complexity measure for active coverage in general, and it reduces exactly to \ding{182} for proportional coverage (see their Lemma 3), i.e., to a loose lower bound on $\varphi^\star(c)$.

\section{Active Coverage by Solving Games}

We propose \covalg (Algorithm \ref{alg:cover}), which adopts a game-based perspective inspired by the bandit literature \citep{degenne2019non}. We first observe that the complexity $\varphi^\star(c)$ can be interpreted as a zero-sum game between a learner trying to produce the best sampling distribution $\rho\in\Omega$ and an adversary trying to challenge it with the tuple $(h,s,a)$ whose sampling requirement is the hardest to meet under $\rho$. \covalg does not directly solve the game in the definition of $\varphi^\star(c)$ but rather an equivalent formulation which simplifies learning. Thanks to the minmax theorem, we can write
\begin{align*}
    \frac{1}{\varphi^\star(c)} = \sup_{\rho \in \Omega} \min_{(s,a,h)\in \cX} \frac{\rho_h(s,a)}{{c}_h(s,a)} &= \sup_{\rho \in \Omega} \inf_{\lambda \in \cP(\cX)} \sum_{(h,s,a)\in\cX}\lambda_h(s,a)\frac{\rho_h(s,a)}{{c}_h(s,a)}
    \\ &= \inf_{\lambda \in \cP(\cX)} \max_{\pi\in\PiD} \sum_{(h,s,a)\in\cX} p_h^\pi(s,a)\frac{\lambda_h(s,a)}{{c}_h(s,a)},
\end{align*}
where in the last equation we used that the inner maximization is a standard RL problem with reward function given by $\frac{\lambda_h(s,a)}{{c}_h(s,a)}\indi{(h,s,a)\in\cX}$ and its optimum is known to be attained by a deterministic policy \citep[e.g.,][]{puterman1994markov}.

\begin{algorithm}[t]
    \caption{\covalg}\label{alg:cover}
    \begin{algorithmic}[1]
    \STATE \textbf{Input:} Target function ${c}_h(s,a)$, RL algorithm $\cA^\Pi$, online learning algorithm $\cA^\lambda$, confidence parameter $\delta \in (0,1)$.
    \STATE \quad Let $\cX_0 := \cX$ and $\cX_k := \{(h,s,a) : c_h(s,a) > \cmin2^k\}$ for all $k\in\mathbb{N}^*$
    \STATE \quad Initialize counts ${n}^0_h(s,a) = 0$ for all $h,s,a$
    \STATE \quad Reset $\cA^\lambda$ on $\cP(\cX)$, set $\lambda^1_h(s,a) \leftarrow \mathds{1}((h,s,a)\in \cX)/|\cX|$ for all $h,s,a$
    \STATE \quad Initialize $k_1 \leftarrow 0$
    \STATE \quad \textbf{for} $t= 1,2,\ldots $  \textbf{do}
    \STATE \qquad Get $\pi^{t}$ from $\cA^\Pi$ given reward function $\lambda^t$ and confidence $1-\delta/2$
    \STATE \qquad Generate a trajectory $\{(s_h^t,a_h^t)\}_{h\in[H]}$ using policy $\pi^{t}$ and update counts ${n}^t$
    \STATE \qquad \textbf{if} $n_h^t(s,a) \geq c_h(s,a)$ for all $h,s,a$ \textbf{then} stop and return all sampled trajectories
    \STATE \qquad Update $k_{t+1} \leftarrow \max \{j\in\mathbb{N} : n_h^t(s,a) \geq c_h(s,a) \ \forall (h,s,a)\in{\cX} \setminus \cX_{j}\}$
    \STATE \qquad \textbf{if} $k_{t+1} \neq k_t$  \textbf{then}
    \STATE \qquad\quad Reset $\cA^\lambda$ on $\cP(\cX_{k_{t+1}})$, set $\lambda^{t+1}_h(s,a) \leftarrow \mathds{1}((h,s,a)\in \cX_{k_{t+1}})/|\cX_{k_{t+1}}|$ for all $h,s,a$
    \STATE \qquad \textbf{else}
    \STATE \qquad\quad Feed $\cA^\lambda$ with loss $\ell^t(\lambda) = \sum_{(h,s,a)\in\cX_{k_t}} \lambda_h(s,a)\mathds{1}(s_h^t = s, a_h^t=a)$, get weight $\lambda^{t+1}$
    \end{algorithmic}
\end{algorithm}

\covalg solves a variant of this minmax game that does not involve the target function $c$ directly. The idea is to cluster the state-action pairs in $\cX$ based on their sampling requirement. To this end, we define the sequence of sets $\{\cX_k\}_{k\in\bN} $ as $\cX_0 := \cX$ and $\cX_k := \{(h,s,a) : c_h(s,a) > \cmin 2^{k}\}$ for all $k\in\mathbb{N}^*$, where $\cmin = \min_{(h,s,a) \in \cX} c_{h}(s,a)\vee 1$. At each round $t\in\bN^*$, \covalg tries to solve the game $\inf_{\lambda \in \cP(\cX_{k_t})} \max_{\pi\in\PiD} \sum_{h,s,a} p_h^\pi(s,a)\lambda_h(s,a)$, where $k_t$ is the largest index such that all state-action pairs in $\cX\setminus\cX_{k_t} = \{(h,s,a)\in\cX : c_h(s,a) \leq \cmin 2^{k_t}\}$ have been already covered. Intuitively, \covalg progressively focuses on covering state-action pairs with larger sampling requirement, while ignoring those that have already been covered. The main advantage over solving the initial formulation of $\varphi^\star(c)$ is two-fold. First, the learner is allowed to play only deterministic policies, each being the solution to an RL problem. Second, in the sequence of games that we consider, the objective function is independent of the scale of $c$, which avoids undesired dependencies (e.g., on the inverse of the minimum value of $c$) when the target function is unbalanced. 

\covalg approximately solves the sequence of games above by leveraging two online learning algorithms, $\cA^\lambda$ and $\cA^\Pi$. The one for the adversary ($\cA^\lambda$) can be any method for online convex optimization on the simplex with linear losses. The one for the learner ($\cA^\Pi$) can be any regret minimizer for RL that handles reward functions changing at each round (but observed at the beginning of the round). A simple approach like UCBVI \citep{Azar17UCBVI} can be adapted to this purpose. 

The final intuition behind \covalg is quite simple: at each round $t$, the adversary produces a reward function $\lambda^t$ supported over $\cX_{k_t}$ (the current set to be covered) and the learner tries to find a good policy for maximizing it. This encourages the learner to visit uncovered state-action pairs, eventually meeting the sampling requirements.

In order to analyze the sample complexity of \covalg, we make the following assumption on the adopted online learning algorithms, which will be satisfied by our specific instance.
\revision{
\begin{assumption}[First-order regret]\label{asm:no-regret-improved}
    There exists a non-decreasing function $\cR^{\lambda}(T)$ such that, if $\cA^\lambda$ is instantiated on $\cP(\cX_k)$ for some $k$ on a sequence of linear losses $\{\ell^{t}\}_{t\geq 1}$ bounded in $[0,1]$,
    \begin{align}\label{eq:regret-lambda-asm}
        \forall T \in \N^*, \ \sum_{t=1}^{T}\ell^{t}(\lambda^{t}) - \min_{\lambda \in \Delta_{\cX_k}} \sum_{t=1}^{T} \ell^{t}(\lambda) \leq \sqrt{\cR^{\lambda}(T) \sum_{t=1}^{T} \ell^{t}(\lambda^t)} + \cR^{\lambda}(T).
    \end{align}
    There exists a non-decreasing function $\cR^{\Pi}_\delta(T)$ such that, if $\cA^\Pi$ is run with confidence $1-\delta$ on a sequence of rewards $\{\lambda^t\}_{t\geq 1}$ with $\lambda^t \in \cP(\cX)$ for all $t$, with probability $1-\delta$, for all $T \in \bN^*$,
    \begin{align}\label{eq:regret-pi-asm}
        \sum_{t=1}^{T} V_1^{\star}\left(s_1 ; \lambda^t\right) - \sum_{t=1}^{T}V_1^{\pi_t}\left(s_1 ; \lambda^{t}\right) \leq \sqrt{\cR^{\Pi}_\delta(T) \sum_{t=1}^{T} V_1^{\pi^t}\left(s_1 ; \lambda^t\right)} + \cR^{\Pi}_\delta(T),
    \end{align}
    where $V_1^\pi(s_1;\lambda) := \sum_{h,s,a}p_h^\pi(s,a)\lambda_h(s,a)$ and $V_1^\star(s_1;\lambda) := \max_\pi V_1^{\pi}\left(s_1 ; \lambda\right)$.
\end{assumption}
\begin{theorem}[Sample complexity of \covalg]\label{th:cover-sample-comp-improved}
    Under Assumption \ref{asm:reachability} and \ref{asm:no-regret-improved}, with probability at least $1-\delta$, \covalg satisfies $n_h^\tau(s,a) \geq c_h(s,a)$ for all $h,s,a$ and its stopping time $\tau$ satisfies $\tau \leq 64 m \varphi^\star(c) + T_1$, with $m:= \lceil \log_2(c_{\max}/\cmin)\rceil\vee 1$, $c_{\max} := \max_{h,s,a} c_h(s,a)$ and
    \[T_1 = \inf\left\{ T \in \N^* : \frac{T}{2} \geq m \varphi^\star(\ind_{{\cX}}) \left(  3\cR^{\Pi}_{\delta/2}(T) + 12\cR^{\lambda}(T) + 24 \log(4T/\delta)  \right) + 1 \right\}.\]
\end{theorem}
While we require both learners to have first-order regret bounds (i.e., depending on the sum of observed losses), standard $\widetilde{O}(\sqrt{T})$ bounds can also be used at the cost of a larger second-order term $T_1$ in Theorem \ref{th:cover-sample-comp-improved}, from $T_1 = \widetilde{O}(\varphi^\star(\ind_{{\cX}}))$ as in our instantiation to $T_1 = \widetilde{O}(\varphi^\star(\ind_{{\cX}})^2)$. The key step in our proof is to show that first-order regret implies convergence to the value $\varphi^\star(c)$ of the game at a rate $\widetilde{O}(1/T)$ instead of the slower $\widetilde{O}(1/\sqrt{T})$ achieved with $\widetilde{O}(\sqrt{T})$ regret. As $\varphi^\star(\ind_{{\cX}})$ depends on the inverse visitation probabilities (see Theorem \ref{th:lb-coverage}), this $\varphi^\star(\ind_{{\cX}})$ versus $\varphi^\star(\ind_{{\cX}})^2$ improvement will be crucial to avoid undesired scaling with these quantities in our applications to PAC RL.}

\subsection{Our instantiation}

For $\cA^{\lambda}$ we propose to use the weighted majority forecaster  \citep[WMF,][]{littlestone1994weighted} with variance-dependent learning rate for which, for any sequence of losses bounded in $[0,1]$, we have by Theorem 5 of \cite{CesaBianchi2005ImprovedSB} that Assumption \ref{asm:no-regret-improved} is satisfied with
\begin{align}\label{eq:regret-lambda}
    \cR^{\lambda}(T) &= 16\log(SAH).
\end{align}
For $\cA^{\Pi}$ we propose to use a variant of UCBVI \citep{Azar17UCBVI} that can cope with varying reward functions. The idea is that, since the reward function $\lambda^t$ is revealed to $\cA^{\Pi}$ at the beginning of round $t$, we can build an upper confidence bound $\overline{Q}_h^{t-1}( s , a; \lambda^t)$ to the optimal action-value function $Q_h^\star(s,a;\lambda^t)$ by estimating the transition probabilities with the data collected up to round $t-1$. Then, we play $\pi^{t}_h(s) = \argmax_{a}\overline{Q}_h^{t-1}( s , a; \lambda^t)$, the greedy policy w.r.t. $\overline{Q}_h^{t-1}$. \revision{We build the UCBs by leveraging the same ``monotonic value propagation'' trick from \cite{zhang2021reinforcement} and prove that Assumption \ref{asm:no-regret-improved} is satisfied with
\begin{align}\label{eq:regret-pi}
    \cR^{\Pi}_\delta(T) = 65536SAH^2 (\log(2SAH/\delta) + 6S)\log(T+1)^2.
\end{align}
See Appendix \ref{app:ucbvi} for details. Notably, we manage to prove a similar first-order regret bound as the one derived by \cite{Jin20RewardFree} for EULER \citep{zanette2019tighter} with a remarkably simple analysis, without using any correction factor in the bonuses, and with improved dependences on $H$ (from $H^4$ to $H^2$) and $\delta$ (from $\log(1/\delta)^3$ to $\log(1/\delta)$).}
As compared to the minimax regret rate \citep{Azar17UCBVI}, our resulting bound in \eqref{eq:regret-pi-asm} features a dependence on $S$ instead of $\sqrt{S}$ in its leading-order term. This is the cost of handling changing rewards, which prevents us from building tight UCBs as commonly done for a fixed reward function. 
Instead, we build UCBs that hold for all rewards simultaneously using techniques from reward-free exploration \citep{Menard21RFE}, a setting where an extra dependence on $S$ is unavoidable in the worst case \citep{Jin20RewardFree}. Time-varying rewards, albeit under a weaker notion of regret, have also been studied in an adversarial setting in which the reward $\lambda^{t}$ is not revealed prior to round $t$ \citep{Rosenberg2019OnlineCO}.
\revision{
\begin{corollary}[Sample complexity of \covalg with WMF and UCBVI]\label{cor:cover-instance-improved}
    With probability at least $1-\delta$, the stopping time of \covalg with \textsc{WMF} and \textsc{UCBVI} is bounded by
    \begin{align*}
        \tau \leq 64 m \varphi^\star(c) + \widetilde{O}(m\varphi^\star(\ind_{{\cX}}) SAH^2 (\log(1/\delta) + S)),
    \end{align*}
    where $m = \lceil \log_2(c_{\max}/\cmin)\rceil \vee 1$ and $\widetilde{O}$ hides poly-logarithmic factors in $S,A,H,\phi^\star(\ind_{\cX}), \log(1/\delta)$.
\end{corollary}}
%
%
The second term in the bound above can be interpreted as the cost incurred for {learning} the optimal coverage complexity $\varphi^\star(c)$ under \emph{unknown} transition probabilities $p$. 
Still, this \emph{learning cost} \revision{depends at most logarithmically} on the total sampling requirement $\|c\|_1 = \sum_{h,s,a} c_h(s,a)$. This implies that, for large $\|c\|_1$, this cost becomes negligible as compared to the first term and $\tau \leq \widetilde{O}(\varphi^\star(c))$, which matches the lower bound of Theorem \ref{th:lb-coverage} up to constant and logarithmic terms. 
\revision{We observe that if $p$ is known, by replacing UCBVI with the computation of the optimal policy w.r.t. to $\lambda^{t}$, for which $\cR^{\Pi}_{\delta/2}(T)=0$, we get a smaller additive cost $\widetilde{O}(m\varphi^\star(\ind_{{\cX}}) \log(SAH)\log(1/\delta))$ which is only due to the randomness in the collection of trajectories.}



\subsection{Comparison with prior work}



While inspired by an original game perspective which is crucial in our analysis, the actual algorithmic approach of \covalg{} has a similar flavor as existing algorithms for different exploration tasks: it runs a regret minimizer on different reward functions enforcing the visitation of uncovered states. Using WMF as the $\lambda$-learner, the reward function in round $t$ is 
\[\lambda^{t+1}_h(s,a) = \frac{\exp\left(-\xi_{t-i_t} \left(n_h^{t}(s,a) - n_h^{i_t}(s,a)\right)\right) \ind\left((h,s,a) \in \cX_{k_t}\right)}{\sum_{(h',s',a') \in \cX_{k_t}}\exp\left(-\xi_{t-i_t} \left(n_{h'}^{t}(s',a') - n_{h'}^{i_t}(s',a')\right)\right)}\;,\]
where $i_t$ is the last restart of WMF that happened before $t$ and \revision{$\xi_t$ is the variance-dependent learning rate defined by \cite{CesaBianchi2005ImprovedSB}}. Our reward function is related to the number of prior visits and smoothly evolves over time, which is in contrasts with most prior approaches that rely on rewards of the form $r_h^{\cY}(s,a)= \ind((h,s,a) \in \cY)$ for some set $\cY$, For example, GOSPRL translated to our episodic setting would use $r^{t+1}_h(s,a) = \ind\left(n_h^{t}(s,a) < c_h^{t}(s,a)\right)$.
The Learn2Explore strategy \citep{wagenmaker21IDPAC} uses a subroutine to visit $N$ times some of the state-action pairs in $\cY$: it runs EULER \citep{Zanette19Euler} on $r^{\cY}$ and restarts the algorithm with a reward function with reduced support whenever some new state-action pair has reached $N$ visits. Several algorithms for RFE \citep{Jin20RewardFree,Du21OptimalRFE} also collect data using regret minimizers on top of indicator-based rewards. 
In Appendix~\ref{app:related}, we further discuss the connections between \covalg and Frank-Wolfe approaches used in the convex RL literature.

%% file: sections/algorithm.tex
\section{Applications to PAC RL}

A strategy for RFE should collect a dataset of trajectories from which it is possible to compute a near-optimal policy for any reward function. To be robust to any possible reward in the test phase, we intuitively need to gather sufficient samples everywhere in the MDP, which we propose to do explicitly by relying on \covalg{} with proportional coverage (Section \ref{sec:RFE}). By adding some ingredients to this exploration strategy, we further obtain a new algorithm for BPI (Section \ref{sec:BPI}).

\begin{algorithm}[t]
    \caption{\rfealg{} (Proportional Coverage Exploration) }\label{alg:RF-algo-final}
    \begin{algorithmic}[1]
    \STATE \textbf{Input:}  Precision $\epsilon$, Confidence $\delta$.
    
    \STATE For each $(h,s)$, run \textsc{\visitalg}$((h,s); \frac{\epsilon}{4SH^2}, \frac{\delta}{3SH})$ to get confidence intervals $\left[\underline{W}_h(s),\overline{W}_h(s)\right]$ on $\max_\pi p_h^\pi(s)$ (see Appendix \ref{app:visitations})
    \STATE Define $\widehat{\cX}: = \{(h,s,a): \underline{W}_h(s)\geq \frac{\epsilon}{32SH^2} \}$
    \STATE Define target function $c_h^0(s,a) = \mathds{1}\big((h,s,a)\in \widehat{\cX}\big)$ for all $(h,s,a)$
    \STATE Execute $\textsc{CovGame}\big(c^0,\ \delta/6\big)$ to get a dataset ${\cD}_{0}$  of $d_0$ episodes \hfill\textsc{// Burn-in phase} 
    \STATE Initialize episode count $t_0 \leftarrow d_0$ and statistics $n_h^{0}(s,a), \widehat{p}^{0}_h(.|s,a)$ using ${\cD}_{0}$
    \FOR{$k=1,\dots$}
    {
    \STATE \textsc{// Proportional Coverage} 
    \STATE  Compute targets $c_h^{k}(s,a) := 2^{k}\overline{W}_h(s)\mathds{1}\big((h,s,a)\in \widehat{\cX}\big)$ for all $(h,s,a)$ 
    
    \vspace{-0.1cm}
    
    \STATE  Execute $\textsc{CovGame}\big(c^{k},\ \delta/6(k+1)^2 \big)$ to get dataset $\cD_{k}$ and number of episodes $d_{k}$}
    \STATE Update episode count $t_{k} \leftarrow t_{k-1}+ d_k$ and statistics $n_h^{k}(s,a),\widehat{p}^{k}_h(.|s,a)$ using $\cD_k$
    
    \STATE \textbf{if} $\sqrt{H\beta^{RF}(t_{k},\delta/3)2^{4-k}} \leq \epsilon$  \textbf{then} stop and return $\widehat{p}^{k}$
    \ENDFOR
    \end{algorithmic}
    \end{algorithm}

\subsection{Proportional Coverage Exploration (PCE)}\label{sec:RFE}

\revision{Algorithm \ref{alg:RF-algo-final} takes as input two parameters $\epsilon,\delta$ and returns an estimate of the transition probabilities $\widehat{p}$ that, with probability $1-\delta$, yields an $\epsilon$-optimal policy for any reward function bounded in $[0,1]$.}
The choice of proportional coverage is motivated by a novel \textit{ellipsoid-shaped confidence region} for the value functions of all policies under any reward. Let $\widehat{p}^{t}$ denote the maximum likelihood estimator of $p$ after observing $t$ episodes. For any reward function $r$, let $V_1^{{\pi}}(s_1 ;{r}) := \sum_{h,s,a} p_h^\pi(s,a) r_h(s,a)$ be the expected return of $\pi$, and $\widehat{V}_1^{{\pi,t}}(s_1 ;{r})$ be the same on the empirical MDP with transitions $\widehat{p}^{t}$. 
\revision{Theorem \ref{thm:new-concentration-RFE} in Appendix \ref{sec:app_concentration} gives that, with probability $1-\delta$, jointly over all episodes $t$,
\begin{equation}\label{eq:new-conf}
 \forall {r} \in [0,1]^{SAH},\  \ \forall \pi\in \Pi^D,\ \ \big| V_1^{{\pi}}(s_1 ;{r}) - \widehat{V}_1^{{\pi,t}}(s_1 ;{r}) \big| \leq  \sqrt{\betarf(t,\delta)\!\!\!\!\!\sum_{(h,s,a)\in \cX_{\varepsilon}}\frac{p_h^{\pi}(s,a)^2}{n_h^t(s,a)}} + \frac{\varepsilon}{4}, 
\end{equation}
where $\beta^{\mathrm{RF}}(t,\delta) \propto H^2 \log(1/\delta) + SH^3\log(A(1 + t))$ and $\cX_{\varepsilon}$ is a subset of triplets that are not too hard to reach: $\cX_{\varepsilon} \subseteq \{(h,s,a) : \max_{\pi} p_h^{\pi}(s,a) \geq \frac{\varepsilon}{4SH^2}\}$. 
If we gather $c_h(s,a) = \cO(H \beta^{\mathrm{RF}}(t,\delta)\sup_{\pi} p_h^\pi(s,a)/\epsilon^2)$ visits from every $(h,s,a) \in \cX_{\varepsilon}$, then the estimation error of $V^{\pi}_1(s_1;r)$ for any $\pi$ and $r$ is below $\epsilon/2$, which is sufficient to solve RFE \citep{Jin20RewardFree}. 

Yet as the visitation probabilities are unknown, neither $\cX_{\varepsilon}$ nor $c_h(s,a)$ can actually be computed. To solve this issue, we rely on an initialization phase based on the \visitalg{} subroutine (line 2 of Algorithm~\ref{alg:RF-algo-final}), described in Appendix~\ref{app:visitations}. This procedure, that is similar to the initialization phase in MOCA \citep{wagenmaker21IDPAC}, outputs for each $(h,s)$ an interval $[\underline{W}_h(s),\overline{W}_h(s)]$ to which $\max_{\pi}p_h^{\pi}(s)$ belongs with high probability using a low-order number of episodes of $\widetilde{O}(S^3AH^4/\varepsilon)$. The lower confidence bound is then used to build a set $\widehat{\cX}$ that satisfies the requirements for $\cX_{\varepsilon}$ and the upper bound is used to define the target function that is given as input to \covalg in phase $k$ of the algorithm:  $c_h^{k}(s,a) := 2^{k}\overline{W}_h(s)\mathds{1}\big((h,s,a)\in \widehat{\cX}\big)$. 


We remark that \rfealg{} is computationally-efficient as it inherits the complexity of \covalg{} and \visitalg, which both require to solve one dynamic program in every round to compute the optimistic policy used by UCBVI. We now present its theoretical properties.}

\begin{theorem}\label{thm:RFE-main-theorem}
Let $\widehat{p}$ be the estimate of the transition probabilities that \rfealg{} outputs. For any reward function ${r}$, let $\hat{\pi}_{r}$ be an optimal policy in the MDP $(\widehat{p}, {r})$. Then, \[\bP\left(\forall {r} \in [0,1]^{SAH}, |V_1^{\hat{\pi}_{ r}}(s_1 ;{r}) - V_1^\star(s_1;{r})| \leq \epsilon\right) \geq 1-\delta.\] 
\revision{Furthermore, with probability at least $1-\delta$, the total sample complexity of \rfealg{} satisfies{\small
\begin{align*}
    \tau &\leq \widetilde{\cO}\bigg(\!\big(H^3\log(1/\delta) + SH^4\big) \varphi^\star\!\bigg(\bigg[ \frac{\sup_{\pi} p_h^{\pi}(s)\ind(\sup_{\pi} p_h^{\pi}(s) \geq \frac{\varepsilon}{32SH^2})}{\epsilon^2} \bigg]_{h,s,a}\!\bigg) + \frac{S^3A^2H^5 (\log(1/\delta) + S)}{\epsilon}  \bigg),
\end{align*}}
where $\widetilde{\cO}$ hides poly-logarithmic factors in $S,A,H, \epsilon$ and $\log(1/\delta)$.}
\end{theorem}
\revision{Perhaps the most interesting feature of this bound is that in the regime of small $\epsilon$ and small $\delta$, the leading term is $H^3 \log(1/\delta) \varphi^\star\big([ \sup_{\pi} p_h^{\pi}(s)\ind(\sup_{\pi}p_h^{\pi}(s) \geq \varepsilon/(32SH^2)) ]_{h,s,a}\big)/\epsilon^2$, which can be much smaller than the $(SAH^3/\epsilon^2)\log(1/\delta)$ minimax rate \citep{Menard21RFE}. First, using the inequality~\eqref{eq:cov-bounds}, this term is always smaller than $|\{ (h,s) : \sup_{\pi} p_h^{\pi}(s) \geq \varepsilon/(32SH^2)\}|AH^3\log(1/\delta)$, which can be better than minimax in MDPs with many states that are hard to reach. Other examples of MDPs for which PCE is better than minimax in the small $\varepsilon,\delta$ regime are given in Appendix~\ref{sec:app_benign}. For any $\alpha \in [0,1)$, we notably propose a family of MDPs satisfying $\varphi^\star\big([ \sup_{\pi} p_h^{\pi}(s) ]_{h,s,a}\big) = \cO(S^\alpha AH)$, leading to an asymptotic sample complexity of order $\left(S^{\alpha}AH^{4}/\epsilon^2\right)\log(1/\delta)$. These examples suggest that, while RFE is by essence a worst-case problem, there is still hope to adapt to the ``explorability'' of the MDP. Beyond this asymptotic regime, a worst-case bound can be directly extracted from Theorem \ref{thm:RFE-main-theorem} for any $\epsilon,\delta$ by using that the $\varphi^\star$ term is at most $SAH/\epsilon^2$,
\[\tau = \widetilde{O}\left(\frac{SAH^4}{\varepsilon^2}\log(1/\delta) + \frac{S^2AH^5}{\varepsilon^2}+\frac{S^3A^2H^5}{\varepsilon}\left(\log(1/\delta)+S\right)\right)\;,\]
which is \emph{minimax optimal} up to an $H^2$ factor and low-order terms scaling in $1/\varepsilon$.}


\revision{
\begin{remark}[Reachability] Thanks to its initialization phase, PCE can be used even when Assumption~\ref{asm:reachability} is violated. All triplets that have zero probability to be reached are filtered out from the set $\widehat{\cX}$, and \covalg{} always targets reachable states. 
\end{remark}}


%% file: sections/BPI.tex

\subsection{PRINCIPLE: PRoportIoNal Coverage with Implicit PoLicy Elimination}\label{sec:BPI}

Our second use-case of \covalg yields PRINCIPLE, an algorithm for BPI. \revision{Given an unknown reward distribution $\{\nu_{h}(s,a)\}_{h,s,a}$ with support in $[0,1]$ and mean $\{r_h(s,a)\}_{h,s,a}$, an $(\epsilon,\delta)$-PAC algorithm for BPI outputs a policy $\widehat{\pi}$ such that $\bP\left(V_1^{\widehat\pi}(s_1 ; r) \geq V_1^\star(s_1 ; r) - \varepsilon\right) \geq 1 - \delta$. }

In the PCE algorithm, we sought to achieve good proportional coverage w.r.t. the set of all policies, i.e., by requiring that $n_h^k(s,a) \geq 2^k \sup_{\pi\in \Pi^D} p_h^{\pi}(s,a)$ for all $h,s,a,k$. This is due to the ``worst-case" nature of RFE, where any policy can be potentially optimal for some reward function at test time. On the contrary, the mean-reward $r$ is fixed in BPI, a property that we can leverage to perform more adaptive exploration. A natural idea, which led to tight theoretical guarantees in recent works \citep{tirinzoni2022NearIP,Wagenmaker22linearMDP}, is to eliminate policies as soon as we are confident enough that they are sub-optimal, so that the algorithm can adapt its exploration to focus on policies of higher value. Unfortunately, while \cite{tirinzoni2022NearIP} managed to achieve so in a computationally-efficient manner for deterministic MDPs, the approach of \cite{Wagenmaker22linearMDP} needs to enumerate all policies to do the same in stochastic environments, hence yielding an exponential time-memory algorithm. Our method, PRINCIPLE, achieves the same while remaining computationally efficient. Due to space constraints, we report its full pseudo-code in Appendix \ref{app:BPI}, while here we highlight its core technique.

\paragraph{Implicit policy elimination}

The key idea is to replace explicit policy eliminations by sequentially constraining the set of state-action distributions corresponding to high-reward policies. In particular, PRINCIPLE maintains, at each phase $k$, a high-probability lower bound $\underline{V}_1^{k}$ on the optimal expected return $V_1^\star(s_1;r)$ computed as
$$\underline{V}_1^{k} :=  \sup_{\substack{{\rho} \in \Omega(\widehat{p}^{k}),\\ \max\limits_{h,s,a} {\rho}_h(s,a)/n_h^{k}(s,a) \leq 2^{-k}}} \sum_{h,s,a}\rho_h(s,a) \widehat{r}^{k}_h(s,a) - \sqrt{2^{2-k}H\beta^{bpi}(t_{k},\delta/3)},$$
where \revision{$\beta^{bpi}(t,\delta) \propto H^2\log(1/\delta) + SAH^3\log\log(t)$} and \revision{$\Omega(\widehat p^{k})$ is the set of valid visitation probabilities in the empirical MDP with transition kernel $\widehat p^k$}. As common, $\underline{V}_1^{k}$ is computed by subtracting a confidence interval to the maximum expected return estimated on the empirical MDP defined by $(\widehat{p}^k,\widehat{r}^k)$. A notable exception is that we focus only on state-action distributions that are \emph{well-covered} by the current data. Then, PRINCIPLE defines a set of ``active'' state-action distributions as
$$\Omega^{k} := \bigg\{{\rho} \in \Omega(\widehat{p}^{k}):\  \sum_{h,s,a}\rho_h(s,a) \widehat{r}^{k}_h(s,a) \geq   \underline{V}_1^{k},\ \   \max\limits_{h,s,a} {\rho}_h(s,a)/n_h^{k}(s,a) \leq 2^{-k} \bigg\}.$$
Intuitively, $\rho$ is active at phase $k$ if (1) it is a valid state-action distribution in the empirical MDP with transition probabilities $\widehat{p}^k$, (2) it induces an estimated expected return $\sum_{h,s,a}\rho_h(s,a) \widehat{r}^{k}_h(s,a)$ larger than $\underline{V}_1^{k}$, and (3) it is well-covered by the current data. Then, as compared to PCE, PRINCIPLE simply replaces the quantity $\sup_{\pi\in \Pi^D} p_h^{\pi}(s,a)$ in the target function used for \covalg at phase $k$ with $\sup_{{\rho}\in \Omega^{k-1}} {\rho}_h(s,a)$, i.e., it restricts the exploration  to active state-action distributions. In our analysis, we show that, with high probability, state-action distributions corresponding to optimal policies are never eliminated from $\Omega^k$ and $\underline{V}_1^{k}$ gradually approaches $V_1^\star(s_1;r)$ from below. That is, $\Omega^k$ is dynamically pruned to contain only distributions corresponding to higher returns, hence achieving implicit eliminations of sub-optimal policies.

\paragraph{Computational complexity}

The computations of $\underline{V}_1^{k}$ and $\sup_{{\rho}\in \Omega^{k-1}} {\rho}_h(s,a)$ amount to solving standard constrained MDPs, which can be done by linear programming \citep[e.g.,][]{efroni2020exploration}. Moreover, PRINCIPLE does not store the set $\Omega^k$ but only its associated constraints, whose number is linear in $SAH$. This implies that PRINCIPLE requires polynomial (in $SAH$) time and memory. 

\paragraph{Theoretical guarantees} We prove that PRINCIPLE enjoys an instance-dependent complexity that scales with policy gaps and visitation probabilities.
\begin{theorem}\label{thm:PRINCIPLE-complexity}
PRINCIPLE is $(\varepsilon,\delta)$-PAC for BPI and, with probability $1-\delta$, \revision{it has sample complexity
\begin{align*}
    \tau &\leq \widetilde{\cO}\bigg( (H^3 \log(1/\delta) + SAH^4) \bigg[\varphi^\star\bigg(\bigg[\sup_{\pi\in \Pi} \frac{p^\pi_h(s,a)}{\max(\epsilon, \Delta(\pi))^2} \bigg]_{h,s,a} \bigg) + \frac{\varphi^\star(\mathds{1})}{\epsilon} +  \varphi^\star(\mathds{1}) \bigg] \bigg),
\end{align*}
where $\Delta(\pi) := V_1^\star(s_1 ; r) - V_1^{\pi}(s_1 ; r)$ denotes the policy gap of $\pi$, $\mathds{1}$ denotes a function equal to $1$ for all $h,s,a$, and $\widetilde{\cO}$ hides poly-logarithmic factors in $S,A,H, \epsilon, \log(1/\delta)$ and $\varphi^\star(\mathds{1})$.}
\end{theorem}

%% file: sections/discussion_bpi.tex
\revision{\paragraph{Comparison with prior work} Besides PRINCIPLE, there exist mostly two BPI algorithms with instance-dependent guarantees for MDPs with stochastic transitions: MOCA \citep{wagenmaker21IDPAC} and PEDEL \citep{Wagenmaker22linearMDP}.
 In the small $(\epsilon,\delta)$ regime, the leading term in the sample complexity of these three algorithms is of the form $\textrm{Alg}(\cM,\epsilon)\log(1/\delta)$. We carefully compare these terms in   Appendix~\ref{app:comparison_bpi}.
Notably, while $\textrm{PRINCIPLE}(\cM, \epsilon)$ and $\textrm{PEDEL}(\cM, \epsilon)$ are both expressed with policy gaps, $\textrm{MOCA}(\cM, \epsilon)$ depends on the value gaps  
$V^\star_h(s) - Q_h^\star(s,a)$. In general, value gaps are known to be worse than policy gaps \citep{dann21ReturnGap,tirinzoni2021fully} and, while there is no clear ordering between PRINCIPLE and MOCA (just like PEDEL and MOCA, see \cite{Wagenmaker22linearMDP}), we can exhibit instances in which the complexity of the former has a better scaling than that of the latter.
\begin{lemma}
    For any $\Delta \in (0,1]$, there exists an MDP $\cM$ where
    \begin{align*}
        \textrm{MOCA}(\cM, \epsilon) = \Omega\bigg(\frac{H^5SA}{\epsilon^2}\bigg) \ \ \textrm{while}\ \ \textrm{PRINCIPLE}(\cM, \epsilon) = \cO\bigg(\frac{H^4SA}{\epsilon \Delta} + \frac{H^4\log(S)\log(A)}{\epsilon^2}\bigg). 
    \end{align*}  
\end{lemma}
On the other hand, PEDEL directly minimizes the confidence interval \eqref{eq:new-conf} over all (active) policies, an objective that is always upper bounded by the complexity of proportional coverage:   
$$\min_{\rho\in\Omega} \max_{\pi\in\Pi^{D}} \sum_{s,a}\frac{ p_h^\pi(s,a)^2}{\rho_h(s,a)} \leq   \min_{\rho\in\Omega} \max_{h,s,a} \frac{\sup_{\pi}p_h^\pi(s,a)}{\rho_h(s,a)}.$$
We prove in Appendix~\ref{app:comparison_bpi} that the complexity of PEDEL is indeed smaller (up to $H$ factors) than that of PRINCIPLE. However, this objective may be intractable in general due to the maximization over all deterministic
policies. 
On the other hand, proportional coverage is sufficient (though less statistically-efficient) to estimate the value of all policies and can be done in polynomial time. 
Besides optimistic algorithms whose sample complexity features policy gaps but with an extra sub-optimal scaling in the \emph{minimal} visitation probability \citep{tirinzoni23optimistic}, this makes PRINCIPLE the first computationally efficient BPI algorithm whose sample complexity scales with policy gaps.

}

%% file: sections/conclusion.tex
\section{Conclusion}

We proposed \covalg{}, a simple algorithm that adaptively collects episodes in an MDP to explicitly gather a required number of samples $c_h(s,a)$ from each triplet $(h,s,a)$. We proved that its sample complexity scales with a new notion of optimal coverage $\phi^\star(c)$, which is \revision{an instance-dependent} lower bound on the sample complexity of \emph{any} adaptive coverage algorithm. We then illustrated the use of \covalg{} as a building block for PAC reinforcement learning algorithms. By relying on (an optimistic variant of) proportional coverage, we proposed an algorithm for reward-free exploration with an instance-dependent sample complexity bound. Further combining proportional coverage with an implicit policy elimination scheme, we obtained the first computationally efficient algorithm for best policy identification whose sample complexity scales with policy gaps. To assess the quality of these approaches, in future work we will investigate instance-dependent lower bounds on the sample complexity of PAC RL algorithms, that are currently missing in the literature.


%% file: appendix/app_flows.tex
\section{Optimal Coverage and Stochastic Minimum Flows}\label{app:flows}

In this appendix, we present an equivalent linear programming formulation of the optimal coverage problem of Section \ref{sec:coverage} that we call \emph{stochastic minimum flow}. It is a direct extension to stochastic MDPs of the minimum flows for directed acyclic graphs employed by \cite{tirinzoni2022NearIP} in deterministic MDPs.

\subsection{Stochastic minimum flows}

We define a \emph{flow} as a non-negative function $\eta : \cS \times \cA \times [H] \rightarrow [0,\infty)$ such that
\begin{align*}
\sum_{a\in\cA} \eta_{h}(s,a) &= \sum_{s'\in\cS}\sum_{a'\in\cA} p_{h-1}(s| s',a')\eta_{h-1}(s',a') \quad \forall s\in\cS, h>1,
\\ \eta_1(s,a) &= 0 \quad \forall s\in\cS \setminus \{s_1\}, a\in\cA.
\end{align*}
That is, a flow $\eta$ is an allocation of visits to each state-action-stage triplet which satisfies the \emph{navigation constraints} of the MDP. Note that the second constraint ensures that flow can only be created in the initial state $s_1$. The value of $\eta$ is the total amount of flow leaving the initial state, i.e.,
\begin{align*}
\varphi(\eta) := \sum_{a\in\cA} \eta_1(s_1,a).
\end{align*}
Let $c : \cS \times \cA \times [H] \rightarrow [0,\infty)$ be a non-negative target function. We say that a flow $\eta$ is \emph{feasible} for $c$ if
\begin{align*}
\eta_h(s,a) \geq {c}_h(s,a) \quad \forall h\in[H],s\in\cS,a\in\cA.
\end{align*}
The \emph{stochastic minimum flow} problem consists in finding a feasible flow of minimum value. It can be clearly solved as a linear program,
\begin{equation}
	\begin{aligned}
    &\underset{\eta\in\mathbb{R}^{SAH}}{\mathrm{minimize}} \sum_{a\in\cA} \eta_1(s_1,a) ,
    \\ & \text{subject to}
     \\
     & \quad \sum_{a\in\cA} \eta_{h}(s,a) = \sum_{s'\in\cS}\sum_{a'\in\cA} p_{h-1}(s| s',a')\eta_{h-1}(s',a') \quad \forall s\in\cS, h>1 ,
    \\
    & \quad \eta_1(s,a) = 0 \quad \forall s\in\cS \setminus \{s_1\}, a\in\cA,
    \\ 
    & \quad \eta_h(s,a) \geq {c}_h(s,a) \quad \forall h\in[H],s\in\cS,a\in\cA.
	\end{aligned}
	\label{eq:min_flow_problem}
	\end{equation}
We now prove that the optimal value of \eqref{eq:min_flow_problem} is equal to $\varphi^\star({c})$, the optimal coverage complexity introduced in Section \ref{sec:coverage}.

\begin{lemma}\label{lem:flow-minmax}
    If there exists a feasible flow for the target function $c$, the optimal value of \eqref{eq:min_flow_problem} is
    \begin{align*}
     \varphi^\star({c}) = \min_{\rho\in\Omega}\max_{h,s,a}\frac{{c}_h(s,a)}{\rho_h(s,a)}.
    \end{align*}
\end{lemma}
\begin{proof}
    Let us start from the linear programming formulation \eqref{eq:min_flow_problem} and perform the change of variables $\rho_h(s,a) \leftarrow \frac{\eta_h(s,a)}{Z}$ and $Z \leftarrow \sum_{s'\in\cS}\sum_{a'\in\cA}\eta_h(s',a')$ for all $h,s,a$. Note that $Z$ is the value of the original flow $\eta$ (and thus it does not depend on the stage), while $\rho_h(s,a)$ is a probability distribution over the state-action space for each $h\in[H]$. We obtain the following optimization problem (no longer a linear program due to the presence of a bilinear constraint):
    \begin{equation*}
        \begin{aligned}
        &\underset{Z\geq 0,\rho\in\mathbb{R}^{SAH}}{\mathrm{minimize}} Z,
        \\ & \text{subject to}
         \\
         & \quad \sum_{a\in\cA} \rho_{h}(s,a) = \sum_{s'\in\cS}\sum_{a'\in\cA} p_{h-1}(s| s',a')\rho_{h-1}(s',a') \quad \forall s\in\cS, h>1,
        \\
        & \quad \rho_1(s,a) = 0 \quad \forall s\in\cS \setminus \{s_1\}, a\in\cA,
        \\
        & \quad \sum_{s\in\cS}\sum_{a\in\cA} \rho_h(s,a) = 1 \quad \forall h\in[H],
        \\
        & \quad \rho_h(s,a) \geq 0 \quad  \forall h\in[H],s\in\cS,a\in\cA,
        \\ 
        & \quad Z \geq \frac{{c}_h(s,a)}{\rho_h(s,a)} \quad \forall h\in[H],s\in\cS,a\in\cA.
        \end{aligned}
    \end{equation*}
    The optimal solution for $Z$ is clearly $Z = \max_{h,s,a} \frac{c_h(s,a)}{\rho_h(s,a)}$, while the first four constraints define exactly the set of valid state-action distributions $\Omega$. This proves the statement.
\end{proof}

\begin{lemma}\label{lem:flow-linear}
    For any $\alpha,\beta \geq 0$ and target functions $c_1,c_2$, $\varphi^\star(\alpha c_1 + \beta c_2) \leq \alpha\varphi^\star( c_1) + \beta\varphi^\star(c_2)$.
\end{lemma}
\begin{proof}
    Clearly, $\varphi^\star(\alpha c_1) = \alpha\varphi^\star(c_1) $ by definition for any $\alpha \geq 0, c_1$. From the LP formulation, we note that if $\eta^\star_1$ (resp. $\eta^\star_2$) is an optimal flow for $c_1$ (resp. $c_2$), then $\eta^\star_1 + \eta^\star_2$ is a feasible flow for $c_1 + c_2$. This implies that $\varphi^\star(c_1 + c_2) \leq \varphi^\star( c_1) + \varphi^\star(c_2)$ for any $c_1,c_2$, which proves the statement.
\end{proof}

\subsection{Executing a minimum flow}

Suppose we computed a solution $\eta_h^\star(s,a)$ to the stochastic minimum flow problem \eqref{eq:min_flow_problem}, or equivalently a solution $\rho_h^\star(s,a)$ to the coverage complexity $\varphi^\star(c)$. What policy should we execute in the MDP to realize the flow? The answer comes easily from standard MDP theory \citep{puterman1994markov}: it is enough to execute a stochastic policy
\begin{align}\label{eq:min-cover-stoch-policy}
    \pi_h(a | s) = \frac{\eta_h^\star(s,a)}{\sum_{b\in\cA}\eta_h^\star(s,b)} = \frac{\rho_h^\star(s,a)}{\sum_{b\in\cA}\rho_h^\star(s,b)} \ \forall h,s,a.
\end{align}
It is then easy to prove that $\pi$ realizes the the optimal distribution $\rho_h^\star(s,a)$.
\begin{proposition}\label{prop:distr-min-cover-policy}
    Let $\pi$ be the policy defined in \eqref{eq:min-cover-stoch-policy}, then, for each $h,s,a$,
    \begin{align*}
        p_h^\pi(s,a) = \frac{\eta_h^\star(s,a)}{\sum_{s'\in\cS}\sum_{a'\in\cA} \eta_h^\star(s',a')} = \rho_h^\star(s,a).
    \end{align*}
\end{proposition}
\begin{proof}
    This is a well-known result \citep[e.g.,][]{puterman1994markov}. For completeness, let us prove it by induction. Note that $\rho^\star$ is the normalization of $\eta^\star$ by definition. Clearly, the statement holds at $h=1$ since, for all actions $a\in\cA$,
    \begin{align*}
        p_1^\pi(s_1,a) = \pi_1(a | s_1) = \frac{\rho_1^\star(s_1,a)}{\sum_b \rho_1^\star(s_1,b)} = \rho_1^\star(s_1,a),
    \end{align*}
    and $p_1^\pi(s,a)=\rho_1^\star(s,a)=0$ for all other states. Suppose the statement holds at $h-1 \geq 1$. Then,
    \begin{align*}
        p_h^\pi(s,a) 
        &= \sum_{s'\in\cS}\sum_{a'\in\cA} \underbrace{p_{h-1}^\pi(s',a')}_{= \rho_{h-1}^\star(s',a')} p_{h-1}(s | s',a')\pi_h(a | s)
        \\ &= \underbrace{\sum_{s'\in\cS}\sum_{a'\in\cA} \rho_{h-1}^\star(s',a')p_{h-1}(s | s',a')}_{= \sum_{b\in\cA} \rho_h^\star(s,b)}\frac{\rho_h^\star(s,a)}{\sum_{b\in\cA} \rho_h^\star(s,b)} = \rho_h^\star(s,a).
    \end{align*}
\end{proof}
Note that the denominator in the expression of $p_h^\pi(s,a)$ is equal to $\varphi^\star({c})$ for any $h\in[H]$. Thus, we have $p_h^\pi(s,a) = \eta_h^\star(s,a) / \varphi^\star({c})$. If we execute $\pi$ for $t = \lceil \varphi^\star({c}) \rceil$ episodes, we have that
\begin{align*}
    \bE[n_h^t(s,a)] = \frac{\lceil \varphi^\star({c}) \rceil}{ \varphi^\star({c})}\eta_h^\star(s,a) \geq \eta_h^\star(s,a) \geq c_h(s,a) \ \forall h,s,a.
\end{align*}
Hence, we realize the flow in expectation.

\subsection{Bounding the minimum flow}

We are interested in upper and lower bounding the value of the stochastic minimum flow $\varphi^\star(c)$ as a function of $c$. We start by deriving some simple (probably loose) bounds.

\begin{lemma}\label{lem:bound-flow-simple}
Suppose there exists a feasible flow for the target function ${c}$. Then,
\begin{align*}
\max_{h\in[H]}\sum_{s\in\cS}\sum_{a\in\cA} {c}_h(s,a) \leq \varphi^\star({c}) \leq \sum_{h\in[H]}\sum_{s\in\cS}\sum_{a\in\cA} \frac{{c}_h(s,a)}{\max_{\pi} p_h^\pi(s,a)}.
\end{align*}
\end{lemma}
\begin{proof}
The proof of the lower bound is trivial by noting that the value of any flow $\eta$ can be written as $\varphi(\eta)=\sum_{s\in\cS}\sum_{a\in\cA}\eta_h(s,a)$ for all $h\in[H]$ and that any optimal flow satisfies $\eta_h^\star(s,a) \geq c_h(s,a)$ for all $h,s,a$. Let us prove the upper bound.

Let us define $w_h(s,a) := \frac{{c}_h(s,a)}{\max_{\pi\in\Pi}p_h^\pi(s,a)}$, with the convention that $w_h(s,a) = 0$ if $c_h(s,a) = 0$ regardless of the value of the denominator. Note that, if $\max_{\pi\in\Pi}p_h^\pi(s,a) = 0$, then $(s,a,h)$ is unreachable and it must be that $c_h(s,a)=0$ since we assumed the minimum flow problem to be feasible. For any reachable $(s,a,h)$, let $\pi_{s,a,h}\in\argmax_{\pi\in\Pi}p_h^\pi(s,a)$. For any unreachable $(s,a,h)$, let $\pi_{s,a,h}$ be an arbitrary deterministic policy. Let us define the following mixed state-action distribution:
\begin{align*}
    \forall h,s,a : \tilde{p}_h(s,a) := \sum_{l\in[H]}\sum_{s'\in\cS}\sum_{a'\in\cA} \frac{w_l(s',a')}{Z}p_h^{\pi_{s',a',l}}(s,a),
\end{align*}
where $Z := \sum_{l\in[H]}\sum_{s'\in\cS}\sum_{a'\in\cA} w_l(s',a')$. Since this is a convex combination of state-action distributions of deterministic policies (i.e., of $\{\pi_{s,a,h}\}_{s,a}$), $\tilde p \in \Omega$ \citep{puterman1994markov}. Then,
\begin{align*}
    \varphi^\star({c}) = \min_{\rho\in\Omega}\max_{h,s,a}\frac{{c}_h(s,a)}{\rho_h(s,a)} \leq \max_{h,s,a}\frac{{c}_h(s,a)}{\tilde p_h(s,a)} 
    &\leq Z\max_{h,s,a}\frac{{c}_h(s,a)}{w_h(s,a) p_h^{\pi_{s,a,h}}(s,a)}
    \\ & = \sum_{h\in[H]}\sum_{s\in\cS}\sum_{a\in\cA} \frac{{c}_h(s,a)}{\max_{\pi} p_h^\pi(s,a)}.
\end{align*}
\end{proof}

\begin{lemma}\label{lem:bound-flow-refined}
Suppose there exists a feasible flow for the lower bound function $c$. Then,
\begin{align*}
 \varphi^\star(c) \leq \sum_{h\in[H]} \underset{\pi\in \PiS}{\inf }\max_{s\in\cS} \frac{1}{p_h^\pi(s)}\sum_{a\in\cA} c_h(s,a).
\end{align*}
\end{lemma}
\begin{proof}
    Fix any $h\in[H]$. Note that
    \begin{align*}
        \min_{\rho \in \Omega} \max_{s,a} \frac{c_h(s,a)}{\rho_h(s,a)} 
        = \min_{\rho \in \Omega} \max_{s} \frac{1}{\rho_h(s)} \min_{\pi \in \cP(A)} \frac{c_h(s,a)}{\pi(a)} = \min_{\rho \in \Omega} \max_{s} \frac{\sum_{a\in\cA}c_h(s,a)}{\rho_h(s)}.
    \end{align*}
    Now let $\rho^h$ denote any solution to this optimization problem and define the mixed distribution $\tilde{\rho} := \sum_{l=1}^H \frac{Z_l}{Z} \rho^l$, where $Z_l := \min_{\rho \in \Omega} \max_{s} \frac{\sum_{a\in\cA}c_l(s,a)}{\rho_l(s)}$ and $Z := \sum_{l=1}^H Z_l$. Then, $\tilde\rho \in \Omega$ and thus
    \begin{align*}
        \varphi^\star({c}) \leq \max_{h,s,a}\frac{{c}_h(s,a)}{\tilde \rho_h(s,a)} 
        \leq \max_h \frac{Z}{Z_h} \max_{s,a} \frac{c_h(s,a)}{\rho^h_h(s,a)}
        &= \max_h \frac{Z}{Z_h} \min_{\rho\in\Omega}\max_{s,a} \frac{c_h(s,a)}{\rho^h(s,a)}
        \\ &= \sum_{h\in[H]} \min_{\rho \in \Omega} \max_{s} \frac{\sum_{a\in\cA}c_l(s,a)}{\rho_l(s)}.
    \end{align*}
\end{proof}


\subsection{Proof of Theorem \ref{th:lb-coverage}}\label{proof:lb-coverage}

Define the coverage event $\cE_{\mathrm{cov}} = \bigg(\forall (h,s,a) \in \cX,\ n_h^\tau(s,a)\geq {c_h(s,a)} \bigg)$. We have that for any $\delta$-correct algorithm $\bP\big(\cE_{\mathrm{cov}}\big) \geq 1-\delta$. Therefore, for any triplet $(h,s,a) \in \cX$, we have that
    \begin{align}\label{ineq:coverage-alg-ineq}
     \bE[n_h^\tau(s,a)] &\geq \bE[n_h^\tau(s,a) \indi{\cE_{\mathrm{cov}}}]
     \geq {c_h(s,a)}\bP\big(\cE_{\mathrm{cov}}\big)
     \geq (1-\delta){c_h(s,a)}.
    \end{align}
    Now consider the function $\eta_h(s,a) := \bE[n_h^\tau(s,a)]$ for all $h,s,a$. We know that $\eta$ satisfies the navigation constraints, hence it is a valid flow (see Appendix \ref{app:flows}). Moreover it satisfies the constraint (\ref{ineq:coverage-alg-ineq}). By definition of stochastic minimum flow, this means that 
    \begin{align*}
        \bE[\tau] &= \sum_{a\in \cA} \bE[n_h^\tau(s_1,a)] = \varphi(\eta) \geq \varphi^\star\left(\left[(1-\delta)c_h(s,a)\right]_{h,s,a}\right) = (1-\delta)\varphi^\star(c),
    \end{align*}
    where in the last line we used that for any constant $\alpha, \varphi^\star(\alpha c) = \alpha \varphi^\star(c)$.

    \qedblack

%% file: appendix/app_coverage.tex
\section{CovGame}\label{app:coverage}

    \subsection{Proof of Theorem \ref{th:cover-sample-comp-improved}}

    Note that, at the beginning of any round $t \geq 1$, the learner $\cA^\lambda$ works over the simplex $\cP(\cX_{k_t})$, hence $\lambda^t \in \cP(\cX_{k_t})$. Let $m$ denote the number of times $k_t$ changes value through the execution of the algorithm, that is
    $m  = \left|\{ t \leq \tau : k_t \neq k_{t+1}\}\right|$.    
    Moreover, let $\tau_0 := 1$ and, for $i\in[m]$, let $\tau_i$ be the round at the beginning of which $k_t$ has changed for the $i$-th time (i.e., $k_{\tau_i} \neq k_{\tau_i-1}$). Note that, for any $i\geq 0$ and $t\in \{\tau_i,\dots,\tau_{i+1}-1\}$, $k_t = k_{\tau_i}$. We start by bounding $m$.

\begin{lemma}\label{lem:cover-bound-m} It holds that $m \leq \lceil \log_2(c_{\max}/\cmin)\rceil \vee 1$. 
Moreover, for any $i\in\{0,\dots,m-1\}$, we have $\min_{(h,s,a) \in \cX_{k_{\tau_{i}}}} n_h^{\tau_{i+1}-1}(s,a) \leq \cmin2^{k_{\tau_i} + 2}$.
\end{lemma}

\begin{proof}
By definition of the update rule, we have that $k_{t+1} \geq k_t$ for all $t\geq 1$. Now take any time $t$ in which $k_t$ has changed value $m$ times. Since $k_1 \geq 0$, this means that $k_t \geq m$. By definition of $k_t$, we know that $n_h^{t-1}(s,a) \geq c_h(s,a)$ for all $(h,s,a)\in \cX\setminus \cX_j$ for some $j\geq m$. However, if $m \geq \lceil\log_2(c_{\max}/\cmin)\rceil \vee 1$, $\cX_j = \emptyset$ and thus the algorithm must have stopped. This prove that $m \leq \lceil\log_2(c_{\max}/\cmin)\rceil \vee 1$. 

To prove the second statement, we note that for any $i < m$, we have $k_{\tau_{i+1}-1} = k_{\tau_i}$ and $n_h^{\tau_{i+1}-2}(s,a) \geq c_h(s,a)$ for all $(h,s,a) \in \cX \setminus \cX_{k_{\tau_{i}}}$. Moreover, there must be some $(h,s,a) \in \cX \setminus \cX_{k_{\tau_{i}}+1}$ such that $n_h^{\tau_{i+1}-2}(s,a) < c_h(s,a)$. Indeed, if this was not the case, we would have an update of $k$ at the end of round $\tau_{i+1}-2$ instead of $\tau_{i+1}-1$. Since all the triplets in $\cX_{k_{\tau_{i}}}$ have been covered, the uncovered triplet must be in $\cX_{k_{\tau_{i}}} \cap \cX \setminus \cX_{k_{\tau_{i}}+1} = \cX_{k_{\tau_{i}}} \setminus \cX_{k_{\tau_{i}}+1}$. By definition, all $(h,s,a)\in \cX_{k_{\tau_{i}}} \setminus \cX_{k_{\tau_{i}}+1}$ satisfy $c_h(s,a) \leq \cmin2^{k_{\tau_{i}}+1}$. Hence, $$\min_{(h,s,a) \in \cX_{k_{\tau_{i}}}} n_h^{\tau_{i+1}-1}(s,a) \leq \min_{(h,s,a) \in \cX_{k_{\tau_{i}}}} n_h^{\tau_{i+1}-2}(s,a) + 1 < \cmin2^{k_{\tau_i}+1} + 1 \leq \cmin2^{k_{\tau_i} + 2}$$
where we use that $\cmin \geq 1$.
\end{proof}

\begin{lemma}\label{lem:cover-lb-counts-improved}
    Under Assumption \ref{asm:reachability} and \ref{asm:no-regret-improved}, with probability at least $1-\delta$, for any $i\in\{0,\dots,m-1\}$,
    \begin{align*}
        \min_{(h,s,a)\in \cX_{k_{\tau_i}}} n_h^{\tau_{i+1}-1}(s,a)  &\geq \frac{1}{8}  \sum_{j=0}^i\frac{\tau_{j+1}-\tau_j}{\varphi^\star(\ind_{\cX_{k_{\tau_j}}})} - \frac{3}{8}\cR^{\Pi}_\delta(\tau_{i+1}) - \frac{3}{2}\sum_{j=0}^i\cR^{\lambda}(\tau_{j+1}-\tau_j) - 3 \log(4\tau_{i+1}/\delta).
        \\ \min_{(h,s,a)\in \cX_{k_{\tau_i}}} n_h^{\tau_{i+1}-1}(s,a)  &\geq \frac{1}{8}  \frac{\tau_{i+1}-\tau_i}{\varphi^\star(\ind_{\cX_{k_{\tau_i}}})} - \frac{3}{8}\cR^{\Pi}_\delta(\tau_{i+1}) - \frac{3}{2}\cR^{\lambda}(\tau_{i+1}) - 3 \log(4\tau_{i+1}/\delta).
    \end{align*}
\end{lemma}
\begin{proof}
    Take any $i \in \{0,\dots,m-1\}$. Note that
\begin{align*}
    \hspace{-0.3cm}\min_{(h,s,a)\in \cX_{k_{\tau_i}}} n_h^{\tau_{i+1}-1}(s,a) 
    &= \min_{(h,s,a)\in \cX_{k_{\tau_i}}} \sum_{t=1}^{\tau_{i+1}-1} \indi{s_h^t=s,a_h^t=a}  \tag{definition of counts}
    \\ &= \min_{(h,s,a)\in \cX_{k_{\tau_i}}} \sum_{j=0}^i \sum_{t=\tau_j}^{\tau_{j+1}-1} \indi{s_h^t=s,a_h^t=a} \tag{definition of $\{\tau_j\}_{j \geq 0}$}
    \\ &\geq \sum_{j=0}^i \min_{(h,s,a)\in \cX_{k_{\tau_j}}} \sum_{t=\tau_j}^{\tau_{j+1}-1} \indi{s_h^t=s,a_h^t=a} \tag{$\cX_{k_{\tau_i}} \subseteq \cX_{k_{\tau_j}}$ for all $j\leq i$}
    \\ &= \sum_{j=0}^i \min_{\lambda\in\cP(\cX_{k_{\tau_j}})} \sum_{(h,s,a)\in\cX_{k_{\tau_j}}} \lambda_h(s,a) \sum_{t=\tau_j}^{\tau_{j+1}-1} \indi{s_h^t=s,a_h^t=a}
    \\ &= \sum_{j=0}^i \min_{\lambda\in\cP(\cX_{k_{\tau_j}})} \sum_{t=\tau_j}^{\tau_{j+1}-1} \ell^t(\lambda). \tag{definition of $\ell_t(\lambda)$}
\end{align*}
For each $j$, by the regret bound of the $\lambda$ player (Assumption \ref{asm:no-regret-improved}),
\begin{align*}
    \min_{\lambda\in\cP(\cX_{k_{\tau_j}})} \sum_{t=\tau_j}^{\tau_{j+1}-1} \ell^t(\lambda) &\geq 
    \sum_{t=\tau_j}^{\tau_{j+1}-1} \ell^t(\lambda^t) - \sqrt{\cR^{\lambda}(\tau_{j+1}-\tau_j) \sum_{t=\tau_j}^{\tau_{j+1}-1} \ell^t(\lambda^t)} - \cR^{\lambda}(\tau_{j+1}-\tau_j)
    \\ &\geq \frac{1}{2}\sum_{t=\tau_j}^{\tau_{j+1}-1} \ell^t(\lambda^t) - \frac{3}{2}\cR^{\lambda}(\tau_{j+1}-\tau_j),
\end{align*}
where in the last step we used the AM-GM inequality $\sqrt{xy} \leq \frac{x+y}{2}$ for $x,y\geq 0$.
Summing over $j$,
\begin{align}\label{eq:covgame-bound-counts-lambda-losses}
    \min_{(h,s,a)\in \cX_{k_{\tau_i}}} n_h^{\tau_{i+1}-1}(s,a)  \geq \frac{1}{2} \sum_{t=1}^{\tau_{i+1}-1} \ell^t(\lambda^t) - \frac{3}{2}\sum_{j=0}^i\cR^{\lambda}(\tau_{j+1}-\tau_j).
\end{align}
Let us now bound $\sum_{t=1}^{\tau_{i+1}-1} \ell^t(\lambda^t)$. Note that $\ell^t(\lambda^t) = \sum_{h,s,a} \lambda_h^t(s,a) \indi{s_h^t=s,a_h^t=a}$ for all for all $t\in \{\tau_j,\dots,\tau_{j+1}-1\}$ since $\lambda^t$ is equal to zero outside $\cX_{k_{\tau_j}}$. Then,
\begin{align*}
    \sum_{t=1}^{\tau_{i+1}-1} \ell^t(\lambda^t) 
    &= \sum_{t=1}^{\tau_{i+1}-1} \sum_{h,s,a} \lambda_h^t(s,a) \Big( \indi{s_h^t=s,a_h^t=a} \pm p_h^{\pi^t}(s,a)\Big)
    \\ &= \sum_{t=1}^{\tau_{i+1}-1} V_1^{\pi_t}\left(s_1 ; \lambda^t\right) +  \underbrace{\sum_{t=1}^{\tau_{i+1}-1} \sum_{h,s,a} \lambda_h^t(s,a) \Big( \indi{s_h^t=s,a_h^t=a} - p_h^{\pi^t}(s,a)\Big)}_{:= M_{\tau_{i+1}-1}}.
\end{align*}
Since both $\lambda^t$ and $\pi^t$ are $\cF_{t-1}$-measurable, $M_{\tau_{i+1}-1}$ is a martingale with differences bounded by 1 in absolute value. Therefore, by Freedman's inequality (e.g., Lemma 26 of \cite{papini2021leveraging}), with probability at least $1-\delta/2$,
\begin{align*}
    \forall T\geq 1,\quad |M_T| &\leq \sqrt{\sum_{t=1}^T V_t \times 4\log(4T/\delta)} + 4\log(4T/\delta)
    \\ &\leq \sqrt{\sum_{t=1}^T V_1^{\pi_t}\left(s_1 ; \lambda^t\right) \times 4\log(4T/\delta)} + 4\log(4T/\delta),
\end{align*}
where we defined $V_t := \mathrm{Var}[\sum_{h,s,a} \lambda_h^t(s,a)\indi{s_h^t=s,a_h^t=a} \mid \cF_{t-1}]$ and used the simple bound $V_t \leq \bE[\sum_{h,s,a} \lambda_h^t(s,a)\indi{s_h^t=s,a_h^t=a} \mid \cF_{t-1}] = V_1^{\pi_t}\left(s_1 ; \lambda^t\right)$, which holds since $\sum_{h,s,a} \lambda_h^t(s,a)\indi{s_h^t=s,a_h^t=a} \leq 1$ almost surely by definition of $\lambda^t$.
Plugging this into the initial decomposition of $\sum_{t=1}^{\tau_{i+1}-1} \ell^t(\lambda^t) $ and using the AM-GM inequality $\sqrt{xy} \leq \frac{x+y}{2}$ for $x,y\geq 0$,
\begin{align*}
    \sum_{t=1}^{\tau_{i+1}-1} \ell^t(\lambda^t) &\geq \sum_{t=1}^{\tau_{i+1}-1} V_1^{\pi_t}\left(s_1 ; \lambda^t\right) - \sqrt{\sum_{t=1}^{\tau_{i+1}-1} V_1^{\pi_t}\left(s_1 ; \lambda^t\right) \times 4\log(4\tau_{i+1}/\delta)} - 4\log(4\tau_{i+1}/\delta)
    \\ &\geq \frac{1}{2}  \sum_{t=1}^{\tau_{i+1}-1} V_1^{\pi_t}\left(s_1 ; \lambda^t\right) - 6 \log(4\tau_{i+1}/\delta).
\end{align*}
We finally bound $\sum_{t=1}^{T} V_1^{\pi_t}\left(s_1 ; \lambda^t\right)$ for any $T$. For all $T \geq 1$, with probability at least $1-\delta/2$ from Assumption \ref{asm:no-regret-improved},
\begin{align*}
    \sum_{t=1}^{T}V_1^{\pi_t}\left(s_1 ; \lambda^t\right) 
    &\geq \sum_{t=1}^{T} V_1^{\star}\left(s_1 ; \lambda^t\right) - \sqrt{\cR^{\Pi}_\delta(T)\sum_{t=1}^{T} V_1^{\star}\left(s_1 ; \lambda^t\right)} - \cR^{\Pi}_\delta(T).
\end{align*}
Applying once again the AM-GM inequality yields
\begin{align*}
    \sum_{t=1}^{T}V_1^{\pi_t}\left(s_1 ; \lambda^t\right) 
    &\geq \frac{1}{2}\sum_{t=1}^{T} V_1^{\star}\left(s_1 ; \lambda^t\right) - \frac{3}{2}\cR^{\Pi}_\delta(T)
    \\ &= \frac{1}{2}\sum_{t=1}^{T} \sup_{\rho\in\Omega} \sum_{h,s,a} \rho_h(s,a)\lambda_h^t(s,a) - \frac{3}{2}\cR^{\Pi}_\delta(T).
\end{align*}
Now note that, since $\lambda^t$ is supported on $\cX_{k_{\tau_j}}$ for any $t\in \{\tau_j,\dots,\tau_{j+1}-1\}$,
\begin{align*}
    \sum_{t=1}^{\tau_{i+1}-1} \sup_{\rho\in\Omega} \sum_{h,s,a} \rho_h(s,a)\lambda_h^t(s,a) 
    &= \sum_{j=0}^i \sum_{t=\tau_j}^{\tau_{j+1}-1} \sup_{\rho\in\Omega} \sum_{h,s,a} \rho_h(s,a)\lambda_h^t(s,a)
    \\ &\geq \sum_{j=0}^i \sum_{t=\tau_j}^{\tau_{j+1}-1}\sup_{\rho\in\Omega}  \min_{(h,s,a)\in\cX_{k_{\tau_j}}} \rho_h(s,a)
     = \sum_{j=0}^i\frac{\tau_{j+1}-\tau_j}{\varphi^\star(\ind_{\cX_{k_{\tau_j}}})}.
\end{align*}
Plugging everything together proves the first statement. 
The second result can be proved analogously by simply using $\sum_{j=0}^i \min_{\lambda\in\cP(\cX_{k_{\tau_j}})} \sum_{t=\tau_j}^{\tau_{j+1}-1} \ell^t(\lambda) \geq \min_{\lambda\in\cP(\cX_{k_{\tau_i}})} \sum_{t=\tau_i}^{\tau_{i+1}-1} \ell^t(\lambda)$ in the first series of inequalities and continuing with the same steps. This yields a single dependence on $\cR^{\lambda}(\tau_{i+1}-\tau_i)$, which can be upper bounded by the stated $\cR^{\lambda}(\tau_{i+1})$ by monotonicity of $\cR^{\lambda}$.
\end{proof}

We are now ready to prove Theorem \ref{th:cover-sample-comp-improved}
\begin{proof}[Proof of Theorem \ref{th:cover-sample-comp-improved}]
    Let $m$ be the number of times $k_t$ has changed throughout the execution of the algorithm. Note that, in the round $\tau$ in which the algorithm stops the last change must occur, thus $\tau_m = \tau+1$, and $k_{\tau+1}$ is set to any value such that $\cX_{k_{\tau+1}} = \emptyset$. Then,
\begin{align*}
    \tau = \tau_m - 1 = \sum_{i=0}^{m-1} \left( \tau_{i+1} - \tau_i \right).
\end{align*}
By combining Lemma \ref{lem:cover-bound-m} with Lemma~\ref{lem:cover-lb-counts-improved} and rearranging, with probability at least $1-\delta$, for any $i\in\{0,\dots,m-1\}$,
\begin{align*}
    \tau_{i+1} - \tau_i &\leq 8\varphi^\star(\ind_{\cX_{k_{\tau_i}}}) \cmin 2^{k_{\tau_i}+2} + 8\varphi^\star(\ind_{\cX_{k_{\tau_i}}}) \left( \frac{3}{8}\cR^{\Pi}_\delta(\tau_{i+1}) + \frac{3}{2}\cR^{\lambda}(\tau_{i+1}) + 3 \log(4\tau_{i+1}/\delta) \right)
    \\ &\leq 8\varphi^\star(\ind_{\cX_{k_{\tau_i}}}) \cmin 2^{k_{\tau_i}+2} + \varphi^\star(\ind_{\cX}) \left( 3\cR^{\Pi}_\delta(\tau_{m}) + 12\cR^{\lambda}(\tau_{m}) + 24 \log(4\tau_{m}/\delta) \right),
\end{align*}
where the second inequality is due to $\cX_k \subseteq \cX$ for all $k\in\mathbb{N}$ and $\tau_{i+1} \leq \tau_m$ for $i\leq m-1$. Then,
\begin{align*}
    \tau_m \leq 8\sum_{i=0}^{m-1} \cmin\varphi^\star(\ind_{\cX_{k_{\tau_i}}}) 2^{k_{\tau_i}+2} + m \varphi^\star(\ind_{{\cX}}) \left(  3\cR^{\Pi}_\delta(\tau_{m}) + 12\cR^{\lambda}(\tau_{m}) + 24 \log(4\tau_{m}/\delta)  \right) + 1.
\end{align*}
The first term can be bounded by
\begin{align*}
    8\sum_{i=0}^{m-1} \cmin\varphi^\star(\ind_{\cX_{k_{\tau_i}}}) 2^{k_{\tau_i}+2}
     &= 8\sum_{i=0}^{m-1} \cmin2^{k_{\tau_i}+2}\min_{\rho\in\Omega}\max_{s,a,h} \frac{\ind((h,s,a)\in\cX_{k_{\tau_i}})}{\rho_h(s,a)}
     \\ &\leq 32\sum_{i=0}^{m-1} \cmin2^{k_{\tau_i}}\min_{\rho\in\Omega}\max_{s,a,h} \frac{\ind(\cmin2^{k_{\tau_i}} < c_h(s,a))}{\rho_h(s,a)}
     \\ &\leq 32\sum_{i=0}^{m-1} \min_{\rho\in\Omega}\max_{s,a,h} \frac{c_h(s,a)}{\rho_h(s,a)} = 32m \varphi^\star(c).
\end{align*}
Plugging this into the bound on $\tau_m$, we obtain the inequality,
\begin{align*}
    \tau_m \leq  32m \varphi^\star(c) + m \varphi^\star(\ind_{{\cX}}) \left(  3\cR^{\Pi}_\delta(\tau_{m}) + 12\cR^{\lambda}(\tau_{m}) + 24 \log(4\tau_{m}/\delta)  \right) + 1.
\end{align*}
Thus, for $\tau_m \geq T_1$, we get that the sample complexity is bounded by $\tau \leq 64m \varphi^\star(c)$. Thus, we conclude that $\tau \leq \tau_m \leq \max\{ T_1, 64m \varphi^\star(c)\} \leq 64m \varphi^\star(c) + T_1$. The proof is concluded by using Lemma \ref{lem:cover-bound-m} to bound $m$.
\end{proof}

\subsection{Proof of Corollary \ref{cor:cover-instance-improved}}

We need to bound $T_1$ from Theorem \ref{th:cover-sample-comp-improved} when using WMF and UCBVI.    By definition of $T_1$ in Theorem \ref{th:cover-sample-comp-improved},
\begin{align*}
    \frac{T_1-1}{2} \leq m \varphi^\star(\ind_{{\cX}}) \left( 3\cR^{\Pi}_\delta(T_1) + 12\cR^{\lambda}(T_1) + 24 \log(4T_1/\delta) \right) + 1.
\end{align*}
Recall that, by \eqref{eq:regret-lambda} and \eqref{eq:regret-pi},
\begin{align*}
    \cR^{\lambda}(T) = 
    16\log(SAH) + 1 \quad \cR^{\Pi}_\delta(T) = 65536SAH^2 (\log(2SAH/\delta) + 6S)\log(T+1)^2.
\end{align*}
For $T \geq 3$ and assuming $SAH \geq 2$ (otherwise the result is trivial), it is easy to see that $\cR^{\lambda}(T) \leq \cR^{\Pi}_\delta(T)$ and $24 \log(4T/\delta) \leq \cR^{\Pi}_\delta(T)$. Theorefore, for some numerical constant $c_1$,
\begin{align*}
    T_1 \leq c_1 m \varphi^\star(\ind_{{\cX}}) SAH^2 (\log(2SAH/\delta) + 6S)\log(T_1+1)^2.
\end{align*}
Solving the inequality in $T_1$ yields the stated bound.


\qedblack

%% file: appendix/app_related.tex
\subsection{Links with concave-utility reinforcement learning}\label{app:related}

The (inverse) complexity term $\varphi^\star(c)$ that we seek to approximate with \covalg{} can be expressed as the maximization of a concave function of the visitation probabilities: 
\[\frac{1}{\phi^\star(c)} = \max_{\rho \in \Omega} f_{c}(\rho) \ \ \text{ where } \ \ \ f_c(\rho) = \min_{(h,s,a) \in \cX} \frac{\rho_h(s,a)}{c_h(s,a)}.\]
Computing the maximizer without the knowledge of the MDP falls in the framework of concave utility reinforcement learning (or convex reinforcement learning when we instead minimize a convex function \citep{Zahavy2021RewardIsEnough}) which has attracted a lot of interest recently \citep{hazan2019provably,Zhang20GeneralUtilities,Geist22MFG}. Several authors proposed the use of a Frank-Wolfe approach, when the function $f$ to maximize is smooth (which is not the case for $f_c$). Indeed, it was observed that in the Frank-Wolfe update the computation of
\[\argmax_{\rho \in \Omega} \rho^\top \nabla f(\rho) = \argmax_{\rho \in \Omega} \sum_{h,s,a} \rho_h(s,a) (\nabla f(\rho))_{h,s,a}\]
can be interpreted as solving the MDP when the reward function is $r_h(s,a) = (\nabla f(\rho))_{h,s,a}$. Different authors proposed to combine Frank-Wolfe with regret minimizers to cope for the unknown MDP \citep{cheung2019exploration,Zahavy2021RewardIsEnough}. For example \cite{Wagenmaker22linearMDP} propose a generic algorithm for smooth experimental design in linear MDPs (which generalizes $\max_{\rho \in \Omega} f(\rho)$ to optimizing over possible covariance matrices) which runs a regret minimizer for a long time on a reward function given by the gradient of the objective. To tackle non-smooth objective, they further propose to use a log-sum-exp smoothening trick. 

Interestingly, each phase of \covalg{} may be interpreted as doing a Frank-Wolfe update on a \emph{sequence} of smoothening of an objective of the form  $g(\rho) = \min_{(h,s,a) \in \cX_{k}} \rho_h(s,a)$, where the regret minimizer is further never restarted. Indeed, introducing 
\[g_{\eta}(\rho) = \frac{1}{\eta}\log\left( \sum_{(h,s,a) \in \cX} e^{\eta \rho_h(s,a)}\right),\]
we have 
\[(\nabla g_{\eta}(\rho))_{h,s,a} = \frac{e^{\eta \rho_h(s,a)}}{\sum_{(h',s',a') \in \cX} e^{\eta \rho_{h'}(s',a')}}\]
and the reward $\lambda^{t}_h(s,a)$ used by \covalg{} when $\cX_k$ is the set to be covered and the last restart occured at time $t_k$ can be written  
\[\lambda^{t}_h(s,a) = \nabla g_{\eta_{t - t_k}} \left( (n_h^{t}(s,a) - n_h^{t_k}(s,a))_{h,s,a}\right) = \nabla g_{\widetilde{\eta}_t} \left( \left(\frac{(n_h^{t}(s,a) - n_h^{t_k}(s,a)}{t - t_k}\right)_{h,s,a}\right)\]
\revision{where $\widetilde{\eta}_t = \xi_{t-t_k}$ is the (time-varying) smoothening parameter, with $\xi_t$ the variance-dependent learning rate defined by \cite{CesaBianchi2005ImprovedSB}. }

%% file: appendix/app_ucbvi.tex
\section{UCBVI with Changing Rewards}\label{app:ucbvi}


In this appendix, we study the following regret minimization setting with changing rewards. At the beginning of each episode $t\geq 1$, the learner receives a known reward function $r_h^t(s,a)$. The learner does not know the transition probabilities $p$ and its goal is to minimize the regret
\begin{align*}
    \sum_{t=1}^{T} \left(V_1^\star(s_1 ; {r}^{t}) - V_1^{\pi^t}(s_1 ; {r}^{t})\right),
\end{align*}
where $V_1^{\pi}(s_1 ; {r}) := \sum_{h,s,a}p_h^\pi(s,a)r_h(s,a)$ and $V_1^\star(s_1 ; {r}) := \max_\pi V_1^{\pi}(s_1 ; {r})$. We make the following assumption on the sequence of rewards.
\begin{assumption}\label{asm:rewards}
    For all $t\geq 1$, $r^t_h(s,a) \in [0,1]$ for all $h,s,a$, and $\sum_{h,s,a} r_h^t(s,a) \leq 1$.
\end{assumption}
Note that this implies that $\sum_{h=1}^H r_h(s_h,a_h) \in [0,1]$ for any trajectory $\{(s_h,a_h)\}_{h\in[H]}$ almost surely.

\subsection{Algorithm} 

We study a variant of the UCBVI algorithm \citep{Azar17UCBVI} adapted to this setting. For any $h<H$, we define recursively upper confidence bounds over optimal value functions for any reward $r$ as
$$\overline{Q}^{t}_h(s,a; r) = \left(r_h(s,a) + \widehat{P}_{h,s,a}^t \overline{V}_{h+1}^t(r) + B_h^t(s,a; r)\right) \wedge 1,$$
where $\overline{Q}^{t}_H(s,a; r) = r_H(s,a)$, $\overline{V}_{h+1}^t(s; r) := \max_a \overline{Q}^{t}_h(s,a; r)$, and 
\begin{align*}
    B_h^t(s,a; r) := \max\left\{\sqrt{\frac{8\mathbb{V}(\widehat{P}_{h,s,a}^t, \overline{V}_{h+1}^t(r))\beta(n_h^t(s,a),\delta)}{n_h^t(s,a)}}, \frac{8\beta(n_h^t(s,a),\delta)}{n_h^t(s,a)}\right\}.
\end{align*}
Note that this bonus is infinite for $n_h^{t}(s,a)=0$.
 As we will only evaluate these quantities in the rewards observed at the corresponding round, we shall abbreviate $\overline{Q}^{t}_h(s,a) := \overline{Q}^{t}_h(s,a; r^{t+1})$, $\overline{V}^{t}_h(s) := \overline{V}^{t}_h(s; r^{t+1})$, and $B_h^t(s,a) := B_h^t(s,a; r^{t+1})$ for all $t\in\bN$. UCBVI plays at each episode
\begin{align*}
    \pi^t_h(s) \in \argmax_{a} \overline{Q}^{t-1}_h(s,a),
\end{align*} 
which is thus greedy w.r.t. the optimistic value function for reward $r^t$.

\subsection{Analysis}

The analysis follows the one of EULER \citep{zanette2019tighter} and uses several technical results from \cite{Menard21RFE} and \cite{zhang2021reinforcement}. Let us define the event
\begin{align*}
    E := \left\{ \forall t\in\bN, h,s,a : \mathrm{KL}(\widehat{p}_h^t(s,a), p_h(s,a)) \leq \frac{\beta(n_h^t(s,a),\delta)}{n_h^t(s,a)}  \right\},
\end{align*}
where $\beta(n,\delta) := \log(2SAH/\delta) + S\log(8e(n+1))$. Moreover, let
\begin{align*}
    G := \left\{ \forall t\in\bN, h,s,a : n_h^t(s,a) \geq \frac{1}{2}\overline{n}_h^t(s,a) - \beta^{\mathrm{cnt}}(\delta)  \right\},
\end{align*}
where $\beta^{\mathrm{cnt}}(\delta) := \log(2SAH/\delta)$ and $\overline{n}_h^t(s,a) = \sum_{j=1}^t p_h^{\pi^j}(s,a)$.

\begin{lemma}[Bernstein-like bound]\label{lem:bernstein-value}
    Under event $E$, for all $h,s,a,t$ and value function $V$ s.t. $V_h(s) \in [0,1]$ for all $h,s$,
    \begin{align*}
        |(P_{h,s,a} - \widehat{P}_{h,s,a}^t)V_{h+1}| &\leq \sqrt{\frac{2\mathbb{V}(\widehat{P}_{h,s,a}^t, V_{h+1})\beta(n_h^t(s,a),\delta)}{n_h^t(s,a)}} + \frac{2\beta(n_h^t(s,a),\delta)}{3n_h^t(s,a)}
        \\ &\leq \max\left\{ \sqrt{\frac{8\mathbb{V}(\widehat{P}_{h,s,a}^t, V_{h+1})\beta(n_h^t(s,a),\delta)}{n_h^t(s,a)}}, \frac{4\beta(n_h^t(s,a),\delta)}{3n_h^t(s,a)}\right\}.
    \end{align*}
\end{lemma}
\begin{proof}
    This is immediate by combining the definition of $E$ with Lemma 10 of \cite{Menard21RFE} and $x+y \leq 2\max\{x,y\}$.
\end{proof}

\begin{lemma}[Optimism]\label{lem:optimism}
    Under event $E$, $\overline{Q}^{t}_h(s,a; r) \geq Q^\star_h(s,a; r)$ for all $t,h,s,a$ and any reward $r$ satisfying Assumption \ref{asm:rewards}.
\end{lemma}
\begin{proof}
    By definition, $\overline{Q}^{t}_H(s,a; r) = Q^\star_H(s,a; r)  = r_H(s,a)$. Thus, the statement holds at stage $H$. Now suppose it holds at stage $h+1$ for $h\in[H-1]$. This implies that $\overline{V}^{t}_{h+1}(s; r) \geq V^\star_{h+1}(s; r)$ for all $s$. Then, 
    \begin{align*}
        r_h & (s,a) + \widehat{P}_{h,s,a}^t \overline{V}_{h+1}^t(r) + B_h^t(s,a; r) 
        \\ &= r_h(s,a) + \widehat{P}_{h,s,a}^t \overline{V}_{h+1}^t(r) + \max\left\{\sqrt{\frac{8\mathbb{V}(\widehat{P}_{h,s,a}^t, \overline{V}_{h+1}^t(r))\beta(n_h^t(s,a),\delta)}{n_h^t(s,a)}}, \frac{8\beta(n_h^t(s,a),\delta)}{n_h^t(s,a)}\right\}
        \\ &\geq r_h(s,a) + \widehat{P}_{h,s,a}^t V^\star_{h+1}(s; r) + \max\left\{\sqrt{\frac{8\mathbb{V}(\widehat{P}_{h,s,a}^t, V_{h+1}^\star(r))\beta(n_h^t(s,a),\delta)}{n_h^t(s,a)}}, \frac{8\beta(n_h^t(s,a),\delta)}{n_h^t(s,a)}\right\}
        \\ &\geq r_h(s,a) + {P}_{h,s,a} V^\star_{h+1}(s; r) = Q_h^\star(s,a; r),
    \end{align*}
    where the first inequality uses the inductive hypothesis together with the monotonicity property in Lemma 14 of \cite{zhang2021reinforcement}, while the second inequality uses Lemma \ref{lem:bernstein-value}. The fact that $Q_h^\star(s,a; r)\in[0,1]$ for any $h,s,a$ and $r$ satisfying Assumption \ref{asm:rewards} concludes the proof.
\end{proof}

\begin{lemma}[Variance concentration]\label{lem:variance-concentration}
  Under event $E$, for any $t,h,s,a$, any reward $r$ satisfying Assumption \ref{asm:rewards}, and any value function $V$ s.t. $V_h(s) \in [0,1]$ for all $h,s$,
  \begin{align*}
    \mathbb{V}(\widehat{P}_{h,s,a}^t, \overline{V}_{h+1}^t(r)) \leq 4\mathbb{V}({P}_{h,s,a}, V_{h+1}) + 4{P}_{h,s,a}|\overline{V}_{h+1}^t(r) - V_{h+1}| + 4\frac{\beta(n_h^t(s,a),\delta)}{n_h^t(s,a)}.
  \end{align*}
\end{lemma}
\begin{proof}
  By combining Lemma 11 and 12 of \cite{Menard21RFE} together with the definition of $E$,
  \begin{align*}
    \mathbb{V}(\widehat{P}_{h,s,a}^t, \overline{V}_{h+1}^t(r)) &\leq 2\mathbb{V}(P_{h,s,a}, \overline{V}_{h+1}^t(r)) + 4\frac{\beta(n_h^t(s,a),\delta)}{n_h^t(s,a)} 
    \\ &\leq 4\mathbb{V}({P}_{h,s,a}, V_{h+1}) + 4{P}_{h,s,a}|\overline{V}_{h+1}^t(r) - V_{h+1}| + 4\frac{\beta(n_h^t(s,a),\delta)}{n_h^t(s,a)}.
  \end{align*}
\end{proof}

\begin{theorem}\label{th:ucbvi-cr}
  Under the assumptions above, with probability $1-\delta$, for any $T \in \mathbb{N}$, the regret of UCBVI for changing rewards is bounded by
  \begin{align*}
      \sum_{t=1}^{T} \left(V_1^\star(s_1 ; {r}^{t}) - V_1^{\pi^t}(s_1 ; {r}^{t})\right) \leq  5140SAH^2 L_{T,\delta} + 256\sqrt{SAH L_{T,\delta} \sum_{t=1}^T V_1^{\pi^t}(s_1; r^t)},
  \end{align*}
  whee $L_{T,\delta} := (\log(2SAH/\delta) + 6S)\log(T+1)^2$.
\end{theorem}
\begin{proof}
  Note that $\bP(E,G) \geq 1-\delta$ by Lemma 3 of \cite{Menard21RFE} and a union bound. We shall thus carry out the proof conditioned on $E$ and $G$ holding. Fix any $T \in \bN$. We start from the same regret decomposition as in the proof of Theorem 2 of \cite{zanette2019tighter}. First, by Lemma \ref{lem:optimism},
  \begin{align}\label{eq:regret-bound-optimism}
      \sum_{t=1}^{T} \left(V_1^\star(s_1 ; {r}^{t}) - V_1^{\pi^t}(s_1 ; {r}^{t})\right) 
      &\leq \sum_{t=1}^{T} \left(\overline{V}_1^{t-1}(s_1) - V_1^{\pi^t}(s_1 ; {r}^{t})\right).
  \end{align}
  For any $t,h,s,a$,
    \begin{align*}
        \overline{Q}^{t-1}_h(s,a) &- Q_h^{\pi^t}(s,a;r^t)
        \leq 
        \widehat{P}_{h,s,a}^t \overline{V}_{h+1}^{t-1} + B_h^{t-1}(s,a)\wedge 1 - {P}_{h,s,a} {V}_{h+1}^{\pi^t}(r^t)
        \\ &\leq 
        {P}_{h,s,a} \overline{V}_{h+1}^{t-1} + |(\widehat{P}_{h,s,a}^t-P_{h,s,a}) \overline{V}_{h+1}^{t-1}| + B_h^{t-1}(s,a)\wedge 1 - {P}_{h,s,a} {V}_{h+1}^{\pi^t}(r^t)
        \\ &\leq {P}_{h,s,a} \overline{V}_{h+1}^{t-1} + 2B_h^{t-1}(s,a)\wedge 1 - {P}_{h,s,a} {V}_{h+1}^{\pi^t}(r^t),
    \end{align*}
    where the last step uses Lemma \ref{lem:bernstein-value} and the fact that values are all in $[0,1]$. Theorefore,
    \begin{align}
        \overline{V}^{t-1}_h(s) - V_h^{\pi^t}(s;r^t)
         &= \overline{Q}^{t-1}_h(s,\pi_h^t(s)) - Q_h^{\pi^t}(s,\pi_h^t(s);r^t) \notag
        \\ &\leq {P}_{h,s,\pi_h^t(s)} \left(\overline{V}_{h+1}^{t-1} - {V}_{h+1}^{\pi^t}(r^t)\right) + 2B_h^{t-1}(s,\pi_h^t(s)) \wedge 1.\label{eq:V-diff-recursion}
    \end{align}
   Enrolling this reasoning, we thus obtain    
    \begin{align}
        \overline{V}^{t-1}_1(s_1) &- V_1^{\pi^t}(s_1;r^t) \leq 2\sum_{h,s,a}p_h^{\pi^t}(s,a) \left( B_h^{t-1}(s,a) \wedge 1 \right) \label{eq:bound-optimistic-V-diff}
        \\ &= 2\sum_{(h,s,a) \in \cZ_t} p_h^{\pi^t}(s,a) \left( B_h^{t-1}(s,a) \wedge 1 \right) + 2\sum_{(h,s,a) \notin \cZ_t}p_h^{\pi^t}(s,a) \left( B_h^{t-1}(s,a) \wedge 1 \right),\notag
    \end{align}
    where $\cZ_t := \{(h,s,a) : \overline{n}_h^{t-1}(s,a) \geq 4\beta^{\mathrm{cnt}}(\delta)\}$. Recall that $\overline{n}_h^{t-1}(s,a) := \sum_{i=1}^{t-1} p_h^{\pi^i}(s,a)$. Let $W_T := 2\sum_{t=1}^{T} \sum_{h,s,a}p_h^{\pi^t}(s,a) \left( B_h^{t-1}(s,a) \wedge 1 \right)$. Summing the previous inequality over all time steps and using Lemma \ref{lem:truncated-sum},
    \begin{align}\label{eq:sum-bonus-bound}
       W_T \leq 2\underbrace{\sum_{t=1}^T\sum_{(h,s,a) \in \cZ_t} p_h^{\pi^t}(s,a) \left( B_h^{t-1}(s,a) \wedge 1 \right)}_{\text{\ding{172}}} + 10SAH\beta^{\mathrm{cnt}}(\delta).
    \end{align}
    Now recall that $B_h^{t-1}(s,a)$ depends on $\mathbb{V}(\widehat{P}_{h,s,a}^{t-1}, \overline{V}_{h+1}^{t-1})$. By Lemma \ref{lem:variance-concentration},
    for any $t,s,a,h$,
    \begin{align*}
      \mathbb{V}(\widehat{P}_{h,s,a}^{t-1}, \overline{V}_{h+1}^{t-1}) \leq 4\mathbb{V}({P}_{h,s,a}, V_{h+1}^{\pi^t}(r^t)) + 4{P}_{h,s,a}|\overline{V}_{h+1}^{t-1} - V_{h+1}^{\pi^t}(r^t)| + 4\frac{\beta(n_h^{t-1}(s,a),\delta)}{n_h^{t-1}(s,a)}.
    \end{align*}
    Plugging this into the definition of $B_h^{t-1}(s,a)$ and using $\sqrt{x+y} \leq \sqrt{x} + \sqrt{y}$,
    \begin{align*}
      B_h^{t-1}(s,a) \leq \sqrt{\frac{32\mathbb{V}({P}_{h,s,a}, V_{h+1}^{\pi^t}(r^t)) \beta(n_h^{t-1}(s,a),\delta)}{n_h^{t-1}(s,a)}} &+ \sqrt{\frac{32{P}_{h,s,a}|\overline{V}_{h+1}^{t-1} - V_{h+1}^{\pi^t}(r^t)| \beta(n_h^{t-1}(s,a),\delta)}{n_h^{t-1}(s,a)}}
      \\ &+ \frac{16\beta(n_h^{t-1}(s,a),\delta)}{n_h^{t-1}(s,a)}.
    \end{align*}
    Back into \ding{172} and using that $\beta(x,\delta) \geq 1$ for all $x \geq 0$, we get
    \begin{align*}
      \text{\ding{172}}  & \leq \underbrace{\sum_{t=1}^T\sum_{(h,s,a) \in \cZ_t}p_h^{\pi^t}(s,a) \sqrt{\frac{32\mathbb{V}({P}_{h,s,a}, V_{h+1}^{\pi^t}(r^t)) \beta(n_h^{t-1}(s,a),\delta)}{n_h^{t-1}(s,a) \vee 1}}}_{\text{\ding{173}}}
      \\ & \qquad + \underbrace{\sum_{t=1}^T\sum_{(h,s,a) \in \cZ_t}p_h^{\pi^t}(s,a) \sqrt{\frac{32{P}_{h,s,a}|\overline{V}_{h+1}^{t-1} - V_{h+1}^{\pi^t}(r^t)| \beta(n_h^{t-1}(s,a),\delta)}{n_h^{t-1}(s,a) \vee 1}}}_{\text{\ding{174}}}
      \\ & \qquad + 16\underbrace{\sum_{t=1}^T\sum_{(h,s,a) \in \cZ_t}p_h^{\pi^t}(s,a) \frac{\beta(n_h^{t-1}(s,a),\delta)}{n_h^{t-1}(s,a) \vee 1}}_{\text{\ding{175}}}
    \end{align*}
    We bound these terms separately. By Lemma \ref{lem:pigeon-hole-classic} and monotonicity of $\beta(\cdot,\delta)$,
    \begin{align*}
      \text{\ding{175}} \leq 16SAH \log(T+1) \beta(T,\delta).
    \end{align*}
    By Cauchy-Schwartz inequality and the bound on \ding{175},
    \begin{align*}
      \text{\ding{173}} &\leq \sqrt{32\sum_{t=1}^T\sum_{(h,s,a) \in \cZ_t}p_h^{\pi^t}(s,a)\mathbb{V}({P}_{h,s,a}, V_{h+1}^{\pi^t}(r^t)) \sum_{t=1}^T\sum_{(h,s,a) \in \cZ_t}p_h^{\pi^t}(s,a)\frac{\beta(n_h^{t-1}(s,a),\delta)}{n_h^{t-1}(s,a) \vee 1}}
      \\ &= \sqrt{32\sum_{t=1}^T\sum_{(h,s,a) \in \cZ_t}p_h^{\pi^t}(s,a)\mathbb{V}({P}_{h,s,a}, V_{h+1}^{\pi^t}(r^t)) \times \text{\ding{175}}} 
      \\ &\leq 32\sqrt{SAH \log(T+1) \beta(T,\delta) \sum_{t=1}^T\sum_{(h,s,a) \in \cZ_t}p_h^{\pi^t}(s,a)\mathbb{V}({P}_{h,s,a}, V_{h+1}^{\pi^t}(r^t))}
      \\ &\leq 64\sqrt{SAH \log(T+1) \beta(T,\delta) \sum_{t=1}^T V_1^{\pi^t}(s_1; r^t)},
    \end{align*}
    where the last inequality uses Lemma \ref{lem:law-tv-bound}. It only remains to bound \ding{174}. By Cauchy-Schwartz inequality and the bound on \ding{175},
    \begin{align*}
      \text{\ding{174}} &\leq \sqrt{32\sum_{t=1}^T\sum_{(h,s,a) \in \cZ_t}p_h^{\pi^t}(s,a){P}_{h,s,a}|\overline{V}_{h+1}^{t-1} - V_{h+1}^{\pi^t}(r^t)| \sum_{t=1}^T\sum_{(h,s,a) \in \cZ_t}p_h^{\pi^t}(s,a)\frac{\beta(n_h^{t-1}(s,a),\delta)}{n_h^{t-1}(s,a) \vee 1}}
      \\ &= \sqrt{32\sum_{t=1}^T\sum_{(h,s,a) \in \cZ_t}p_h^{\pi^t}(s,a){P}_{h,s,a}|\overline{V}_{h+1}^{t-1} - V_{h+1}^{\pi^t}(r^t)| \times \text{\ding{175}}} 
      \\ &\leq 32\sqrt{SAH \log(T+1) \beta(T,\delta) \sum_{t=1}^T\sum_{(h,s,a) \in \cZ_t}p_h^{\pi^t}(s,a){P}_{h,s,a}|\overline{V}_{h+1}^{t-1} - V_{h+1}^{\pi^t}(r^t)|}
      \\ &\leq 32\sqrt{SAH^2 \log(T+1) \beta(T,\delta) \underbrace{2\sum_{t=1}^{T} \sum_{h,s,a}p_h^{\pi^t}(s,a) \left( B_h^{t-1}(s,a) \wedge 1 \right)}_{= W_T}},
    \end{align*}
    where the last inequality uses Lemma \ref{lem:diff-V-recursion} together with $|\overline{V}_{h+1}^{t-1}(s) - V_{h+1}^{\pi^t}(s; r^t)| = \overline{V}_{h+1}^{t-1}(s) - V_{h+1}^{\pi^t}(s; r^t)$ for any $s$ (due to optimism). Plugging the bounds on \ding{173},\ding{174},\ding{175} into \ding{172} in \eqref{eq:sum-bonus-bound},
    \begin{align*}
      W_T &\leq 64\sqrt{SAH^2 \log(T+1) \beta(T,\delta) W_T} + 128\sqrt{SAH \log(T+1) \beta(T,\delta) \sum_{t=1}^T V_1^{\pi^t}(s_1; r^t)}
      \\ & + 512 SAH \log(T+1) \beta(T,\delta) + 10SAH\beta^{\mathrm{cnt}}(\delta).
    \end{align*}
    The sum of the last two terms can be bounded by $522 SAH \log(T+1) \beta(T,\delta)$ since $\beta^{\mathrm{cnt}}(\delta) \leq \beta(T,\delta)$. Solving the quadratic inequality in $\sqrt{W_T}$, we get
    \begin{align*}
      W_T \leq 4096SAH^2 \log(T+1) \beta(T,\delta) &+ 256\sqrt{SAH \log(T+1) \beta(T,\delta) \sum_{t=1}^T V_1^{\pi^t}(s_1; r^t)}
      \\ & + 1044 SAH \log(T+1) \beta(T,\delta).
    \end{align*}
    Finally, note that $W_T$ bounds the regret by \eqref{eq:regret-bound-optimism} and \eqref{eq:bound-optimistic-V-diff}. The proof is concluded by using that $\beta(T,\delta) \leq (\log(2SAH/\delta) + 6S)\log(T+1)$ to simplify the expression.
\end{proof}

\begin{lemma}\label{lem:truncated-sum}
  For any $T \geq 1$, $\sum_{t=1}^T\sum_{(h,s,a) \notin \cZ_t}p_h^{\pi^t}(s,a) \leq 5SAH\beta^{\mathrm{cnt}}(\delta)$.
\end{lemma}
\begin{proof}
  By definition of $\cZ_t$ and since $p_h^{\pi^t}(s,a) \leq 1$,
  \begin{align*}
    \sum_{t=1}^T\sum_{(h,s,a) \notin \cZ_t}p_h^{\pi^t}(s,a) = \sum_{h,s,a}\sum_{t=1}^T p_h^{\pi^t}(s,a)\indi{\overline{n}_h^{t-1}(s,a) < 4\beta^{\mathrm{cnt}}(\delta)} \leq SAH(4\beta^{\mathrm{cnt}}(\delta) + 1).
  \end{align*}
  The result is proved by noting that $1\leq \beta^{\mathrm{cnt}}(\delta)$.
\end{proof}

\begin{lemma}\label{lem:pigeon-hole-classic}
Under event $G$, for any $T \geq 1$,
\begin{align*}
  \sum_{t=1}^T\sum_{(h,s,a) \in \cZ_t}p_h^{\pi^t}(s,a) \frac{1}{n_h^{t-1}(s,a) \vee 1} \leq 16SAH \log(T+1).
\end{align*}
\end{lemma}
\begin{proof}
  By definition of $G$ and $Z_t$, if $(h,s,a) \in \cZ_t$ then $n_h^{t-1}(s,a) \geq \overline{n}_h^{t-1}(s,a)/4$. Then,
  \begin{align*}
    \sum_{t=1}^T\sum_{(h,s,a) \in \cZ_t}p_h^{\pi^t}(s,a) \frac{1}{n_h^{t-1}(s,a) \vee 1} &\leq 4 \sum_{t=1}^T\sum_{h,s,a} p_h^{\pi^t}(s,a) \frac{1}{\overline{n}_h^{t-1}(s,a) \vee 1}
    \\ &= 4\sum_{h,s,a}\sum_{t=1}^T \frac{\overline{n}_h^t(s,a) - \overline{n}_h^{t-1}(s,a)}{\overline{n}_h^{t-1}(s,a) \vee 1} \leq 16SAH \log(T+1),
  \end{align*}
  where the last inequality uses Lemma 9 of \cite{Menard21RFE}.
\end{proof}

\begin{lemma}\label{lem:law-tv-bound}
  For any $T\geq 1$,
  \begin{align*}
    \sum_{t=1}^T\sum_{h,s,a}p_h^{\pi^t}(s,a)\mathbb{V}({P}_{h,s,a}, V_{h+1}^{\pi^t}(r^t)) \leq 4\sum_{t=1}^T V_1^{\pi^t}(s_1; r^t).
  \end{align*}
\end{lemma}
\begin{proof}
  Starting from the well-known variance decomposition lemma (see, e.g., Lemma 7 of \cite{Menard21RFE}) and following with the same bounds as in the proof of Lemma 3.4 of \cite{Jin20RewardFree},
  \begin{align*}
    \sum_{h,s,a}p_h^{\pi^t}(s,a)\mathbb{V}({P}_{h,s,a}, V_{h+1}^{\pi^t}(r^t)) &= \bE^{\pi^t}\left[\left(\sum_{h=1}^H r_h^t(s_h,a_h) - V_1^{\pi^t}(s_1; r^t) \right)^2 \right]
    \\ &\leq 2\bE^{\pi^t}\left[\left(\sum_{h=1}^H r_h^t(s_h,a_h) \right)^2 \right] + 2V_1^{\pi^t}(s_1; r^t)^2
    \\ &\leq 2\bE^{\pi^t}\left[\sum_{h=1}^H r_h^t(s_h,a_h) \right] + 2V_1^{\pi^t}(s_1; r^t)
    \\ &= 4V_1^{\pi^t}(s_1; r^t) 
  \end{align*}
  where the first inequality uses $(x+y)^2 \leq 2x^2 + 2y^2$ and the second one uses Assumption \ref{asm:rewards}.
\end{proof}

\begin{lemma}\label{lem:diff-V-recursion}
  Under event $E$, for any $t$,
  \begin{align*}
    \sum_{h,s,a}p_h^{\pi^t}(s,a) & {P}_{h,s,a}(\overline{V}_{h+1}^{t-1} - V_{h+1}^{\pi^t}(r^t)) \leq 2H \sum_{h,s,a}p_h^{\pi^t}(s,a) \left( B_h^{t-1}(s,a) \wedge 1 \right).
  \end{align*}
\end{lemma}
\begin{proof}
  Since $\sum_{s,a} p_h^{\pi^t}(s,a) p_h(s'|s,a) = p_{h+1}^{\pi^t}(s')$ for any $s'$,
  \begin{align*}
    \sum_{h,s,a}p_h^{\pi^t}(s,a) & {P}_{h,s,a}(\overline{V}_{h+1}^{t-1} - V_{h+1}^{\pi^t}(r^t))
    = \sum_{h=2}^H\sum_{s} p_h^{\pi^t}(s)(\overline{V}_{h}^{t-1}(s) - V_{h}^{\pi^t}(s; r^t)) 
    \\ &\leq \sum_{h=2}^H \sum_{s} p_h^{\pi^t}(s){P}_{h,s,\pi_h^t(s)} \left(\overline{V}_{h+1}^{t-1} - {V}_{h+1}^{\pi^t}(r^t)\right) 
    + \sum_{h=2}^H \sum_{s} p_h^{\pi^t}(s)2B_h^{t-1}(s,\pi_h^t(s)) \wedge 1
    \\ &\leq H \sum_{h,s}p_h^{\pi^t}(s) 2B_h^{t-1}(s,\pi_h^t(s)) \wedge 1,
  \end{align*}
  where the first inequality uses the decomposition in \eqref{eq:V-diff-recursion} while the second one applies this reasoning recursively.
\end{proof}

%% file: appendix/app_concentration_final.tex
\section{Concentration of Value Functions}\label{sec:app_concentration}
In this appendix, we derive the concentration bounds on value functions needed for our PAC RL algorithms. We shall assume that rewards lie in $[0,1]$ almost surely.


\subsection{General results}

\begin{lemma}\label{lem:global-concentration-pV}[Concentration of $\widehat{p}^TV$]
    Let $\cZ \subseteq [H] \times \cS \times \cA$, $Z := |\cZ|$, and $\{V_h : \cS \rightarrow [0,H]\}_{h\in[H+1]}$ be a collection of bounded functions. With probability at least $1-\delta$, for any $t \geq t_0 := \inf \{t : n_h^{t}(s,a) \geq 1, \forall (h,s,a)\in\cZ\}$,
    \begin{align*}
        \sum_{(h,s,a)\in\cZ} n_h^t(s,a) \big|(\widehat{p}_h^t(s,a) - p_h(s,a))^T {V}_{h+1}\big|^2 \leq 4H^2\log(1/\delta) + 2ZH^2\log(1+ t).
    \end{align*}
    \end{lemma}
    \begin{proof}
        We start by building a suitable stochastic process to apply Theorem 1 of \cite{abbasi2011improved}. Let $\cF_{t,h}$ denote the filtration up to stage $h$ of round $t$. For any $h\in[H],t\geq 1$, the random variable $\eta_{h}^t := V_{h+1}(s_{h+1}^t) - p_h(s_h^t,a_h^t)^T {V}_{h+1}$ is zero-mean and $H^2$-subgaussian conditionally on $\cF_{t,h}$ due to the boundedness of the functions $\{V_h\}_{h\in[H]}$. Let $X_h^t$ be a $Z$-dimensional vector containing a value $1$ at position $(h,s_h^t,a_h^t)$ if $(h,s_h^t,a_h^t)\in\cZ$, and zero at all other positions. Note that $X_h^t$ is $\cF_{t,h}$-measurable, while $\eta_h^t$ is $\cF_{t,h+1}$-measurable. Let $Y_t := \sum_{j=1}^t \sum_{h=1}^{H} X_h^t \eta_h^t$. For all $(h,s,a)\in\cZ$, we have
        \begin{align*}
            [Y_t]_{h,s,a} &= \sum_{j=1}^t \indi{s_h^j=s,a_h^j=a} \Big(V_{h+1}(s_{h+1}^j) - p_h(s_h^j,a_h^j)^T {V}_{h+1} \Big)
            \\ &= n_h^t(s,a) (\widehat{p}_h^t(s,a) - p_h(s,a))^T {V}_{h+1}.
        \end{align*}
        Let $D_t := \sum_{j=1}^t \sum_{h=1}^{H} X_h^t (X_h^t)^T = \mathrm{diag}([n_h^t(s,a)]_{(h,s,a)\in\cZ})$. Theorem 1 of \cite{abbasi2011improved} combined with Equation 20.9 from \cite{BanditBook} yield that
        \begin{align*}
            \bP \bigg(\forall t \geq 1,\ \norm{Y^t}_{(I + D_t)^{-1}}^2 \leq 2H^2\log(1/\delta) + ZH^2\log(1+ t/Z) \bigg) \geq 1-\delta. 
        \end{align*}
        Since $n_h^{t}(s,a) \geq 1$ for any $t\geq t_0$ and $(h,s,a)\in\cZ$, following Corollary 3 in \cite{reda2021dealing},
\begin{align*}
 D_t = \mathrm{diag}\big([n_h^t(s,a)]_{(h,s,a)\in\cZ}\big) \succeq (I + D_t)/2,
\end{align*}
which implies $\norm{Y^t}^2_{D_t^{-1}} \leq 2 \norm{Y^t}^2_{(I + D_t)^{-1}}$ for any $t \geq t_0$. Plugging this into the probability above and using that $\norm{Y^t}^2_{D_t^{-1}}$ is exactly the left-hand side of the statement concludes the proof.
    \end{proof}

    \begin{lemma}\label{lem:global-concentration-pV-all}[Concentration of $\widehat{p}^TV$ for all $V$]
        Let $\cZ \subseteq [H] \times \cS \times \cA$, $Z := |\cZ|$, and $\mathcal{V} := \{V : \cS \rightarrow [0,H]\}$ be the set of all bounded functions mapping $\cS$ into $[0,H]$. With probability at least $1-\delta$, for any functions $\{V_h\in \mathcal{V}\}_{h=2}^{H+1}$ and $t \geq t_0 := \inf \{t : n_h^{t}(s,a) \geq 1, \forall (h,s,a)\in\cZ\}$,
        \begin{align*}
            \sum_{(h,s,a)\in\cZ} n_h^t(s,a) \big|(\widehat{p}_h^t(s,a) - p_h(s,a))^T {V}_{h+1}\big|^2 \leq 4H^2\log(1/\delta) + 12(SH + Z)H^2\log(1+t).
        \end{align*}
        \end{lemma}
        \begin{proof}
            Let $Y_t(V_2,\dots,V_{H+1}) := \sum_{(h,s,a)\in\cZ} n_h^t(s,a) \big|(\widehat{p}_h^t(s,a) - p_h(s,a))^T {V}_{h+1}\big|^2$ denote the quantity to be bounded for fixed functions $V_h \in \mathcal{V}$ for all $2\leq h \leq H+1$. Let $\{\xi_t\}_{t\geq 1}$ be a sequence of positive values to be specified later. For all $t$, let $\Xi_t := \{\xi_t, 2\xi_t, \dots \lfloor H/\xi_t\rfloor\xi_t\}$. Note that $|\Xi_t| = \lfloor H/\xi_t\rfloor$ and, for all $x\in[0,H]$, there exists $y\in\Xi_t$ s.t. $|x-y| \leq \xi_t$. For all $t$, we build a discrete cover $\overline{\mathcal{V}}_t$ of $\mathcal{V}$ as $\overline{\mathcal{V}}_t := \{V : \cS \rightarrow [0,H] \mid \forall s : V(s) \in \Xi_t \} $. For any $t$, $\{V_h\in \mathcal{V}\}_{h=2}^{H+1}$, and $\{\overline{V}_h\in \overline{\mathcal{V}}_t\}_{h=2}^{H+1}$, using $x^2-y^2 = (x+y)(x-y)$ and abbreviating $p_h(s,a)$ and $\widehat{p}_h^t(s,a)$ respectively as $p_{h,s,a}$ and $\widehat{p}_{h,s,a}^t$,
            \begin{align*}
                \big| Y_t(V_2,&\dots,V_{H+1}) - Y_t(\overline{V}_2,\dots,\overline{V}_{H+1}) \big|\\ &= \Big|\sum_{(h,s,a)\in\cZ} n_h^t(s,a) (\widehat{p}_{h,s,a}^t - p_{h,s,a})^T ({V}_{h+1} + \overline{V}_{h+1})(\widehat{p}_{h,s,a}^t - p_{h,s,a})^T ({V}_{h+1} - \overline{V}_{h+1}) \Big|
                \\ &\leq 2H \sum_{(h,s,a)\in\cZ} n_h^t(s,a) \Big|(\widehat{p}_{h,s,a}^t - p_{h,s,a})^T ({V}_{h+1} - \overline{V}_{h+1}) \Big| 
                \\ &\leq 4Ht \|{V}_{h+1} - \overline{V}_{h+1} \|_{\infty}.
            \end{align*}
            Therefore,
            \begin{align}\label{eq:cover-min-dist}
                \min_{\{\overline{V}_h\in \overline{\mathcal{V}}_t\}_{h=2}^{H+1}} \big| Y_t(V_2,\dots,V_{H+1}) &- Y_t(\overline{V}_2,\dots,\overline{V}_{H+1}) \big| \leq 4H \xi_t t.
            \end{align}
            Now let $\alpha_{t} := 4H^2\log(1/\delta_t) + 2ZH^2\log(1+ t) + 4H\xi_t t$ for a sequence $\{\delta_t\}_t$ of values in $(0,1)$ to be defined. We have
            \begin{align*}
                &\bP \bigg(\exists t \geq t_0, \{V_h\in \mathcal{V}\}_{h=2}^{H+1} : Y_t(V_2,\dots,V_{H+1}) \geq \alpha_{t} \bigg)
                \\ & \qquad\leq \bP \bigg(\exists t \geq t_0, \{\overline{V}_h\in \overline{\mathcal{V}}_t\}_{h=2}^{H+1} : Y_t(\overline{V}_2,\dots,\overline{V}_{H+1}) \geq \alpha_{t} - 4H \xi_t t \bigg)
                \\ & \qquad\leq \sum_{t=t_0}^\infty \sum_{\{\overline{V}_h\in \overline{\mathcal{V}}_t\}_{h=2}^{H+1}} \bP \bigg( Y_t(\overline{V}_2,\dots,\overline{V}_{H+1}) \geq 4H^2\log(1/\delta_t) + 2ZH^2\log(1+ t) \bigg)
                \\ & \qquad\leq \sum_{t=t_0}^\infty \sum_{\{\overline{V}_h\in \overline{\mathcal{V}}_t\}_{h=2}^{H+1}} \delta_t = \sum_{t=t_0}^\infty \delta_t \lfloor H/\xi_t\rfloor^{SH},
            \end{align*}
            where the first inequality uses \eqref{eq:cover-min-dist}, the second one uses a union bound and the definition of $\alpha_t$, the third one uses Lemma \ref{lem:global-concentration-pV}, and the equality uses the sizes of the two sets in the sums. Setting $\xi_t = H/t$ and $\delta_t = \frac{\delta}{ 2 t^{SH+2}}$,
            \begin{align*}
                 \sum_{t=t_0}^\infty \delta_t \lfloor H/\xi_t\rfloor^{SH} \leq \frac{\delta}{2} \sum_{t=t_0}^\infty \frac{1}{t^2} \leq \delta.
            \end{align*}
            Finally, with these choices we have 
            \begin{align*}
                \alpha_t &= 4H^2\log(1/\delta) + 4H^2\log(2) + 4H^2\log(t^{SH+2}) + 2ZH^2\log(1+ t) + 4H^2
                \\ &\leq 4H^2\log(1/\delta) + 4H^2\log(2) + 12SH^3\log(t) + 2ZH^2\log(1+ t) + 4H^2
                \\ &\leq 4H^2\log(1/\delta) + 12SH^3\log(t) + 12ZH^2\log(1+ t).
            \end{align*}
            This implies the statement.
        \end{proof}

        \begin{lemma}\label{lem:global-concentration-r}[Concentration of $\widehat{r}$]
            Let $\cZ \subseteq [H] \times \cS \times \cA$ and $Z := |\cZ|$. With probability at least $1-\delta$, for any $t \geq t_0 := \inf \{t : n_h^{t}(s,a) \geq 1, \forall (h,s,a)\in\cZ\}$,
            \begin{align*}
                \sum_{(h,s,a)\in\cZ} n_h^t(s,a) \big(\widehat{r}_h^t(s,a) - r_h(s,a)\big)^2 \leq 4\log(1/\delta) + 2Z \log(1+t).
            \end{align*}
            \end{lemma}
            \begin{proof}
                Following the proof of Lemma \ref{lem:global-concentration-pV}, we build a suitable stochastic process to apply Theorem 1 of \cite{abbasi2011improved}. We define $\cF_{t,h}, X_h^t, Y_t, D_t$ exactly as in the proof of Lemma \ref{lem:global-concentration-pV}, while we redefine $\eta_{h}^t := r_h^t - r_h(s_h^t,a_h^t)$, with $r_h^t$ the random reward sample observed at stage $h$ of episode $t$. Since rewards lie in $[0,1]$ almost surely, $\eta_{h}^t$ is zero-mean and $1$-subgaussian conditionally on $\cF_{t,h}$. Moreover, it is easy to see that, for all $(h,s,a)\in\cZ$,
                \begin{align*}
                    [Y_t]_{h,s,a} = n_h^t(s,a) (\widehat{r}_h^t(s,a) - r_h(s,a)).
                \end{align*}
                Theorem 1 of \cite{abbasi2011improved} combined with Equation 20.9 from \cite{BanditBook} yield that
                \begin{align*}
                    \bP \bigg(\forall t \geq 1,\ \norm{Y^t}_{(I + D_t)^{-1}}^2 \leq 2\log(1/\delta) + Z\log(1+ t/Z) \bigg) \geq 1-\delta. 
                \end{align*}
                We can then conclude exactly as in Lemma \ref{lem:global-concentration-pV} by showing that $\norm{Y^t}^2_{D_t^{-1}} \leq 2 \norm{Y^t}^2_{(I + D_t)^{-1}}$ for any $t \geq t_0$, which implies the statement.
            \end{proof}

\subsection{Concentration results for RFE}\label{sec:app-concentration-RFE}

For reward-free exploration, it is sufficient to concentrate the values of all \emph{deterministic} policies. Our concentration result stated below features the threshold function 
\begin{align*}
    \beta^{RF}(t,\delta) := 4H^2\log(1/\delta) + 24SH^3\log(A(1+t)).
\end{align*}

\begin{theorem}\label{thm:new-concentration-RFE}
    Let $\cZ \subseteq [H] \times \cS \times \cA$ and $Z := |\cZ|$. Suppose that, for some $\epsilon_0 > 0$, $\max_\pi p_h^\pi(s,a) \leq \epsilon_0$ for all $(h,s,a)\notin \cZ$. With probability at least $1-\delta$, for any $t \geq t_0 := \inf \{t : n_h^{t}(s,a) \geq 1, \forall (h,s,a)\in\cZ\}$, $\pi\in\PiD$, and reward function $r\in [0,1]^{SAH}$,
    \begin{align*}
        \Big|\sum_{h,s,a} \big(\widehat{p}_h^{\pi,t}(s,a)- p_h^{\pi}(s,a)\big){r}_h(s,a) \Big| &\leq \sqrt{\beta^{RF}(t,\delta)\sum_{(h,s,a)\in\cZ}\frac{p_h^{\pi}(s,a)^2}{n_h^t(s,a)}} + (SH - Z_\pi)H\epsilon_0,
    \end{align*}
    where $Z_\pi := |\cZ \cap \{(h,s,\pi_h(s)) : h\in[H], s\in\cS\}|$.
\end{theorem}
\begin{proof}
    Fix any reward $r$ and deterministic policy $\pi$. Let $V_h^\pi$ and $\widehat{V}_h^{\pi,t}$ denote the value functions of $\pi$ under $(p,r)$ and $(\widehat{p}^t,r)$, respectively. By Lemma \ref{lem:simulation} and the assumption on the set $\cZ$,
\begin{align*}
    \Big|\sum_{h,s,a} \big(\widehat{p}_h^{\pi,t}(s,a) &- p_h^{\pi}(s,a)\big){r}_h(s,a) \Big|  
    \leq \sum_{h,s,a} p_h^\pi(s,a) \big|(\widehat{p}_h^t(s,a) - p_h(s,a))^T \widehat{V}_{h+1}^{\pi,t}\big|
    \\ &\leq \sum_{(h,s,a)\in\cZ} p_h^\pi(s,a) \big|(\widehat{p}_h^t(s,a) - p_h(s,a))^T \widehat{V}_{h+1}^{\pi,t}\big| + (SH - Z_\pi)H\epsilon_0.
\end{align*}
By applying Lemma \ref{lem:global-concentration-pV-all} on the set $\cZ_{\pi} = \cZ \cap \{(h,s,\pi_h(s)) : h\in[H], s\in\cS\}$, whose cardinality is at most $SH$, and union bounding over all $A^{SH}$ deterministic policies, with probability at least $1-\delta$, the following holds for all $t \geq t_0$, $\pi\in\PiD$, and value functions bounded in $[0,H]$:
\begin{align*}
    \sum_{(h,s,\pi_h(s))\in\cZ} n_h^t(s,\pi_h(s)) \big|(\widehat{p}_h^t(s,\pi_h(s)) - p_h(s,\pi_h(s)))^T {V}_{h+1}\big|^2 \leq \beta^{RF}(t,\delta).
\end{align*}
Thus, by Lemma \ref{lem:qp-concentration},
\begin{align*}
    \sum_{(h,s,a)\in\cZ} p_h^\pi(s,a) \big|(\widehat{p}_h^t(s,a) - p_h(s,a))^T \widehat{V}_{h+1}^{\pi,t}\big| &= \!\!\!\!\!\!\sum_{(s,\pi_h(s),h)\in\cZ} p_h^\pi(s) \big|(\widehat{p}_h^t(s,\pi_h(s)) - p_h(s,\pi_h(s)))^T \widehat{V}_{h+1}^{\pi,t}\big|
    \\ &\leq \!\!\!\!\!\!\!\!\!\sup_{\substack{u\in\bR^{SH},\\ \sum_{(s,\pi_h(s),h)\in\cZ} n_h^t(s,\pi_h(s)) u_{s,h}^2 \leq \beta^{RF}(t,\delta)}} \sum_{(s,\pi_h(s),h)\in\cZ} {p}_h^{\pi}(s) u_{s,h}
    \\ &= \sqrt{\beta^{RF}(t,\delta)\sum_{(h,s,a)\in\cZ}\frac{{p}_h^{\pi}(s,a)^2}{n_h^t(s,a)}}.
\end{align*}
\end{proof}

\subsection{Concentration results for BPI}

For BPI, we need concentration bounds on $\big| \widehat{V}_1^{\pi,t} - V_1^\pi \big|$ that hold uniformly across all time steps and \emph{stochastic} policies. Here $\widehat{V}_1^{\pi,t} := \sum_{h,s,a} \widehat{p}_h^{\pi,t}(s,a) \widehat{r}_h^t(s,a)$, where $\widehat{r}_h^t(s,a)$ is the MLE of $r_h(s,a)$ and $\widehat{p}_h^{\pi,t}(s,a)$ is an estimator of $p_h^\pi(s,a)$ computed from the MLEs $\{\widehat{p}_h(s'|s,a)\}_{h,s,a,s'}$ of the transition probabilities. To this end, we shall define the thresholds 
\begin{align*}
    \beta^r(t,\delta) &:= 4\log(2/\delta) + 2SAH \log(1+t),\\
    \beta^p(t,\delta) &:= 4H^2\log(2/\delta) + 24SAH^3\log(1+t),\\
    \beta^{bpi}(t,\delta)&:= 16H^2\log(2/\delta) + 96SAH^3\log(1+t).
\end{align*}
Compared to $\beta^{RF}(t,\delta)$, we note that $\beta^{bpi}(t,\delta)$ features larger multiplicative constants but also a dependency in $A$ instead of $\log(A)$ in its second term which comes from the need to concentrate the values of all {stochastic} policies.   


\begin{theorem}\label{thm:new-concentration-BPI}
    With probability at least $1-\delta$, for any $t \geq t_0 := \inf \{t : n_h^{t}(s,a) \geq 1, \forall (h,s,a)\}$ and $\pi\in\PiS$, the following holds:
    \begin{align*}
        \big| \widehat{V}_1^{\pi,t} - V_1^\pi \big| &\leq \sqrt{\beta^{bpi}(t,\delta)\min\Big(\sum_{h,s,a}\frac{p_h^{\pi}(s,a)^2}{n_h^t(s,a)}, \sum_{h,s,a}\frac{\widehat{p}_h^{\pi,t}(s,a)^2}{n_h^t(s,a)} \Big)}.
    \end{align*}
    Moreover, for any $\widetilde r \in [0,1]^{SAH}$, 
        \begin{align*}
        \Big|\sum_{h,s,a} \big(\widehat{p}_h^{\pi,t}(s,a)- p_h^{\pi}(s,a)\big){\widetilde r}_h(s,a) \Big| &\leq \sqrt{\beta^{p}(t,\delta)\sum_{h,s,a}\frac{p_h^{\pi}(s,a)^2}{n_h^t(s,a)}}\;. 
    \end{align*}
\end{theorem}


\begin{proof}
    Fix any stochastic policy $\pi$. By Lemma \ref{lem:simulation},
\begin{align*}
    \big| \widehat{V}_1^{\pi,t} - V_1^\pi \big|
    \leq \sum_{h,s,a} p_h^\pi(s,a) \big|\widehat{r}_h^t(s,a) - r_h(s,a) \big|  + \sum_{h,s,a} p_h^\pi(s,a) \big|(\widehat{p}_h^t(s,a) - p_h(s,a))^T \widehat{V}_{h+1}^{\pi,t}\big|.
\end{align*}
By applying Lemma \ref{lem:global-concentration-r} and Lemma \ref{lem:global-concentration-pV-all} for the set $\cZ=\{(h,s,a) : h\in[H], s\in\cS, a\in\cA\}$, which is of cardinality $SAH$, with probability at least $1-\delta$, the following hold for all $t \geq t_0$ and for all value functions $(V_{h})_{h \in [H]}$ supported in $[0,H]$: 
\begin{align}
    \sum_{h,s,a} n_h^t(s,a) \big|\widehat{r}_h^t(s,a) - r_h(s,a) \big|^2 &\leq \beta^{r}(t,\delta),\nonumber
    \\ \sum_{h,s,a} n_h^t(s,a) \big|(\widehat{p}_h^t(s,a) - p_h(s,a))^T {V}_{h+1}^{\pi,t}\big|^2 &\leq \beta^{p}(t,\delta).\label{eq:beta-p-ineq}
\end{align}
Thus, by Lemma \ref{lem:qp-concentration}, optimizing over the deviations as in the proof of Lemma \ref{thm:new-concentration-RFE},
\begin{align*}
    \big| \widehat{V}_1^{\pi,t} - V_1^\pi \big| \leq \sqrt{\beta^{r}(t,\delta)\sum_{h,s,a}\frac{{p}_h^{\pi}(s,a)^2}{n_h^t(s,a)}} + \sqrt{\beta^{p}(t,\delta)\sum_{h,s,a}\frac{{p}_h^{\pi}(s,a)^2}{n_h^t(s,a)}}.
\end{align*}
Using that $\beta^{r}(t,\delta) \leq \beta^{p}(t,\delta)$ and noting that $\beta^{bpi}(t,\delta)=4\beta^{p}(t,\delta)$ proves the first statement with the first term in the minimum only. To prove it with the second term as well, it is enough to use Lemma \ref{lem:simulation} with the roles of the two value functions swapped and repeat the same steps as above.

To prove the second statement, we proceed as in the proof of Theorem~\ref{thm:new-concentration-RFE} and write 
\begin{eqnarray*}
    \Big|\sum_{h,s,a} \big(\widehat{p}_h^{\pi,t}(s,a) - p_h^{\pi}(s,a)\big){\widetilde r}_h(s,a) \Big| &\leq& \sum_{h,s,a} p_h^\pi(s,a) \big|(\widehat{p}_h^t(s,a) - p_h(s,a))^T \widehat{V}_{h+1}^{\pi,t}\big| \\
    & \leq &   \!\!\!\!\!\!\!\!\!\sup_{\substack{u\in\bR^{SH},\\ \sum_{h,s,a} n_h^t(s,a) u_{h,s,a}^2 \leq \beta^{p}(t,\delta)}} \sum_{h,s,a} {p}_h^{\pi}(s,a) u_{h,s,a}
    \\ 
    &= & \sqrt{\beta^{p}(t,\delta)\sum_{h,s,a}\frac{{p}_h^{\pi}(s,a)^2}{n_h^t(s,a)}},
\end{eqnarray*}
where we used Lemma~\ref{lem:qp-concentration} and together with inequality \eqref{eq:beta-p-ineq}.
\end{proof}

\subsection{Auxiliary results}

\begin{lemma}[Lemma E.15 of \cite{dann2017unifying}]\label{lem:simulation} Consider two MDPs with transitions $p,\widehat{p}$ and rewards $r,\widehat{r}$, respectively. Let $V_{h}^\pi,\widehat{V}_{h}^\pi$ denote the value function of a (possibly stochastic) policy $\pi$ in these two MDPs. Then, for any $s,h$,
    \begin{align*}
        V_{h}^\pi(s) - \widehat{V}_{h}^\pi(s) = \widehat{\bE}^{\pi}\left[\left. \sum_{\ell=h}^H \left( r_{\ell}(s_{\ell},a_{\ell}) - \widehat{r}_{\ell}(s_{\ell},a_{\ell}) + \big(p_{\ell}(s_{\ell},a_{\ell}) - \widehat{p}_{\ell}(s_{\ell},a_{\ell})\big)^T V_{{\ell}+1}^\pi \right) \right| s_h = s\right].
    \end{align*}
\end{lemma}


\begin{lemma}\label{lem:qp-concentration}
Let $n\in\mathbb{N}$, $p,b\in\bR^n$ with $b$ having strictly positive entries, and $c \in \bR_{\geq 0}$. Then,
\begin{align*}
    \sup_{\substack{x \in \bR^{n} :\\
    \sum_{i=1}^n b_i x_i^2 \leq c}} \sum_{i=1}^n p_i x_i = \sqrt{c\sum_{i=1}^n \frac{p_i^2}{b_i}} .
\end{align*}
\end{lemma}
\begin{proof}
Let $v$ be the value of the optimization program. Then we know that 
\begin{align}\label{eq:QP-concentration}
   -v = \inf_{\substack{x \in \bR^{n} :\\
    \sum_{i=1}^n b_i x_i^2 \leq c}} -\sum_{i=1}^n p_i x_i.
\end{align}
The Lagrangian of the quadratic program above writes as
\begin{align*}
    \cL(x, \lambda ) = -\sum_{i=1}^n p_i x_i + \lambda \bigg(\sum_{i=1}^n b_i x_i^2 - c \bigg),
\end{align*}
where $\lambda \geq 0$. The KKT conditions then yield that the optimal solution satisfies that
\begin{align*}
 &\forall i\in [|1,n|],\quad    x_i = -\frac{p_i}{2\lambda b_i}\\
 &\sum_{i=1}^n b_i x_i^2 = c
\end{align*}
Solving this system yields that the optimal Lagrange multiplier $\lambda = \sqrt{\frac{c}{\sum_{i=1}^n \frac{p_i^2}{b_i}}}$ which implies that the value of (\ref{eq:QP-concentration}) is $- \sqrt{c\sum_{i=1}^n \frac{p_i^2}{b_i}}$. 
\end{proof}



%% file: appendix/app_RFE_final.tex
\section{Analysis of \rfealg{}}

To simplify the presentation of the algorithm and the analysis, we index the counts as well as the empirical estimates of transitions and rewards by their phase number. Hence, for each triplet $(h,s,a)$, $n_h^k(s,a)$ and $\widehat{p}_h^k(.|s,a)$ will refer to the number of visits and the empirical transition kernel respectively after $t_k$ episodes, i.e. at the end of the $k$-th phase. Finally, for a dataset of episodes $\cD$, $n_h(s,a; \cD)$ denotes the number of visits of $(h,s,a)$ in the episodes stored in $\cD$.

\subsection{Good event}
We introduce the following events
\begin{align*}
\cE_{vis} &:= \bigg(\textrm{The set built using \visitalg$\Big((h,s); \frac{\epsilon}{4SH^2}, \frac{\delta}{3SH}\Big)$ for all $(h,s)$}\\
    &\quad\quad \textrm{satisfies }\Big\{(h,s) : \sup_{\pi} p_h^{\pi}(s) \geq \frac{\epsilon}{4SH^2} \Big\} \subseteq \widehat{\cX} \subseteq \Big\{(h,s) : \sup_{\pi} p_h^{\pi}(s) \geq \frac{\epsilon}{32SH^2}\Big\}\\
    &\quad\quad\textrm{and } \forall(h,s)\in \widehat{\cX},\ \sup_{\pi} p_h^{\pi}(s) \leq \overline{W}_h(s)\leq 36\sup_{\pi} p_h^{\pi}(s)\bigg), \\
    \cE_{p}^{RF} &:= \bigg(\forall k \in \mathbb{N}^\star, \forall \pi \in \Pi^D, \forall r \in [0,1]^{SAH},
    \\ & \qquad\qquad \Big|\sum_{s,a,h} \big(\widehat{p}_h^{\pi,k}(s,a)- p_h^{\pi}(s,a)\big){r}_h(s,a) \Big| \leq \sqrt{\beta^{RF}(t_k,\delta/3)\sum_{(s,a,h)\in\widehat{\cX}}\frac{p_h^{\pi}(s,a)^2}{n_h^k(s,a)}} + \frac{\epsilon}{4} \bigg),\\
    \cE_{cov} &:= \bigg(\forall k\in \mathbb{N},\ \textrm{CovGame run with inputs $(c^{k}, \delta/6(k+1)^2)$ terminates after at most}\\
    & 64 m_k \varphi^\star(c^{k}) + \widetilde{\cO}\big(m_k \varphi^\star(\mathds{1}_{\widehat{\cX}}) SAH^2 (\log(6(k+1)^2/\delta) + S) \big)\ \textrm{episodes and returns a dataset $\cD_k$ }\\
    &\textrm{such that for all } (h,s,a)\in \widehat{\cX}, n_h(s,a;\cD_k)\geq c^{k}_h(s,a) \bigg),
\end{align*}
where $m_k = \log_2\big(\frac{\max_{s,a,h} c_h^{k}(s,a)}{\min_{s,a,h} c_h^{k}(s,a) \vee 1}\big) \vee 1$ and $\beta^{RF}$ is defined in appendix \ref{sec:app-concentration-RFE}. Then our good event is defined as the intersection 
\begin{align*}
    \cE_{good}^{RF} :=  \cE_{vis} \cap \cE_{p}^{RF} \cap \cE_{cov}.
\end{align*}


\begin{lemma}
We have that $\bP_{\cM}(\cE_{good}^{RF}) \geq 1-\delta$.
\end{lemma}
\begin{proof}
Let $\overline{\cE}$ denote the complementary event of $\cE$. We start by the following decomposition
\begin{align*}
 \bP_{\cM}(\overline{\cE_{good}^{RF}}) &\leq \bP_{\cM}(\overline{\cE_{vis}})+ \bP_{\cM}(\overline{\cE_{cov}}) + \bP_{\cM}(\overline{\cE_{p}^{RF}}\cap \cE_{vis} \cap \cE_{cov}). 
\end{align*}
Now we bound each term separately. First observe that applying Theorem \ref{thm:estimate-visitations} with parameter $\epsilon_0 = \epsilon/4SH^2$ yields $\bP_{\cM}(\overline{\cE_{vis}}) \leq \delta/3$. Second, using
Corollary \ref{cor:cover-instance-improved} we have
\begin{align*}
\bP_{\cM}(\overline{\cE_{cov}}) &\leq \sum_{k=0}^{\infty} \bP_{\cM}(\textrm{CovGame with inputs $(c^{k}, \delta/6(k+1)^2)$ fails}) \\
&\leq \sum_{k=0}^{\infty} \frac{\delta}{6(k+1)^2} = \frac{\delta \pi^2}{36} \leq \delta/3.
\end{align*}
Next, note that by design of \rfealg{} $n_h^0(s,a) = n_h(s,a; \widetilde{\cD}_0)$ and $c^0 = \mathds{1}_{\widehat{\cX}}$ so that $\cE_{cov} \subset \big(\forall (h,s,a)\in \widehat{\cX},\ n_h^0(s,a) \geq 1 \big)$. Therefore we have
\begin{align*}
 \bP_{\cM} & (\overline{\cE_{p}^{RF}} \cap \cE_{vis} \cap \cE_{cov}) \leq \bP_{\cM}\big(\overline{\cE_{p}^{RF}},\ \big\{(h,s) : \sup_{\pi} p_h^{\pi}(s) \geq \frac{\epsilon}{4SH^2} \big\} \subseteq \widehat{\cX} ,\ \forall (h,s,a)\in \widehat{\cX}\  n_h^0(s,a) \geq 1 \big) \\
 &= \bP_{\cM}\bigg(\big\{(h,s) : \sup_{\pi} p_h^{\pi}(s) \geq \frac{\epsilon}{4SH^2} \big\} \subseteq \widehat{\cX},\ \exists k\geq 0\ \exists \pi \in \Pi^D\ \exists r\in[0,1]^{SAH}:\\
 &\quad \quad \quad \quad \Big|\sum_{s,a,h} \big(\widehat{p}_h^{\pi,k}(s,a)- p_h^{\pi}(s,a)\big){r}_h(s,a) \Big| > \sqrt{\beta^{RF}(t_k,\delta/3)\sum_{(s,a,h)\in\widehat{\cX}}\frac{p_h^{\pi}(s,a)^2}{n_h^k(s,a)}} + \frac{\epsilon}{4} \bigg)\\
 &\stackrel{(a)}{\leq} \bP_{\cM}\bigg(\big\{(h,s) : \sup_{\pi} p_h^{\pi}(s) \geq \frac{\epsilon}{4SH^2} \big\} \subseteq \widehat{\cX},\ \exists t\geq t_0\ \exists \pi \in \Pi^D\ \exists r\in[0,1]^{SAH}:\\
 & \quad \quad \quad \quad \Big|\sum_{s,a,h} \big(\widehat{p}_h^{\pi,t}(s,a)- p_h^{\pi}(s,a)\big){r}_h(s,a) \Big| > \sqrt{\beta^{RF}(t,\delta/3)\sum_{(s,a,h)\in\widehat{\cX}}\frac{p_h^{\pi}(s,a)^2}{n_h^t(s,a)}} + \frac{\epsilon}{4} \bigg)\\
 &\stackrel{(b)}{\leq} \delta/3,
\end{align*}
where in (a) we introduced $t_0 = \inf\{t\geq 1: n_h^{t}(s,a) \geq 1, \forall (h,s,a) \in \widehat{\cX} \}$ and switched back to indexing counts and estimates by the episode number (instead of the phase) in order to apply Theorem \ref{thm:new-concentration-RFE} in (b) with $\cZ = \{ (h,s,a) : (h,s) \in \widehat{\cX}\}$ and $\epsilon_0 = \epsilon/4SH^2$. Combining the four inequalities above yields the desired result.
\end{proof}

\subsection{Low concentrability / Good coverage of all policies}

The next lemma shows that \rfealg{} achieves proportional coverage.
\begin{lemma}\label{lem:new-RFE-coverage}
Under the good event, for all phases $k \geq 0$, we have that 
\begin{align*} 
n_h^{k}(s,a) \geq 2^{k} \sup_{\pi} p^{\pi}_h(s,a)\quad \forall (h,s,a) \in \widehat{\cX}.
\end{align*}
\end{lemma}

\begin{proof}
First of all, note that for any triplet $(h,s,a) \in \widehat{\cX}, \sup_{\pi} p^{\pi}_h(s,a)$ is always attained by some deterministic policy. Therefore, it is sufficient to prove that, given a fixed deterministic policy $\pi \in \Pi^D$, 
\begin{align*}
 \forall k\geq 0,  \forall (h,s,a) \in \widehat{\cX},\quad n_h^{k}(s,a) \geq 2^{k} p^{\pi}_h(s,a)\;.  
\end{align*}
We do this by induction over $k$. For $k = 0$ the result is trivial since, under the good event, we have that for all $(h,s,a) \in \widehat{\cX}$, $n_h^{0}(s,a) \geq c^0_h(s,a) = 1 \geq 2^{0} p^{\pi}_h(s,a)$. Now suppose that the property holds for phase $k$. Then under the good event we know that for all $(h,s,a), n_h^{k+1}(s,a) - n_h^{k}(s,a) = n_h(s,a,\cD_{k+1}) \geq c^{k+1}_h(s,a)$. Plugging the definition of $c^{k+1}$ (Line 9 of Algorithm \ref{alg:RF-algo-final}) we get that for any $(h,s,a) \in \widehat{\cX}$, 

\begin{align}\label{ineq:new-PCE-correctness-1}
  n_h^{k+1}(s,a) &\geq  c^{k+1}_h(s,a) \nonumber\\
  &= 2^{k+1} \overline{W}_h(s) \nonumber\\
  &\geq 2^{k+1} \sup_{\pi} p_h^{\pi}(s) \nonumber\\
  & = 2^{k+1} \sup_{\pi} p_h^{\pi}(s,a),
\end{align}
where the second inequality uses the event $\cE_{vis}$.
\end{proof}
\subsection{Correctness}
\begin{lemma}
Let $\widehat{p}$ be the estimate of the transition probabilities that \rfealg{} outputs. For any reward function ${r}$, let $\hat{\pi}_{r}$ be an optimal policy in the MDP $(\widehat{p}, {r})$. Then 
$$\bP\left(\forall {r} \in [0,1]^{SAH}, V_1^{\hat{\pi}_{ r}}(s_1 ;{r}) \geq V_1^\star(s_1;{r}) - \epsilon\right) \geq 1-\delta.$$
In other words, \rfealg{} is $(\epsilon, \delta)$-PAC for reward-free exploration.
\end{lemma}

\begin{proof}
Assume that \rfealg{} stops as phase $k$ and let $\widehat{p}^k$ denote the empirical transition estimates that it returns. Fix any reward function $r = [r_h(s,a)]_{h,s,a} \in [0,1]^{SAH}$ and let $\widehat{\pi} \in \argmax_{\pi\in\Pi^D} (\widehat{p}^{\pi,k})^\top r $ be the policy obtained when planning for reward function $r$ under the transition model $\widehat{p}^k$. Further define $\pi^\star \in \argmax_{\pi\in\Pi^D} (p^{\pi})^\top r, V_1^\star := (p^{\pi^\star})^\top r$, and $V_1^{\widehat{\pi}} := (p^{\widehat{\pi}})^\top r$. Note that both $\widehat{\pi}$ and $\pi^\star$ are deterministic. Therefore under the good event $\cE_{good}^{RF}$ we have
\begin{align*}
   V_1^{\widehat{\pi}} &= (p^{\widehat{\pi}})^\top r\\
   &\stackrel{(a)}{\geq} (\widehat{p}^{\widehat{\pi},k})^\top r - \sqrt{\beta^{RF}(t_{k},\delta/3) \sum_{(s,a,h)\in \widehat{\cX}}\frac{p^{\widehat{\pi}}_h(s,a)^2}{n_h^{k}(s,a)}} - \frac{\epsilon
   }{4}\\
   & \stackrel{(b)}{\geq} (\widehat{p}^{\pi^\star,k})^\top r - \sqrt{\beta^{RF}(t_{k},\delta/3) \sum_{(s,a,h)\in \widehat{\cX}}\frac{p^{\widehat{\pi}}_h(s,a)^2}{n_h^{k}(s,a)}} - \frac{\epsilon
   }{4}\\
   &\stackrel{(c)}{\geq} ({p}^{\pi^\star})^\top r - \sqrt{\beta^{RF}(t_{k},\delta/3) \sum_{(s,a,h)\in \widehat{\cX}}\frac{p^{\pi^\star}_h(s,a)^2}{n_h^{k}(s,a)}} - \sqrt{\beta^{RF}(t_{k},\delta/3) \sum_{(s,a,h)\in \widehat{\cX}}\frac{p^{\widehat{\pi}}_h(s,a)^2}{n_h^{k}(s,a)}} - \frac{\epsilon
   }{2}\\
   &\stackrel{(d)}{\geq} V_1^\star - 2\sqrt{H\beta^{RF}(t_{k},\delta/3)2^{-k}}- \frac{\epsilon
   }{2}\\
   &\stackrel{(e)}{\geq} V_1^\star - \epsilon,
\end{align*}
where (a) and (c) use the good event $\cE_{p}^{RF}$ for policies $\widehat{\pi}$ and $\pi^\star$ respectively, (b) uses the definition of $\widehat{\pi}$, (d) uses Lemma \ref{lem:new-RFE-coverage} and (e) uses the stopping condition of \rfealg{} (Line 10 in Algorithm \ref{alg:RF-algo}). Note that the inequality above holds, under the good event $\cE_{good}$, jointly for all reward functions $r$. Since $\bP_{\cM}(\cE_{good}) \geq 1-\delta$, we have just proved that \rfealg{} is $(\epsilon, \delta)$-PAC for reward-free exploration.
\end{proof}

\subsection{Upper bound on the number of phases}
\begin{lemma}\label{lem:new-PCE-final-phase}
Define the index of the final phase of PCE, $\kappa_f := \inf \big\{k\in \bN_{+}: \sqrt{H\beta^{RF}(t_{k},\delta/3)2^{4-k}} \leq \epsilon \big\}$. Further let $\tau$ denote the number of episodes played by the algorithm. Then under the good event, it holds that $\kappa_f < \infty$ and
\begin{align*}
     2^{\kappa_f} \leq \frac{32H\beta^{RF}(\tau,\delta/3)}{\epsilon^2}.
\end{align*}
\end{lemma}
\begin{proof}
First we prove that $\kappa_f$ is finite. Under the good event we have
\begin{align*}
    t_k &= \sum_{j=0}^k d_j\nonumber\\    
    &\leq \sum_{j=0}^k \big[64 m_j \varphi^\star(c^{j}) + \widetilde{\cO}\big(m_j \varphi^\star(\mathds{1}_{\widehat{\cX}}) SAH^2 (\log(6(j+1)^2/\delta) + S) \big) \big],
\end{align*}
where we recall that $m_j = \log_2\big(\frac{\max_{s,a,h} c_h^{j}(s,a)}{\min_{s,a,h} c_h^{j}(s,a) \vee 1}\big) \vee 1$. Now using the fact that $c_h^j(s,a) \leq 2^j \mathds{1}((h,s,a)\in \widehat{\cX})$ for $j\geq 0$ we deduce that $m_0 = 1$ and $m_j \leq j\  \forall j \geq 1$ so that
\begin{align}\label{ineq:new-RF-final-phase-finite-1}
  t_k &\leq \sum_{j=0}^k \big[8 (j+1)2^j  \varphi^\star(\mathds{1}_{\widehat{\cX}}) + \widetilde{\cO}\big((j+1) \varphi^\star(\mathds{1}_{\widehat{\cX}}) SAH^2 (\log(4(j+1)^2/\delta) + S) \big) \big] \nonumber\\
  &= \cO_{k\to\infty}\big( k^2 2^k \big).
\end{align}
Now recall that the threshold $\beta^{RF}$ was defined in Appendix \ref{sec:app_concentration} as
\begin{align}\label{ineq:new-RF-final-phase-finite-2}
    \beta^{RF}(t,\delta) = 4H^2\log(1/\delta) + 24SH^3\log(A(1+ t))
\end{align}
Combining (\ref{ineq:new-RF-final-phase-finite-1}) and (\ref{ineq:new-RF-final-phase-finite-2}) gives that 
\begin{align*}
   \beta^{RF}(t_{k},\delta/3) = o_{k\to\infty}\big(2^k\big).
\end{align*}
Therefore $\kappa_f = \inf \big\{k\in \bN_{+}: \sqrt{H\beta^{RF}(t_{k},\delta/3)2^{4-k}} \leq \epsilon \big\}$ is indeed finite. The proof of the second statement is straightforward by noting that $\kappa_f-1$ does not satisfy the stopping condition (Line 12 in Algorithm \ref{alg:RF-algo-final}) and using the (crude) upper bound $t_{\kappa_f-1} \leq \tau$.
\end{proof}

\subsection{Upper bound on the phase length}

\begin{lemma}\label{lem:new-RFPCRL-phase-length}
Let $k\geq 1$ be such that \rfealg{} did not stop before phase $k$. Under the good event, the number of episodes played by \rfealg{} during phase $k$ satisfies
\begin{align*}
d_k &\leq  c_1 k H\beta^{RF}(\tau,\delta/3) \varphi^\star\bigg(\bigg[ \frac{\sup_{\pi} p_h^{\pi}(s) \ind\left(\sup_{\pi} p_h^{\pi}(s) \geq \frac{\varepsilon}{32SH^2}\right)}{\epsilon^2} \bigg]_{h,s,a}\bigg) \\ & \quad + \quad \widetilde{\cO}\bigg(k \frac{S^3A^2H^5 (\log(6(k+1)^2/\delta) + S)}{\epsilon} \bigg),
\end{align*}
where $c_1 = 73728$. Furthermore, the duration of the initial phase is upper bounded as 
\begin{align*}
    d_0 \leq \widetilde{\cO}\bigg(\frac{S^3A^2H^5 (\log(6/\delta) + S)}{\epsilon} \bigg).
\end{align*}


\end{lemma}
\begin{proof}
Using the good event and the definition of $c^k$ we write
\begin{align}\label{ineq:new-RFPCRL-phase-length-1}
&d_k \leq 64 m_k \varphi^\star\bigg(\bigg[2^{k} \overline{W}_h(s) \indi{(h,s,a)\in \widehat{\cX}}\bigg]_{h,s,a}\bigg) + \widetilde{\cO}\big(m_k \varphi^\star(\mathds{1}_{\widehat{\cX}}) SAH^2 (\log(6(k+1)^2/\delta) + S) \big)\nonumber\\
&\stackrel{(a)}{\leq} 64k\varphi^\star\!\bigg(\bigg[2^{k} \overline{W}_h(s) \indi{\!(h,s,a)\in \widehat{\cX}} \!\!\bigg]_{h,s,a}\bigg)\! + \widetilde{\cO}\big(k \varphi^\star(\mathds{1}_{\widehat{\cX}}) SAH^2 (\log(6(k+1)^2/\delta) + S) \big),
\end{align}
where (a) uses that $m_k = \log_2\big(\frac{\max_{s,a,h} c_h^{k}(s,a)}{\min_{s,a,h} c_h^{k}(s,a) \vee 1}\big) \vee 1 \leq k$. Now by definition of the good event we have that for 
any triplet $(h,s,a)\in \widehat{\cX},\ \overline{W}_h(s) \leq  36 \sup_{\pi} p_h^{\pi}(s)$. Therefore
\begin{align}\label{ineq:new-RFPCRL-phase-length-2}
    &\varphi^\star\bigg(\bigg[2^{k} \overline{W}_h(s) \indi{(h,s,a)\in \widehat{\cX}} \bigg]_{h,s,a}\bigg) \stackrel{(a)}{\leq} \varphi^\star\bigg(\bigg[36 \times 2^{k} \sup_{\pi} p_h^{\pi}(s) \indi{(h,s,a)\in \widehat{\cX}} \bigg]_{h,s,a}\bigg) \nonumber\\
    &\quad\quad\stackrel{(b)}{\leq} \varphi^\star\bigg(\bigg[ \frac{1152 H\beta^{RF}(\tau,\delta/3)\sup_{\pi} p_h^{\pi}(s) \indi{(h,s,a)\in \widehat{\cX}}}{\epsilon^2} \bigg]_{h,s,a}\bigg) \nonumber\\
    &\quad\quad\stackrel{(c)}{\leq} 1152 H\beta^{RF}(\tau,\delta/3)\varphi^\star\bigg(\bigg[ \frac{\sup_{\pi} p_h^{\pi}(s)\indi{\sup_{\pi} p_h^{\pi}(s) \geq \frac{\varepsilon}{32SH^2}}}{\epsilon^2} \bigg]_{h,s,a}\bigg),
\end{align}
where (a) uses that $\varphi^\star(c) \leq \varphi^\star(c')$ if $ \forall (h,s,a)\ c_h(s,a) \leq  c_h'(s,a)$, (b) uses Lemma \ref{lem:new-PCE-final-phase} and the fact that $k\leq \kappa_f$ since \rfealg{} did not stop before phase $k$ and (c) uses Lemma \ref{lem:flow-linear} and the fact that $\widehat{\cX} \subseteq \big\{(h,s,a) : \sup_{\pi} p_h^{\pi}(s) \geq \frac{\epsilon}{32SH^2}\big\}$ on the good event. Using again this last property yields
\begin{align}\label{ineq:new-RFPCRL-phase-length-3}
    \varphi^\star(\mathds{1}_{\widehat{\cX}}) &{\leq} \sum_{h,s,a} \frac{\indi{(h,s,a)\in \hat
    \cX}}{\sup_{\pi} p_h^{\pi}(s,a)}\nonumber\\
    &= \sum_{(h,s,a)\in \widehat{\cX}} \frac{1}{\sup_{\pi} p_h^{\pi}(s)}
    \leq  \frac{32H^3S^2A}{\epsilon},
\end{align}
where the first inequality uses Lemma \ref{lem:bound-flow-simple}.
Combining (\ref{ineq:new-RFPCRL-phase-length-1}), (\ref{ineq:new-RFPCRL-phase-length-2}) and (\ref{ineq:new-RFPCRL-phase-length-3})  
proves the statement for $k\geq 1$. Now it remains to upper bound the duration of the burn-in phase. To that end, we write that by definition of the good event
\begin{align*}
    d_0 &\leq 64m_0 \varphi^\star( \mathds{1}_{\widehat{\cX}}) + \widetilde{\cO}\big(\varphi^\star(\mathds{1}_{\widehat{\cX}}) SAH^2 (\log(6/\delta) + S) \big),
\end{align*}
where $m_0 = \log_2\big(\frac{\max_{s,a,h} c_h^{0}(s,a)}{\min_{s,a,h} c_h^{0}(s,a) \vee 1}\big) \vee 1 = 1 $. Therefore 
\begin{align*}
    d_0 &\leq \widetilde{\cO}\big(\varphi^\star(\mathds{1}_{\widehat{\cX}}) SAH^2 (\log(6/\delta) + S) \big)\\
        &\leq \widetilde{\cO}\bigg(\frac{S^3A^2H^5 (\log(6/\delta) + S)}{\epsilon} \bigg),
\end{align*}
where the last inequality uses (\ref{ineq:new-RFPCRL-phase-length-3}).
\end{proof}

\subsection{Total sample complexity}

\begin{theorem}
With probability at least $1-\delta$, the total sample complexity of \rfealg{} satisfies {\small
\begin{align*}
    \tau &\leq \widetilde{\cO}\bigg(\!\big(H^3\log(1/\delta) + SH^4\big) \varphi^\star\!\bigg(\bigg[ \frac{\sup_{\pi} p_h^{\pi}(s) \ind\left(\sup_{\pi} p_h^{\pi}(s) \geq \frac{\varepsilon}{32SH^2}\right)}{\epsilon^2} \bigg]_{h,s,a}\bigg) + \frac{S^3A^2H^5 (\log(1/\delta) + S)}{\epsilon}  \!\bigg),
\end{align*}}
where $\widetilde{\cO}$ hides poly-logarithmic factors in $S,A,H, \epsilon$ and $\log(1/\delta)$.
\end{theorem}


\begin{proof}
 Denoting by $T_{vis}$ the number of episodes used by the \visitalg sub-routine in line 2 of the algorithm, we write
\begin{align}\label{ineq:new-RF-final-complexity-1}
    \tau &= T_{vis} + \sum_{k=0}^{\kappa_f} d_k \nonumber\\
    &\leq T_{vis} + \widetilde{\cO}\bigg(\frac{S^3A^2H^5 (\log(6/\delta) + S)}{\epsilon} \bigg) \nonumber \\
    & \quad + \sum_{k=1}^{\kappa_f} \bigg[  c_1 k H\beta^{RF}(\tau,\delta/3) \varphi^\star\bigg(\bigg[ \frac{\sup_{\pi} p_h^{\pi}(s,a)\ind\left(\sup_{\pi} p_h^{\pi}(s) \geq \frac{\varepsilon}{32SH^2}\right)}{\epsilon^2} \bigg]_{h,s,a}\bigg) \nonumber \\
    &\quad + \widetilde{\cO}\bigg(k \frac{S^3A^2H^5 (\log(6(k+1)^2/\delta) + S)}{\epsilon} \bigg) \bigg]\nonumber\\
    &\leq T_{vis} + c_1 \kappa_f^2 H\beta^{RF}(\tau,\delta/3) \varphi^\star\bigg(\bigg[ \frac{\sup_{\pi} p_h^{\pi}(s) \ind\left(\sup_{\pi} p_h^{\pi}(s) \geq \frac{\varepsilon}{32SH^2}\right)}{\epsilon^2} \bigg]_{h,s,a}\bigg) \nonumber\\ & \quad \quad+ \widetilde{\cO}\bigg(\kappa_f^2 \frac{S^3A^2H^5 (\log(6(\kappa_f+1)^2/\delta) + S)}{\epsilon} \bigg),
\end{align}
where we used Lemma \ref{lem:new-RFPCRL-phase-length} to upper bound $(d_k)_{k\geq 0}$. From Theorem~\ref{thm:estimate-visitations}, we know that $T_{vis}$ is deterministic and satisfies 
\begin{equation}\label{ineq:new-RF-final-complexity-2}
    T_{vis} = \widetilde{\cO}\bigg(\frac{S^3AH^4\left(\log\left(\frac{SAH}{\delta}\right) + S\right)}{\epsilon}\bigg)
    =\widetilde{\cO}\bigg(\kappa_f^2 \frac{S^3A^2H^5 (\log(6(\kappa_f+1)^2/\delta) + S)}{\epsilon} \bigg).
\end{equation}
Combining inequalities (\ref{ineq:new-RF-final-complexity-1}) and (\ref{ineq:new-RF-final-complexity-2})  with the definition of the threshold $\beta^{RF}(t,\delta) = 4H^2\log(1/\delta) + 24SH^3\log(A(1+t))$ we get
\begin{align}\label{ineq:new-RF-final-complexity-3}
 \tau &\leq c_1 \kappa_f^2 H\beta^{RF}(\tau,\delta/3) \varphi^\star\bigg(\bigg[ \frac{\sup_{\pi} p_h^{\pi}(s) \ind\left(\sup_{\pi} p_h^{\pi}(s) \geq \frac{\varepsilon}{32SH^2}\right)}{\epsilon^2} \bigg]_{h,s,a}\bigg) \nonumber\\ 
 &\quad  + \widetilde{\cO}\bigg(\kappa_f^2 \frac{S^3A^2H^5 (\log(6(\kappa_f+1)^2/\delta) + S)}{\epsilon} \bigg)\nonumber\\
 &\leq c_2 \kappa_f^2 \bigg(H^3\log(1/\delta) + SH^4 \log(A(1+\tau))\bigg) \varphi^\star\bigg(\bigg[ \frac{\sup_{\pi} p_h^{\pi}(s) \ind\left(\sup_{\pi} p_h^{\pi}(s) \geq \frac{\varepsilon}{32SH^2}\right)}{\epsilon^2} \bigg]_{h,s,a}\bigg) \nonumber\\
 &\quad + \widetilde{\cO}\bigg(\kappa_f^2 \frac{S^3A^2H^5 (\log(6(\kappa_f+1)^2/\delta) + S)}{\epsilon} \bigg),
\end{align} 
where $c_2 = 24c_1$ On the other hand, thanks to Lemma \ref{lem:new-PCE-final-phase} and the definition of the threshold $\beta^{RF}$ we have that 
\begin{align}\label{ineq:new-RF-final-complexity-4}
\kappa_f &\leq \log_2\bigg(\frac{128H^3\log(1/\delta) + 768SH^4 \log(A(1+\tau))}{\epsilon^2}\bigg)
\end{align}
Combining (\ref{ineq:new-RF-final-complexity-3}) with (\ref{ineq:new-RF-final-complexity-4}) and solving for $\tau$ we get that {\small
\begin{align*}
    \tau &\leq \widetilde{\cO}\bigg(\!\big(H^3\log(1/\delta) + SH^4\big) \varphi^\star\!\bigg(\bigg[ \frac{\sup_{\pi} p_h^{\pi}(s) \ind\left(\sup_{\pi} p_h^{\pi}(s) \geq \frac{\varepsilon}{32SH^2}\right)}{\epsilon^2} \bigg]_{h,s,a}\bigg) + \frac{S^3A^2H^5 (\log(1/\delta) + S)}{\epsilon}  \!\bigg),
\end{align*}}
where $\widetilde{\cO}$ hides poly-logarithmic factors in $S,A,H, \epsilon$ and $\log(1/\delta)$.
\end{proof}

%% file: appendix/app_benign.tex
\subsection{Benign instances for PCE}\label{sec:app_benign}

In this section we propose some MDP instances in which the quantity
\begin{align}
    \cC(\textrm{PCE}, \epsilon) := \varphi^\star([\sup_{\pi} p_h^\pi(s,a)]_{h,s,a})H^3/\epsilon^2,
\end{align}
which is (an upper bound on) the leading term in the small $(\delta,\varepsilon)$ regime in our sample complexity bound for PCE, can be smaller than the minimax rate $SAH^3/\epsilon^2$.
\subsubsection{Disguised contextual bandits}
\begin{lemma}
Suppose that $\cM$ is a "disguised" contextual bandit, i.e.,
\begin{align*}
\forall (h,s,a,s'),\ p_h(s'|s,a) = p_h(s'|s).    
\end{align*}
Then $\cC(\textrm{PCE},\epsilon) = AH^3/\epsilon^2$.
\end{lemma}

\begin{proof}
In this case for any $(h,s)$ and any policy $\pi$, $p_h^\pi(s) = p_h(s)$ is independent of the policy. Thanks to Lemma \ref{lem:flow-minmax} we have
\begin{align*}
  \varphi^\star([\sup_{\pi} p_h^\pi(s,a)]_{h,s,a}) &= \inf_{\pi^{exp}\in \Pi^S} \max_{s,a,h} \frac{\sup_{\pi} p_h^\pi(s,a)}{p^{\pi^{exp}}_h(s,a)}\\
  &= \inf_{\pi^{exp}\in \Pi^S} \max_{s,a,h} \frac{ p_h(s) \sup_{\pi} \pi_h(a|s)}{p_h(s)\pi_h^{exp}(a|s)}\\ 
  &= \inf_{\pi^{exp}\in \Pi^S} \max_{s,h} \frac{1}{\min_{a} \pi_h^{exp}(a|s)}\\
  &= A,
\end{align*}
where the last equality is because $(\min_{a} \pi_h^{exp}(a|s))^{-1} \geq A$ and the infimum over $\Pi^S$ is achieved by the uniform policy.
\end{proof}

\subsubsection{Ergodic MDPs}
Let $\alpha, \beta \in (0,1)$ such that $\alpha > \beta$. Further define the set of probability vectors such that 
\begin{align*}
    \cP_{\alpha,\beta} = \bigg\{q\in \bR_{+}^{S}: \sum_{i=1}^S q_i = 1,\ \max_i q_i \leq S^{\alpha-1}, \min_i q_i \geq \frac{1- S^{\beta-1}}{S-1} \bigg\}.
\end{align*}
Note that such set is never empty since the vector $(S^{\beta-1},\frac{1- S^{\beta-1}}{S-1},\ldots,\frac{1- S^{\beta-1}}{S-1})$ always satisfies the inequalities in its definition. We define the class of MDPs $\mathfrak{M}_{erg}$ such that their transition kernel satisfies
\begin{align*}
    \forall (h,s,a),\ p_h(.|s,a) \in \cP_{\alpha,\beta}.
\end{align*}

\begin{lemma}
Assume that $\cM \in \mathfrak{M}_{erg}$, then $\cC(\textrm{PCE},\epsilon) \leq S^\alpha AH^4/\epsilon^2$. 
\end{lemma}
\begin{remark}
Note that the "ergodicity" of MDPs in $\mathfrak{M}_{erg}$ can be as small as one wishes: by taking the limit $\beta \to 1$, the constraint $\min_{s'} p_h(s'|s,a) \geq \frac{1- S^{\beta-1}}{S-1}$ becomes vacuous so the MDP can be non-ergodic. In that regime, $\alpha = 1$ and we recover the minimax sample complexity (up to an $H$ factor) $SAH^3/\epsilon^2$.
\end{remark}

\begin{proof}
First of all we note that
\begin{align}\label{ineq:benign-1}
 \forall \pi\in \Pi\ \forall s\in \cS,\   p^\pi_h(s) &= \sum_{s'\in \cS} p^\pi_{h-1}(s) p_{h}(s|s', \pi_{h-1}(s'))\nonumber \\
    &\leq  \sum_{s'\in \cS} p^\pi_{h-1}(s) S^{\alpha-1} = S^{\alpha-1}.
\end{align}
Similarly 
\begin{align}\label{ineq:benign-2}
   \forall \pi\in \Pi\ \forall s\in \cS,\   p^\pi_h(s) \geq  \frac{1- S^{\beta-1}}{S-1}.
\end{align}
Now using Lemma \ref{lem:bound-flow-refined} we have that
\begin{align}\label{ineq:benign-3}
    \varphi^\star([\sup_{\pi} p_h^\pi(s,a)]_{h,s,a}) &\leq \sum_{h=1}^H \inf_{\pi^{exp}\in \Pi^S} \max_{s} \frac{1}{p^{\pi^{exp}}_h(s)} \sum_{a}\sup_{\pi} p_h^\pi(s,a)\nonumber\\
    &= \sum_{h=1}^H \inf_{\pi^{exp}\in \Pi^S} \max_{s} \frac{A\sup_{\pi} p_h^\pi(s)}{p^{\pi^{exp}}_h(s)} =  A \sum_{h=1}^H \underbrace{\inf_{\pi^{exp}\in \Pi^S} \max_{s} \frac{\sup_{\pi} p_h^\pi(s)}{p^{\pi^{exp}}_h(s)}}_{:= \cC_h},
\end{align}
Now fix $h\in [H]$ and denote by $\pi^{s}$ any policy in $\argmax_{\pi\in \Pi} p_h^\pi(s)$. Further define the stochastic policy $\widetilde{\pi}$ such that
\begin{align*}
    p^{\widetilde{\pi}} = \frac{\sum_{s'\in \cS} p^{\pi^s} }{S}. 
\end{align*}
Using (\ref{ineq:benign-2}) we have that
\begin{align}\label{ineq:benign-4}
  \forall s\in \cS,\   p_h^{\widetilde{\pi}}(s) &= \frac{\sum_{s'\in \cS} p_h^{\pi^{s'}}(s)}{S}\nonumber\\
  &\geq \frac{\sup_{\pi\in \Pi} p_h^\pi(s) + (S-1)\frac{1- S^{\beta-1}}{S-1} }{S}\nonumber\\
  &= \frac{\sup_{\pi\in \Pi} p_h^\pi(s) + 1- S^{\beta-1}}{S}.
\end{align}
Therefore
\begin{align*}
   \cC_h &= \inf_{\pi^{exp}\in \Pi^S} \max_{s} \frac{\sup_{\pi} p_h^\pi(s)}{p^{\pi^{exp}}_h(s)} \\ 
   &\leq  \max_{s} \frac{\sup_{\pi} p_h^\pi(s)}{p_h^{\widetilde{\pi}}(s)}\\
   &\stackrel{(a)}{\leq} \max_{s} \frac{S\sup_{\pi} p_h^\pi(s)}{\sup_{\pi\in \Pi} p_h^\pi(s) + 1- S^{\beta-1}}\\
   &= \max_{s} \frac{S}{1 + \frac{1- S^{\beta-1}}{\sup_{\pi} p_h^\pi(s)}}\\
   &\stackrel{(b)}{\leq} \max_{s} \frac{S}{1 + S^{1-\alpha}(1- S^{\beta-1})}\\
   &= \frac{S}{1 + S^{1-\alpha}- S^{\beta-\alpha}} \leq S^\alpha,
\end{align*}
where (a) uses (\ref{ineq:benign-4}) and (b) uses (\ref{ineq:benign-1}). Combining (\ref{ineq:benign-3}) with the previous inequality yields that $\varphi^\star([\sup_{\pi} p_h^\pi(s,a)]_{h,s,a}) \leq S^\alpha A H$. 
\end{proof}

%% file: appendix/app_BPI.tex
\section{PRINCIPLE and its Analysis}
\subsection{Pseudo-code of PRINCIPLE}
The pseudo code of PRINCIPLE is detailed in Algorithm \ref{alg:PRINCIPLE}.
\begin{algorithm}[!ht]
\caption{ PRINCIPLE (PRoportIoNal Coverage with Implicit PoLicy Elimination) }\label{alg:PRINCIPLE}
\begin{algorithmic}[1]
\STATE \textbf{Input:}  Precision $\epsilon$, Confidence $\delta$, set of reachable states $\cS$
\STATE \textbf{Output:}  A policy $\widehat{\pi}$ that is $\epsilon$-optimal w.p larger than $1-\delta$
\STATE Define target function $c_h^0(s,a) = 1$ for all $(h,s,a)$
\STATE Execute $\textsc{CovGame}\big(c^0,\ \delta/4\big)$ to get dataset $\cD_{0}$ and number of episodes $d_{0}$ \hfill \textsc{// Burn-in phase}
\STATE Initialize episode count $t_0 \leftarrow d_0$ and statistics $n_h^{0}(s,a), \widehat{r}_h^{0}(s,a), \widehat{p}^{0}_h(.|s,a)$ using $\widetilde{\cD}_{0}$
\STATE Initialize the set of active distributions $\Omega^{0} \leftarrow \Omega(\widehat{p}^0)$

\FOR{$k=1,\dots$}

\blue{
\STATE \textsc{// Proportional Coverage}
\STATE  Compute $c_h^{k}(s,a) := 2^{k} \min\big(\sup_{\widehat{\rho}\in \Omega^{k-1}} \widehat{\rho}_h(s,a) + 2\sqrt{H\beta^{bpi}(t_{k-1}+ SAH2^k,\delta/2) 2^{1-k}},\ 1 \big)$ for all $(h,s,a)$ 

\STATE  Execute $\textsc{CovGame}\big(c^{k},\ \delta/4(k+1)^2 \big)$ to get dataset $\widetilde{\cD}_{k}$ and number of episodes $T_{k}$}

\STATE \textbf{if} $T_{k} > SAH 2^k $ \textbf{then} 
\STATE \quad Run $\textsc{PruneDataset}(\widetilde{\cD}_{k}, c^k)$ to get \textit{effective} dataset $\cD_k$ and \textit{effective phase length} $d_k$
\STATE \textbf{else}
\STATE \quad Set  $d_k \leftarrow T_{k}$ and $\cD_k \leftarrow \widetilde{\cD}_{k}$
\STATE \textbf{end if}

\STATE Update \textit{effective episode count} $t_{k} \leftarrow t_{k-1}+ d_k$ and statistics $n_h^{k}(s,a), \hat{r}_h^{k}(s,a), \widehat{p}^{k}_h(.|s,a)$ using $\cD_k$\\
\red{\normalsize{\textsc{// state-action-Distribution Elimination}}}
 
\STATE Compute the lower confidence bound 
$$\underline{V}_1^{k} :=  \sup_{\substack{\widehat{\rho} \in \Omega(\widehat{p}^{k}),\\ \max\limits_{h,s,a} \widehat{\rho}_h(s,a)/n_h^{k}(s,a) \leq 2^{-k}}} \widehat{\rho}^\top \widehat{r}^{k} - \sqrt{2^{2-k}H\beta^{bpi}(t_{k},\delta/2)}$$ 
\STATE Update the set of active state-action distributions 
$$\Omega^{k} \leftarrow \bigg\{\widehat{\rho} \in \Omega(\widehat{p}^{k}):\  \widehat{\rho}^\top \widehat{r}^{k} \geq   \underline{V}_1^{k}\ \textrm{and}\  \max\limits_{h,s,a} \widehat{\rho}_h(s,a)/n_h^{k}(s,a) \leq 2^{-k} \bigg\}$$

\STATE \textbf{if} $\sqrt{2^{2-k}H\beta^{bpi}(t_{k},\delta/2)} \leq \epsilon$  \textbf{then} 
\STATE \quad Compute any $\widehat{\rho}^\star \in \argmax_{\widehat{\rho} \in \Omega^{k}} \widehat{\rho}^\top \widehat{r}^{k}$ and extract the corresponding policy $\widehat{\pi}$

\STATE \quad \textbf{return} $\widehat{\pi}$
\STATE \textbf{end if}
\ENDFOR
\end{algorithmic}
\end{algorithm}

\begin{algorithm}[ht]
\caption{PruneDataset}\label{alg:prune-episodes}
\begin{algorithmic}[1]
\STATE \textbf{Input:}  Target counts $c$, Dataset $\widetilde{\cD}$ such that $n_h(s,a; \widetilde{\cD}) \geq c_h(s,a)$ for all $(h,s,a)$ 
\STATE \textbf{Output:} A dataset $\cD$ of $d \leq SAH 2^k$ episodes satisfying $n_h(s,a; \cD) \geq c_h(s,a)$ for all $(h,s,a)$
\STATE Initialize dataset $\cD \leftarrow \emptyset$, episode number $d\leftarrow 0$ and dataset-counts $n_h(s,a;\cD) \leftarrow 0$ for all $(h,s,a)$

\FOR{episode $e = (s_\ell^e, a_\ell^e, R_\ell^e)_{1\leq \ell\leq H}$ in $\widetilde{\cD}$}
\STATE \textbf{if} $\exists \ell \in [H]$ such that $n_\ell(s_\ell^e, a_\ell^e; \cD) < c_\ell(s_\ell^e, a_\ell^e)$ \textbf{then} 
\STATE \quad Update dataset-counts $n_h(s_h^e,a_h^e;\cD) \leftarrow n_h(s_h^e,a_h^e;\cD) +1 $ for all $h\in [H]$
\STATE \quad Update dataset $\cD \leftarrow \cD \cup \{e\}$ and episode number $d \leftarrow d+1$

\STATE \quad \textbf{if} $n_h(s,a; \cD) \geq c_h(s,a)$ for all $(h,s,a)$ \textbf{then}
\STATE \quad \quad \textbf{return} $(\cD, d)$
\STATE  \quad \textbf{end if}
\STATE \textbf{end if}
\ENDFOR
\end{algorithmic}
\end{algorithm}

\subsection{Analysis of PRINCIPLE}\label{app:BPI}
To simplify the presentation of the algorithm and the analysis, we index the counts as well as the empirical estimates of transitions and rewards by their phase number. Hence, for each triplet $(h,s,a)$, $n_h^k(s,a), \widehat{p}_h^k(.|s,a)$ and $\widehat{r}_h^k(s,a)$ will refer to the number of visits, the empirical transition kernel and the empirical mean reward respectively after $t_k$ episodes, i.e. at the end of the $k$-th phase. For a transition kernel $\widetilde{p}$, we define the corresponding set of state-action distributions as $\Omega(\widetilde{p}) = \big\{ \widetilde{p}^\pi : \pi\in\Pi^S \big\}$. Finally, for a dataset of episodes $\cD$, $n_h(s,a; \cD)$ denotes the number of visits of $(h,s,a)$ in the episodes stored in $\cD$.

\subsubsection{Good event}
We introduce the following events
\begin{align*}
    \cE_{bpi} &:= \bigg(\forall k \in \mathbb{N}^\star,\forall \pi \in \Pi^S,\ \big|\widehat{V}_1^{\pi,k} - V_1^\pi \big| \leq \sqrt{\beta^{bpi}(t_k,\delta/2)\min\Big(\sum_{s,a,h}\frac{p_h^{\pi}(s,a)^2}{n_h^k(s,a)}, \sum_{s,a,h}\frac{\widehat{p}_h^{\pi,k}(s,a)^2}{n_h^k(s,a)} \Big)}\\ 
    &\textrm{and}\ \Big|\sum_{s,a,h} \big(\widehat{p}_h^{\pi,k}(s,a)- p_h^{\pi}(s,a)\big){\widetilde r}_h(s,a) \Big| \leq \sqrt{\beta^{bpi}(t_k,\delta/2)\sum_{s,a,h}\frac{p_h^{\pi}(s,a)^2}{n_h^k(s,a)}}\ \textrm{for all } \widetilde r \in [0,1]^{SAH}\bigg),\\
    \cE_{cov} &:= \bigg(\forall k\in \mathbb{N},\ \textrm{CovGame run with inputs $(c^{k}, \delta/4(k+1)^2)$ terminates after at most}\\
    & 64 m_k \varphi^\star(c^{k}) + \widetilde{\cO}\big(m_k \varphi^\star(\mathds{1}) SAH^2 (\log(4(k+1)^2/\delta) + S) \big) \textrm{ episodes and returns a dataset $\widetilde{\cD}_k$ }\\
    &\textrm{such that for all } (h,s,a), n_h(s,a;\widetilde{\cD}_k)\geq c^{k}_h(s,a) \bigg),
\end{align*}
where $m_k = \log_2\big(\frac{\max_{s,a,h} c_h^{k}(s,a)}{\min_{s,a,h} c_h^{k}(s,a) \vee 1}\big) \vee 1$ and $\beta^{bpi}(t,\delta) = 16H^2\log(2/\delta) + 96SAH^3\log(1+t)$ is defined in Appendix \ref{sec:app_concentration}. Then our good event is defined as the intersection 
\begin{align*}
    \cE_{good} :=  \cE_{bpi}\cap \cE_{cov}.
\end{align*}

\begin{lemma}
We have that $\bP_{\cM}(\cE_{good}) \geq 1-\delta$.
\end{lemma}
\begin{proof}
Let $\overline{\cE}$ denote the complementary event of $\cE$. We start by the following decomposition
\begin{align*}
 \bP_{\cM}(\overline{\cE_{good}}) &\leq \bP_{\cM}(\overline{\cE_{cov}}) + \bP_{\cM}(\overline{\cE_{bpi}} \cap \cE_{cov}).
\end{align*}
Now we bound each term separately. First observe that using Corollary \ref{cor:cover-instance-improved} we have
\begin{align*}
\bP_{\cM}(\overline{\cE_{cov}}) &\leq \sum_{k=0}^{\infty} \bP_{\cM}(\textrm{CovGame with inputs $(c^{k}, \delta/4(k+1)^2)$ fails}) \\
&\leq \sum_{k=0}^{\infty} \frac{\delta}{4(k+1)^2} = \frac{\delta \pi^2}{24} \leq \delta/2.
\end{align*}
Next, note that by design of PRINCIPLE $n_h^0(s,a) = n_h(s,a; \widetilde{\cD}_0)$ and $c^0 = \mathds{1}$ so that $\cE_{cov} \subset \big(\forall (h,s,a),\ n_h^0(s,a) \geq 1 \big)$. Therefore we have
\begin{align*}
 \bP_{\cM}(\overline{\cE_{bpi}} \cap \cE_{cov}) &\leq \bP_{\cM}\big(\overline{\cE_{bpi}}\ \textrm{and}\ \forall (h,s,a)\  n_h^0(s,a) \geq 1 \big) \\
 &{\leq} \delta/2,
\end{align*}
where we applied Theorem \ref{thm:new-concentration-BPI} and used the fact that $\beta^{p}(t,\delta) \leq \beta^{bpi}(t,\delta)$. Combining the two inequalities above yields the desired result.
\end{proof}

\subsubsection{Low Concentrability / Good coverage of optimal policies}

\begin{lemma}\label{fact1} Under the good event, for all $k\geq 1$ such that PRINCIPLE did not stop before phase $k$, it holds that $n_h(s,a,\cD_k) \geq c_h^k(s,a)$ for all $(h,s,a)$ and $d_k \leq SAH2^k$.  
\end{lemma}

\begin{proof}
Fix $k\geq 1$ such that PRINCIPLE did not stop before phase $k$. By definition of the good event we know that at the end of CovGame, $n_h(s,a;\widetilde{\cD}_k)\geq c^{k}_h(s,a)$ for all $(h,s,a)$. Now we distinguish two cases. \textbf{If $T_k \leq SAH2^k$:} then the result follows immediately since in this case, by design of PRINCIPLE (line 13 in Algorithm \ref{alg:PRINCIPLE}), $\cD_k = \widetilde{\cD}_k$ and $d_k = T_k$. \\
\textbf{If $T_k > SAH2^k$:} the first statement is a direct consequence of the stopping condition of \textsc{PruneDataset} run with parameters $(\widetilde{\cD}_k, c^k)$ (lines 7-8 in Algorithm \ref{alg:prune-episodes}). Now for the second statement, observe that each new episode $e$ added by \textsc{PruneDataset} to $\cD_k$ increments the dataset-count of at least one triplet $(h,s,a)$ that is not yet covered, i.e. $n_h(s,a; \cD_k) < c_h^k(s,a)$. By the pigeon-hole principle it takes at most $\sum_{h,s,a} c_h^k(s,a)$ episodes to ensure that $n_h(s,a,\cD_k) \geq c_h^k(s,a)$ for all $(h,s,a)$. Therefore
\begin{align*}
   d_{k} &\leq \sum_{h,s,a} c_h^k(s,a) \leq SAH 2^k,
\end{align*}
where we used that $c_h^k(s,a) \leq 2^k$ due to the clipping.
\end{proof}

The next lemma shows that the set of active state-action distributions always contains the distributions induced by optimal policies. 
\begin{lemma}\label{lem:PRINCIPLE-optimal-never-elim}
Under the good event, for all optimal policies $\pi^\star \in \Pi^\star$ and all phases $k \geq 0$, we have that 
\begin{align*}
\widehat{p}^{\pi^\star,k} \in \Omega^{k} \quad \textrm{and}\ 
    n_h^{k}(s,a) \geq 2^{k} p^{\pi^\star}_h(s,a)\quad \forall (h,s,a).
\end{align*}
\end{lemma}

\begin{proof}
We fix an optimal policy $\pi^\star$ and prove the statement by induction. For $k = 0$, the fact that $\widehat{p}^{\pi^\star,0} \in \Omega^{0}$ is trivial since $\Omega^{0} = \Omega(\widehat{p}^0)$ consists of all possible state-action distributions induced in the MDP whose transition kernel is $\widehat{p}^0$. Furthermore, under the good event we have that, for all $(h,s,a)$, $n_h^{0}(s,a) \geq c^0_h(s,a) = 1 \geq 2^{0} \max\big(p^{\pi^\star}_h(s,a), \widehat{p}^{\pi^\star,0}_h(s,a) \big)$. Now suppose that the property holds for phase $k$. Then we know that for any $(h,s,a)$
\begin{align}\label{ineq:PRINCIPLE-correctness-1}
    \big|\widehat{p}^{\pi^\star, k+1}_h(s,a) - \widehat{p}^{\pi^\star, k}_h(s,a)\big| &\leq \big|\widehat{p}^{\pi^\star, k+1}_h(s,a) -p^{\pi^\star}_h(s,a)\big| + \big|p^{\pi^\star}_h(s,a) - \widehat{p}^{\pi^\star, k}_h(s,a)\big|\nonumber\\
    &\stackrel{(a)}{\leq} \sqrt{\beta^{bpi}(t_{k+1},\delta/2) \sum_{s,a,h}\frac{p^{\pi^\star}_h(s,a)^2}{n_h^{k+1}(s,a)}} + \sqrt{\beta^{bpi}(t_{k},\delta/2) \sum_{s,a,h}\frac{p^{\pi^\star}_h(s,a)^2}{n_h^{k}(s,a)}} \nonumber\\
    &\stackrel{(b)}{\leq} 2\sqrt{\beta^{bpi}(t_{k+1},\delta/2) \sum_{s,a,h}\frac{p^{\pi^\star}_h(s,a)^2}{n_h^{k}(s,a)}} \nonumber\\
    &\stackrel{(c)}{\leq} 2\sqrt{\beta^{bpi}(t_{k+1},\delta/2) H 2^{-k}}\nonumber\\
    &= 2\sqrt{\beta^{bpi}(t_{k}+ d_{k+1},\delta/2) H 2^{-k}}\nonumber\\
    &\stackrel{(d)}{\leq} 2\sqrt{\beta^{bpi}(t_{k}+ SAH2^{k+1},\delta/2) H 2^{-k}},
\end{align}
where (a) uses the event $\cE_{bpi}$ for the reward $\widetilde{r}_\ell(s',a') = \mathds{1}\big((\ell, s',a') = (h,s,a)\big)$, (b) uses the facts that $t\mapsto \beta(t,\delta)$ is non-decreasing and $n_h^{k+1}(s,a) \geq n_h^{k}(s,a)$, (c) uses the induction hypothesis which yields that $n_h^{k}(s,a) \geq 2^{k} p^{\pi^\star}_h(s,a)$ and (d) uses Lemma~\ref{fact1}. Similarly we have that 
\begin{align}\label{ineq:PRINCIPLE-correctness-2}
 \big|p^{\pi^\star}_h(s,a) - \widehat{p}^{\pi^\star, k}_h(s,a)\big| \leq  \sqrt{\beta^{bpi}(t_{k}+ SAH2^{k+1},\delta/2) H 2^{-k}}  
\end{align}
Now thanks to Lemma~\ref{fact1}, we know that for all $(h,s,a), n_h^{k+1}(s,a) - n_h^{k}(s,a) = n_h(s,a,\cD_{k+1}) \geq c^{k+1}_h(s,a)$. Plugging the definition of $c^{k+1}$ (Line 8 of Algorithm \ref{alg:PRINCIPLE}) we get that, 
\begin{align}\label{ineq:PRINCIPLE-correctness-3}
  n_h^{k+1}(s,a) &\geq  2^{k+1} \min\big(\sup_{\widehat{\rho}\in \Omega^{k}} \widehat{\rho}_h(s,a) + 2\sqrt{H\beta^{bpi}(t_{k}+ SAH2^{k+1},\delta/2) 2^{-k}},\ 1\big) \nonumber\\
  &\stackrel{(a)}{\geq} 2^{k+1} \min\big(\widehat{p}^{\pi^\star, k}_h(s,a) + 2\sqrt{H\beta^{bpi}(t_{k}+ SAH2^{k+1},\delta/2) 2^{-k}},\ 1 \big)\nonumber\\
  &\stackrel{(b)}{\geq}  2^{k+1}\max\big(\widehat{p}^{\pi^\star, k+1}_h(s,a),\ p^{\pi^\star}_h(s,a) \big),
\end{align}
where (a) uses that, by the induction hypothesis, $\widehat{p}^{\pi^\star, k} \in \Omega^{k}$ and (b) uses (\ref{ineq:PRINCIPLE-correctness-1}) along with (\ref{ineq:PRINCIPLE-correctness-2}). In particular we have proved that $\max_{h,s,a} \widehat{p}^{\pi^\star, k+1}_h(s,a)/n_h^{k+1}(s,a) \leq 2^{-(k+1)}$. Now it remains to show that $(\widehat{p}^{\pi^\star, k+1})^\top \widehat{r}^{k+1} \geq \underline{V}_1^{k+1}$. Let us consider $\widetilde{\rho}$ achieving the supremum in the definition of $\underline{V}_1^{k+1}$, i.e. 
\begin{align*}
\widetilde{\rho} \in \argmax_{\substack{\widehat{\rho} \in \Omega(\widehat{p}^{k+1}),\\ \max\limits_{h,s,a} \widehat{\rho}_h(s,a)/n_h^{k+1}(s,a) \leq 2^{-(k+1)}}} \widehat{\rho}^\top \widehat{r}^{k+1},
\end{align*}
and let $\widetilde{\pi}$ be a policy corresponding to $\widetilde{\rho}$\footnote{i.e. $\widetilde{\pi}$ is the policy obtained by renormalization of $\widetilde{\rho}$.}. Then we have that 
\begin{align}\label{ineq:PRINCIPLE-correctness-4}
  (\widehat{p}^{\pi^\star, k+1})^\top \widehat{r}^{k+1} &\stackrel{(a)}{\geq} V_1^\star - \sqrt{\beta^{bpi}(t_{k+1},\delta/2) \sum_{s,a,h}\frac{p^{\pi^\star}_h(s,a)^2}{n_h^{k+1}(s,a)}}\nonumber\\
  &\geq V_1^{\widetilde{\pi}} - \sqrt{\beta^{bpi}(t_{k+1},\delta/2) \sum_{s,a,h}\frac{p^{\pi^\star}_h(s,a)^2}{n_h^{k+1}(s,a)}}\nonumber\\
  &\stackrel{(b)}{\geq} \widetilde{\rho}^\top \widehat{r}^{k+1} - \sqrt{\beta^{bpi}(t_{k+1},\delta/2) \sum_{s,a,h}\frac{\widetilde{\rho}_h(s,a)^2}{n_h^{k+1}(s,a)}} - \sqrt{\beta^{bpi}(t_{k+1},\delta/2) \sum_{s,a,h}\frac{p^{\pi^\star}_h(s,a)^2}{n_h^{k+1}(s,a)}}\nonumber\\
  &\stackrel{(c)}{\geq}  \widetilde{\rho}^\top \widehat{r}^{k+1} - 2\sqrt{2^{-(k+1)} H\beta^{bpi}(t_{k+1},\delta/2)} \nonumber\\
  &= \underline{V}_1^{k+1}
\end{align}
where (a) uses the event $\cE_{bpi}$ for policy $\pi^\star$, (b) uses the same event combined with the fact that $\widetilde{\rho} = \widehat{p}^{\widetilde{\pi},k+1}$, and (c) uses (\ref{ineq:PRINCIPLE-correctness-3}) and the fact that by definition of $\widetilde{\rho}, \max\limits_{h,s,a} \widetilde{\rho}_h(s,a)/n_h^{k+1}(s,a) \leq 2^{-(k+1)}$. Now combining (\ref{ineq:PRINCIPLE-correctness-3}) with (\ref{ineq:PRINCIPLE-correctness-4}) gives that $\widehat{p}^{\pi^\star, k+1} \in \Omega^{k+1}$. This finishes the proof.
\end{proof}

\subsubsection{Correctness}
\begin{lemma}
Under the good event, if PRINCIPLE stops then the recommended policy satisfies $V_1^{\widehat{\pi}} \geq V_1^\star - \epsilon$.
\end{lemma}
\begin{proof}
Suppose that PRINCIPLE stops at phase $k\geq 1$. Let $\pi^\star$ be any optimal policy and recall the definition $\widehat{\rho}^\star = \argmax_{\widehat{\rho} \in \Omega^{k}} \widehat{\rho}^\top \widehat{r}^{k}$ with ties broken arbitrarily. We have that 
\begin{align*}
V_1^{\widehat{\pi}} &\stackrel{(a)}{\geq}  (\widehat{\rho}^\star)^\top \widehat{r}^{k} -\sqrt{\beta^{bpi}(t_{k},\delta/2) \sum_{s,a,h}\frac{\widehat{\rho}^\star_h(s,a)^2}{n_h^{k}(s,a)}}\nonumber\\
&\stackrel{(b)}{\geq} (\widehat{p}^{\pi^\star, k})^\top \widehat{r}^{k} -\sqrt{\beta^{bpi}(t_{k},\delta/2) \sum_{s,a,h}\frac{\widehat{\rho}^\star_h(s,a)^2}{n_h^{k}(s,a)}}\nonumber\\
&\stackrel{(c)}{\geq} V_1^\star -\sqrt{\beta^{bpi}(t_{k},\delta/2) \sum_{s,a,h}\frac{\widehat{p}^{\pi^\star, k}_h(s,a)^2}{n_h^{k}(s,a)}} - \sqrt{\beta^{bpi}(t_{k},\delta/2) \sum_{s,a,h}\frac{\widehat{\rho}^\star_h(s,a)^2}{n_h^{k}(s,a)}}\nonumber\\
&\stackrel{(d)}{\geq} V_1^\star - 2\sqrt{2^{-k}H\beta^{bpi}(t_{k},\delta/2)} \stackrel{(e)}{\geq} V_1^\star - \epsilon,
\end{align*}
where (a) uses the event $\cE_{bpi}$ for policy $\widehat{\pi}$ and the fact that $\widehat{\rho}^\star = \widehat{p}^{\widehat{\pi}, k}$, (b) uses the definition of $\widehat{\rho}^\star$ and the fact that, by Lemma \ref{lem:PRINCIPLE-optimal-never-elim}, $\widehat{p}^{\pi^\star, k} \in \Omega^k$, (c) uses the event $\cE_{bpi}$ for the policy $\pi^\star$, and (d) uses that for all $\rho \in \Omega^k, \max_{h,s,a} \rho_h(s,a)/n_h^{k}(s,a) \leq 2^{-k}$ and (e) uses the stopping condition of PRINCIPLE (Line 20 of Algorithm \ref{alg:PRINCIPLE}).
\end{proof}

\subsubsection{Upper bound on the number of phases}
\begin{lemma}\label{lem:new-PRINCIPLE-final-phase}
Define the index of the final phase of PRINCIPLE, $\kappa_f := \inf \big\{k\in \bN_{+}: \sqrt{2^{2-k}H\beta^{bpi}(t_{k},\delta/2)} \leq \epsilon \big\}$. Further let $\tau$ denote the number of episodes played by the algorithm. Then under the good event, it holds that $\kappa_f < \infty$ and
\begin{align*}
     2^{\kappa_f} \leq \frac{8H\beta^{bpi}(\tau,\delta/2)}{\epsilon^2}.
\end{align*}
\end{lemma}
\begin{proof}
To prove that $\kappa_f$ is finite we write
\begin{align}\label{ineq:final-phase-finite-1}
    t_k &= \sum_{j=0}^k d_j\nonumber\\
    &\leq d_0 + SAH \sum_{j=1}^k 2^{j}\nonumber\\
    &\leq \widetilde{\cO}\bigg(\varphi^\star(\mathds{1}) SAH^{2} \big(\log(4/\delta) + S\big) \bigg)+ SAH 2^{k+1},
\end{align}
where we have used the coverage event $\cE_{cov}$ and Lemma~\ref{fact1} to upper bound $d_0$ and $(d_k)_{1\leq j\leq k}$ respectively. This means that $t_k = \cO_{k\to\infty}\big( 2^k \big)$. Now recall that
\begin{align}\label{ineq:final-phase-finite-2}
    \beta^{bpi}(t,\delta) := 16H^2\log(1/\delta)+  96SAH^3 \log(1+t).
\end{align}
Combining (\ref{ineq:final-phase-finite-1}) and (\ref{ineq:final-phase-finite-2}) gives that 
\begin{align*}
   \beta^{bpi}(t_{k},\delta/2) = o_{k\to\infty}\big(2^k\big).
\end{align*}
Therefore $\kappa_f = \inf \big\{k\in \bN_{+}: \sqrt{2^{2-k}H\beta^{bpi}(t_{k},\delta/2)} \leq \epsilon \big\}$ is indeed finite. The proof of the second statement is straightforward by noting that $\kappa_f-1$ does not satisfy the stopping condition (Line 12 in Algorithm \ref{alg:PRINCIPLE}) and using the (crude) upper bound $t_{\kappa_f-1} \leq \tau$.
\end{proof}

\begin{lemma}\textsc{(Upper bound on phases where a suboptimal policy is active)}\label{lem:new-PRINCIPLE-elimination-phase}
Let $\pi$ be any suboptimal policy and $k$ such that PRINCIPLE did not stop at phase $k$ and $\widehat{p}^{\pi,k} \in \Omega^k$. Further let $\tau$ denote the number of episodes played by the algorithm. Then under the good event, we have the inequality
\begin{align*}
     2^{k} \leq \frac{16H\beta^{bpi}(\tau,\delta/2)}{\max(\epsilon, \Delta(\pi))^2},
\end{align*}
where $\Delta(\pi) := V_1^\star(s_1 ; r) - V_1^{\pi}(s_1 ; r)$ denotes the policy gap of $\pi$.
\end{lemma}
\begin{proof}
Let $\pi^\star$ be any optimal policy. Then we have
\begin{align*}
 V_1^\star - \sqrt{\beta^{bpi}(t_k,\delta/2) \sum_{s,a,h}\frac{\widehat{p}^{\pi^\star,k}_h(s,a)^2}{n_h^{k}(s,a)}} &\stackrel{(a)}{\leq} (\widehat{p}^{\pi^\star,k})^\top \widehat{r}^{k} \\
 &\stackrel{(b)}{\leq}  \sup_{\substack{\widehat{\rho} \in \Omega(\widehat{p}^{k}),\\ \max\limits_{h,s,a} \widehat{\rho}_h(s,a)/n_h^{k}(s,a) \leq 2^{-k}}} \widehat{\rho}^\top \widehat{r}^{k} \\ 
 &= \underline{V}_1^{\star, k} + \sqrt{2^{2-k}H\beta^{bpi}(t_k,\delta/2)}\\
 &\stackrel{(c)}{\leq} (\widehat{p}^{\pi,k})^\top \widehat{r}^{k} + \sqrt{2^{2-k}H\beta^{bpi}(t_k,\delta/2)}\\
 &\stackrel{(d)}{\leq} V_1^\pi + \sqrt{\beta^{bpi}(t_k,\delta/2) \sum_{s,a,h}\frac{\widehat{p}^{\pi,k}_h(s,a)^2}{n_h^{k}(s,a)}}+ \sqrt{2^{2-k}H\beta^{bpi}(t_k,\delta/2)},
\end{align*}
where (a) uses the event $\cE_{bpi}$ for $\pi^\star$, (b) uses the definition of $\Omega^k$ along with Lemma \ref{lem:PRINCIPLE-optimal-never-elim} which gives that $\widehat{p}^{\pi^\star,k} \in \Omega^k$, (c) uses our assumption that $\widehat{p}^{\pi,k} \in \Omega^k$ and (d) uses the event $\cE_{bpi}$ for policy $\pi$.  Rewriting the inequality above we get that
\begin{align}\label{ineq:elimination-PRINCIPLE}
    \Delta(\pi) &= V_1^{\star} - V_1^\pi \nonumber\\
    &\leq \sqrt{\beta^{bpi}(t_k,\delta/2) \sum_{s,a,h}\frac{\widehat{p}^{\pi^\star,k}_h(s,a)^2}{n_h^{k}(s,a)}} + \sqrt{\beta^{bpi}(t_k,\delta/2) \sum_{s,a,h}\frac{\widehat{p}^{\pi,k}_h(s,a)^2}{n_h^{k}(s,a)}}+ \sqrt{2^{2-k}H\beta^{bpi}(t_k,\delta/2)} \nonumber\\
     &\leq 2\sqrt{2^{-k}H\beta^{bpi}(t_k,\delta/2)} + \sqrt{2^{2-k}H\beta^{bpi}(t_k,\delta/2)} = 4\sqrt{2^{-k}H\beta^{bpi}(t_k,\delta/2)},
\end{align}
where the last inequality uses the fact that $\widehat{p}^{\pi^\star,k} \in \Omega^k$ by Lemma \ref{lem:PRINCIPLE-optimal-never-elim} and that $\widehat{p}^{\pi,k} \in \Omega^k$ by assumption.
Therefore, using a crude bound $t_k \leq \tau$ we get that
\begin{align*}
    2^{k} \leq  \frac{16H\beta^{bpi}(\tau,\delta/2)}{\Delta(\pi)^2}.
\end{align*}
Combining the result above with Lemma \ref{lem:new-PRINCIPLE-final-phase} and the fact that $k \leq \kappa_f$ yields the final result.
\end{proof}

\subsubsection{Upper bound on the phase length }

\begin{lemma}\label{lem:new-PRINCIPLE-phase-length}
Let $T_k$ denote the number of episodes played by PRINCIPLE during phase $k\geq 1$. Then we have
\begin{align*}
T_k &\leq 256H\beta^{bpi}(\tau,\delta/2)k \varphi^\star\bigg(\bigg[\sup_{\pi\in \Pi} \frac{p^\pi_h(s,a)}{\max(\epsilon, \Delta(\pi))^2} \bigg]_{h,s,a} \bigg) \\& \quad + 48k\sqrt{H\beta^{bpi}(t_{k-1}+ SAH2^{k-1},\delta/2) 2^{k}}\varphi^\star(\mathds{1})\\
&\quad + \widetilde{\cO}\bigg(k\varphi^\star(\mathds{1}) SAH^{2} \big(\log(4(k+1)^2/\delta) + S \big) \bigg).
\end{align*}
\end{lemma}
\begin{proof}
Define $m_k = \log_2\big(\frac{\max_{s,a,h}c_h^{k}(s,a)}{\min_{s,a,h}c_h^{k}(s,a)\vee1}\big) \vee 1$. Under the good event we have 
\begin{align}\label{ineq:phase-length-1}
    T_k &\leq  64m_k \varphi^\star(c^{k}) + \widetilde{\cO}\bigg(m_k\varphi^\star(\mathds{1}) SAH^{2} \big(\log(4(k+1)^2/\delta) + S \big) \bigg) \nonumber\\
    &{\leq} 64k \varphi^\star(c^{k}) + \widetilde{\cO}\bigg(k\varphi^\star(\mathds{1}) SAH^{2} \big(\log(4(k+1)^2/\delta) + S \big) \bigg),
\end{align}
where the last inequality uses the fact that for all $(h,s,a), c_h^{k}(s,a)\leq 2^k$. Now we simplify the expression of $\varphi^\star(c^{k})$ as follows
\begin{align}\label{ineq:flow-phase-1}
  \varphi^\star(c^{k}) &= \varphi^\star\bigg( \bigg[2^{k}\min\big(\sup_{\widehat{\rho}\in \Omega^{k-1}} \widehat{\rho}_h(s,a) + 2\sqrt{H\beta^{bpi}(t_{k-1}+ SAH2^{k-1},\delta/2) 2^{1-k}},\ 1 \big) \bigg]_{h,s,a} \bigg) \nonumber\\
  &\leq \varphi^\star\bigg( \bigg[\sup_{\substack{\pi\in \Pi^S: \\
  \widehat{p}^{\pi, k-1}\in \Omega^{k-1}} } 2^k \widehat{p}_h^{\pi, k-1}(s,a) + 2\sqrt{H\beta^{bpi}(t_{k-1}+ SAH2^{k-1},\delta/2) 2^{k+1}}\bigg]_{h,s,a} \bigg),
\end{align}
where we have used that $\varphi^\star(c) \leq \varphi^\star(c')$ if $ \forall (h,s,a)\ c_h(s,a) \leq  c_h'(s,a)$. Now fix a policy $\pi$ in the set $\{\pi\in \Pi^S: \widehat{p}^{\pi, k-1}\in \Omega^{k-1}\}$. Using the event $\cE_{bpi}$ for the rewards $\widetilde{r}_\ell(s',a') = \mathds{1}\big((\ell, s',a') = (h,s,a)\big)$ we have that for all $(h,s,a)$
\begin{align*}
    2^k\widehat{p}^{\pi, k-1}_h(s,a)  &\leq 2^k p^{\pi}_h(s,a)+ 2^k\sqrt{\beta^{bpi}(t_{k-1},\delta/2) \sum_{s',a',\ell}\frac{\widehat{p}^{\pi, k-1}_\ell(s',a')^2}{n_\ell^{k-1}(s',a')}} \\
    &\stackrel{(a)}{\leq} 2^k p^{\pi}_h(s,a)+ 2^k\sqrt{\beta^{bpi}(t_{k-1},\delta/2) H 2^{1-k}}\\
    &\leq 2^k p^{\pi}_h(s,a)+ \sqrt{H\beta^{bpi}(t_{k-1}+ SAH2^{k-1},\delta/2) 2^{k+1}}\\
    &\stackrel{(b)}{\leq} \frac{32H\beta^{bpi}(\tau,\delta/2)p^{\pi}_h(s,a)}{\max(\epsilon, \Delta(\pi))^2}+ \sqrt{H\beta^{bpi}(t_{k-1}+SAH2^{k-1},\delta/2) 2^{k+1}},
\end{align*}
where (a) uses that $\max\limits_{s',a',\ell} \frac{\widehat{p}^{\pi, k-1}_\ell(s',a')}{n_\ell^{k-1}(s',a')} \leq 2^{1-k}$ since $\widehat{p}^{\pi, k-1}\in \Omega^{k-1}$ and (b) uses Lemma \ref{lem:new-PRINCIPLE-elimination-phase}. Plugging the inequality above into (\ref{ineq:flow-phase-1}) we get that 
\begin{align}\label{ineq:phase-length-2}
 \varphi^\star(c^{k}) &\leq \varphi^\star\bigg( \bigg[\sup_{\pi\in\Pi^S} \frac{32H\beta^{bpi}(\tau,\delta/2) p^{\pi}_h(s,a)}{\max(\epsilon, \Delta(\pi))^2} + 3\sqrt{H\beta^{bpi}(t_{k-1}+SAH2^{k-1},\delta/2)2^{k+1}}\bigg]_{h,s,a} \bigg)\nonumber\\
 &\leq 32H\beta^{bpi}(\tau,\delta/2) \varphi^\star\bigg( \bigg[\sup_{\pi\in\Pi^S} \frac{p^{\pi}_h(s,a)}{\max(\epsilon, \Delta(\pi))^2}\bigg]_{h,s,a} \bigg)\nonumber \\ & \quad+ 3\sqrt{H\beta^{bpi}(t_{k-1}+SAH2^{k-1},\delta/2)2^{k+1}}\varphi^\star(\mathds{1}),
\end{align}
where we used Lemma \ref{lem:flow-linear} in the last step. Combining (\ref{ineq:phase-length-1}) and (\ref{ineq:phase-length-2}) finishes the proof.
\end{proof}

\subsubsection{Total sample complexity}
\begin{theorem}
With probability at least $1-\delta$, the total sample complexity of PRINCIPLE satisfies
\begin{align*}
\tau &\leq \widetilde{\cO}\bigg( (H^3 \log(1/\delta) + SAH^4) \bigg[\varphi^\star\bigg(\bigg[\sup_{\pi\in \Pi} \frac{p^\pi_h(s,a)}{\max(\epsilon, \Delta(\pi))^2} \bigg]_{h,s,a} \bigg) + \frac{\varphi^\star(\mathds{1})}{\epsilon} +  \varphi^\star(\mathds{1}) \bigg] \bigg),
\end{align*}
where $\widetilde{\cO}$ hides poly-logarithmic factors in $S,A,H, \epsilon, \log(1/\delta)$ and $\varphi^\star(\mathds{1})$ and $\Delta(\pi) := V_1^\star(s_1 ; r) - V_1^{\pi}(s_1 ; r)$ denotes the policy gap of $\pi$.
\end{theorem}
\begin{proof}
We write
\begin{align*}
    \tau &= \sum_{k=0}^{\kappa_f} T_k\\
    &\leq \widetilde{\cO}\bigg(\varphi^\star(\mathds{1})^2 SAH^{2} \big(\log(4/\delta) + S \big) \bigg) + \sum_{k=1}^{\kappa_f} T_k\\
    &\leq  \underbrace{\sum_{k=1}^{\kappa_f} 256H\beta^{bpi}(\tau,\delta/2) k \varphi^\star\bigg(\bigg[\sup_{\pi\in \Pi} \frac{p^\pi_h(s,a)}{\max(\epsilon, \Delta(\pi))^2} \bigg]_{h,s,a} \bigg)}_{:=A} + \underbrace{\sum_{k=1}^{\kappa_f} 48k\sqrt{H\beta^{bpi}(t_{k-1}+ SAH2^{k-1},\delta/2) 2^{k}}\varphi^\star(\mathds{1})}_{:=B}\\
    &+ \underbrace{\widetilde{\cO}\bigg(\sum_{k=1}^{\kappa_f} k\varphi^\star(\mathds{1}) SAH^{2} \big(\log(4(k+1)^2/\delta) + S \big) \bigg)}_{:=C},
\end{align*}
where we have used Lemma \ref{lem:new-PRINCIPLE-phase-length}. Now we bound each term separately. First note that
\begin{align*}
    A &\leq 256H\beta^{bpi}(\tau,\delta/2)  \varphi^\star\bigg(\bigg[\sup_{\pi\in \Pi} \frac{p^\pi_h(s,a)}{\max(\epsilon, \Delta(\pi))^2} \bigg]_{h,s,a} \bigg) \kappa_f^2\\
    &\stackrel{(a)}{\leq} 256H\beta^{bpi}(\tau,\delta/2) \varphi^\star\bigg(\bigg[\sup_{\pi\in \Pi} \frac{p^\pi_h(s,a)}{\max(\epsilon, \Delta(\pi))^2} \bigg]_{h,s,a} \bigg) \log_2^2\big(8H\beta^{bpi}(\tau,\delta/2)/\epsilon^2\big)\\
    &\stackrel{(b)}{\leq}\cO\bigg([H^3\log(1/\delta)+SAH^4 \log(1+\tau)]\varphi^\star\bigg(\bigg[\sup_{\pi\in \Pi} \frac{p^\pi_h(s,a)}{\max(\epsilon, \Delta(\pi))^2} \bigg]_{h,s,a} \bigg)\log_2^2\big(8H\beta^{bpi}(\tau,\delta/2)/\epsilon^2\big)  \bigg),
\end{align*}
where (a) uses Lemma \ref{lem:new-PRINCIPLE-final-phase} and (b) uses the definition of $\beta^{bpi}$. Similarly
\begin{align*}
    B&\leq 48\sqrt{H\beta^{bpi}(\tau+ SAH2^{\kappa_f-1},\delta/2) 2^{\kappa_f}}\varphi^\star(\mathds{1}) \kappa_f^2\\
    &\stackrel{(a)}{\leq} 48\sqrt{\frac{4H^2\beta^{bpi}(\tau+ SAH2^{\kappa_f-1},\delta/2)\beta^{bpi}(\tau,\delta/2)}{\epsilon^2}}\varphi^\star(\mathds{1})\log_2^2\big(8H\beta^{bpi}(\tau,\delta/2)/\epsilon^2\big)\\
    &\leq \frac{48H}{\epsilon}\beta^{bpi}(\tau+ SAH2^{\kappa_f-1},\delta/2)\varphi^\star(\mathds{1})\log_2^2\big(8H\beta^{bpi}(\tau,\delta/2)/\epsilon^2\big)\\
    &\stackrel{(b)}{\leq} \cO\bigg(\frac{\varphi^\star(\mathds{1})}{\epsilon} \bigg[H^3\log(1/\delta)+SAH^4 \log\bigg(1+\tau +\frac{4SAH^2 \beta^{bpi}(\tau,\delta/2)}{\epsilon^2}\bigg)\bigg] \log_2^2\big(8H\beta^{bpi}(\tau,\delta/2)/\epsilon^2\big)  \bigg),
\end{align*}
where (a) and (b) use Lemma \ref{lem:new-PRINCIPLE-final-phase}. Finally
\begin{align*}
    C&\leq \widetilde{\cO}\bigg(\varphi^\star(\mathds{1}) SAH^{2} \big(\log(4(\kappa_f+1)^2/\delta) + S \big)\kappa_f^2 \bigg)\\
    &\leq  \widetilde{\cO}\bigg(\varphi^\star(\mathds{1}) SAH^{2} \big[\log\big(\frac{4\log_2^2\big(8H\beta^{bpi}(\tau,\delta/2)/\epsilon^2\big)}{\delta}\big) + S \big]\log_2^2\big(8H\beta^{bpi}(\tau,\delta/2)/\epsilon^2\big) \bigg),
\end{align*}
where we have used Lemma \ref{lem:new-PRINCIPLE-final-phase} again. Combining the three inequalities with the definition of $\beta^{bpi}$ we get that
\begin{align*}
    \tau &\leq \cO\bigg( (H^3 \log(1/\delta) + SAH^4) \bigg[\varphi^\star\bigg(\bigg[\sup_{\pi\in \Pi} \frac{p^\pi_h(s,a)}{\max(\epsilon, \Delta(\pi))^2} \bigg]_{h,s,a} \bigg) + \frac{\varphi^\star(\mathds{1})}{\epsilon} +  \varphi^\star(\mathds{1})\bigg]\\
    &\qquad \times \textrm{polylog}(\tau,S,A,H,\varphi^\star(\mathds{1}), \epsilon,\log(1/\delta)) \bigg).
\end{align*}
Solving for $\tau$ yields
\begin{align*}
    \tau &\leq \widetilde{\cO}\bigg( (H^3 \log(1/\delta) + SAH^4) \bigg[\varphi^\star\bigg(\bigg[\sup_{\pi\in \Pi} \frac{p^\pi_h(s,a)}{\max(\epsilon, \Delta(\pi))^2} \bigg]_{h,s,a} \bigg) + \frac{\varphi^\star(\mathds{1})}{\epsilon} +  \varphi^\star(\mathds{1}) \bigg] \bigg),
\end{align*}
where $\widetilde{\cO}$ hides poly-logarithmic factors in $S,A,H, \epsilon, \log(1/\delta)$ and $\varphi^\star(\mathds{1})$.
\end{proof}

\revision{
\begin{remark}[Reachability]\label{rm:reachability_bpi}
 While for the PCE algorithm we were able to reduce the sample complexity by ignoring states that are hard to reach (which also allows using PCE when Assumption~\ref{asm:reachability} is violated), we did not manage to propose a similar improvement for PRINCIPLE. This is because in reward-free exploration it is sufficient to guarantee that the \emph{true confidence intervals} that depend on the visitation probabilities \emph{under the true MDP} are small, i.e., $\sqrt{\betarf(t_k,\delta)\sum_{(h,s,a)}\frac{p_h^{\pi}(s,a)^2}{n_h^k(s,a)}} \leq 2^k$. This allows us to filter out all $(h,s,a)$ for which $\sup_{\pi} p_h^\pi(s,a) \leq \cO(\epsilon/SH^2)$, by arguing that their contribution to the true confidence interval is negligible. In contrast, the analysis of PRINCIPLE crucially relies on concentrating the values of policies by minimizing \emph{their empirical confidence intervals}, i.e., $\sqrt{\beta^{bpi}(t_k,\delta)\sum_{(h,s,a)}\frac{\widehat{p}_h^{\pi, k}(s,a)^2}{n_h^k(s,a)}} \leq 2^k$. We do not see a straightforward way to ignore the contribution of hard-to-reach states to these empirical confidence intervals. 
\end{remark}}

\subsection{Comparison with other BPI-algorithms}\label{app:comparison_bpi}

In this section we compare PRINCIPLE with other algorithms for Best-Policy Identification algorithms that enjoy problem-dependent guarantees, namely PEDEL \citep{Wagenmaker22linearMDP} and MOCA \citep{wagenmaker21IDPAC}. Recalling that $\Delta(\pi) = V_1^\star(s_1) - V_1^{\pi}(s_1)$ denotes the policy gap of $\pi$, we first note that by Theorem \ref{thm:PRINCIPLE-complexity}, the leading term in the sample complexity of PRINCIPLE in the small $(\epsilon, \delta)$ regime is $\textrm{PRINCIPLE}(\cM, \epsilon)\log(1/\delta)$ where
\begin{align*}
    \textrm{PRINCIPLE}(\cM, \epsilon) &:= H^3\varphi^\star\left(\left[\sup_{\pi\in \PiS} \frac{p^\pi_h(s,a)}{\max(\epsilon, \Delta(\pi) )^2} \right]_{h,s,a} \right).
\end{align*} 
We will now compare this term with the leading terms in the sample complexities of PEDEL and MOCA respectively, in the same asymptotic regime.
\subsubsection{Comparison with PEDEL}
Define the minimum policy gap 
\begin{align*}
    \Delta_{\min}(\PiD):=
    \begin{cases}
         \min_{\pi\neq \pi^\star} \Delta(\pi),\quad \textrm{if the optimal policy $\pi^\star$ is unique}\\
         0,\quad \textrm{otherwise}.
    \end{cases}
\end{align*}
Then instantiating Theorem 1 from \cite{Wagenmaker22linearMDP} for our setting of tabular MDPs (i.e. with $d = SAH$ and $\Pi = \PiD$), we see that the sample complexity achieved by PEDEL satisfies
\begin{align*}
   &\tau \leq \widetilde{O}\bigg(\textrm{PEDEL}(\cM, \epsilon) (\log(1/\delta) + SH) + \textrm{poly}(SAH, \log(1/\epsilon), \log(1/\delta)) \bigg)\\
&\textrm{where}\quad \textrm{PEDEL}(\cM, \epsilon) := H^4\sum_{h=1}^H \min_{\rho \in \Omega} \max_{\pi\in\PiD} \sum_{s,a} \frac{p_h^\pi(s,a)^2/\rho_h(s,a)}{\max(\epsilon, \Delta(\pi), \Delta_{\min}(\PiD) )^2}.
\end{align*}
Therefore the leading term PEDEL's complexity in the small $(\epsilon, \delta)$ regime is $\textrm{PEDEL}(\cM, \epsilon)\log(1/\delta)$. The next lemma shows that, up to $H$ factors, this rate is always better than the complexity measure achieved by PRINCIPLE. 

\begin{lemma}\label{lem:PRINCIPLEvsPEDEL}
    For any MDP $\cM$, it holds that $\textrm{PEDEL}(\cM, \epsilon) \leq H^2 \textrm{PRINCIPLE}(\cM, \epsilon)$.
\end{lemma}


\begin{proof}
 Fix any $h\in [H], \rho\in \Omega, \pi\in \PiD$. Then we have  
\[
    \sum_{s,a} \frac{p_h^\pi(s,a)^2}{\rho_h(s,a)} \leq \bigg(\max_{s,a,h} \frac{p_h^\pi(s,a)}{\rho_h(s,a)}\bigg) \sum_{s,a} p_h^\pi(s,a)=  \max_{s,a,h} \frac{p_h^\pi(s,a)}{\rho_h(s,a)}.
\]
 Therefore for all $h\in [H]$, using that $\PiD\subset \PiS$ we have
 \begin{align*}
    \min_{\rho \in \Omega} \max_{\pi\in\PiD} \sum_{s,a} \frac{p_h^\pi(s,a)^2/\rho_h(s,a)}{\max(\epsilon, \Delta(\pi), \Delta_{\min}(\PiD) )^2} &\leq \min_{\rho \in \Omega} \max_{\pi\in\PiD} \max_{s,a,h} \frac{  p_h^\pi(s,a)/\rho_h(s,a) }{\max(\epsilon, \Delta(\pi), \Delta_{\min}(\PiD))^2}\\
    &= \min_{\rho \in \Omega} \max_{s,a,h} \max_{\pi\in\PiD}  \frac{ p_h^\pi(s,a)/\rho_h(s,a) }{\max(\epsilon, \Delta(\pi), \Delta_{\min}(\PiD))^2}\\
    &\leq \min_{\rho \in \Omega} \max_{s,a,h} \sup_{\pi\in\PiS} \frac{p_h^\pi(s,a)}{\rho_h(s,a) \max(\epsilon, \Delta(\pi))^2}\\
    &=  \varphi^\star\bigg( \bigg[\sup_{\pi\in\PiS} \frac{p_h^\pi(s,a)}{\max(\epsilon, \Delta(\pi))^2}\bigg]_{h,s,a} \bigg).
 \end{align*}
 Therefore
 \begin{align*}
    \textrm{PEDEL}(\cM, \epsilon) &:= H^4\sum_{h=1}^H \min_{\rho \in \Omega} \max_{\pi\in\PiD} \sum_{s,a} \frac{p_h^\pi(s,a)^2/\rho_h(s,a)}{\max(\epsilon, \Delta(\pi), \Delta_{\min}(\PiD) )^2}\\
    &\leq  H^5\varphi^\star\bigg( \bigg[\sup_{\pi\in\PiS} \frac{p_h^\pi(s,a)}{\max(\epsilon, \Delta(\pi))^2}\bigg]_{h,s,a} \bigg)\\
    &= H^2  \textrm{PRINCIPLE}(\cM, \epsilon).
 \end{align*}
\end{proof}

\subsubsection{Comparison with MOCA}

Let us define the complexity functional,
\begin{align*}
  \textrm{MOCA}(\cM, \epsilon) &:= H^2\sum_{h=1}^H \min_{\rho\in\Omega} \max_{s,a} \frac{1}{\rho_h(s,a)} \min\big(\frac{1}{\widetilde{\Delta}_h(s,a)^2},\ \frac{W_h(s)^2}{\epsilon^2}\big)\ \\& \quad + \frac{H^4 \big|(h,s,a):\ \widetilde{\Delta}_h(s,a)\leq 3\epsilon/W_h(s) \big|}{\epsilon^2},
\end{align*} 
where $W_h(s) := \sup_{\pi} p_h^\pi(s)$ is the reachability of $(h,s)$ and
\begin{align*}
\widetilde{\Delta}_h(s,a) := 
    \begin{cases}
      \min_{b \neq a} V_h^\star(s) - Q_h^\star(s,b) \quad \textrm{if $a$ is the unique optimal action at $(h,s)$,} \\
      V_h^\star(s) - Q_h^\star(s,a) \quad \textrm{otherwise}
    \end{cases}  
\end{align*}
is the value gap of $(h,s,a)$.
Theorem 1 together with Proposition 2 of \cite{wagenmaker21IDPAC} yield that the stopping time of MOCA satisfies
\begin{align*}
    \tau \leq \widetilde{\cO}\bigg(\textrm{MOCA}(\cM, \epsilon)\log(1/\delta) + \frac{\textrm{poly}\big(SAH, \log(1/\epsilon), \log(1/\delta)\big)}{\epsilon} \bigg).
\end{align*}
Therefore we see that $\textrm{MOCA}(\cM, \epsilon)\log(1/\delta)$ is the dominating term in the sample complexity of MOCA in the regime of small $\epsilon$ and small $\delta$. On the other hand, as stated earlier, the leading term in PRINCIPLE's complexity in that regime is $\textrm{PRINCIPLE}(\cM,\epsilon)\log(1/\delta)$. Therefore we compare $\textrm{MOCA}(\cM, \epsilon)$ with $\textrm{PRINCIPLE}(\cM, \epsilon)$ to assess which algorithm is better in this regime.
\begin{lemma}
Fix any $\Delta \in (0,1]$. There exists an MDP $\cM$ where
\begin{align*}
    \textrm{MOCA}(\cM, \epsilon) = \Omega\bigg(\frac{H^5SA}{\epsilon^2}\bigg)\quad \textrm{while}\  \textrm{PRINCIPLE}(\cM, \epsilon) = \cO\bigg(\frac{H^4SA}{\epsilon \Delta} + \frac{H^4\log(S)\log(A)}{\epsilon^2}\bigg). 
\end{align*}  
\end{lemma}

\begin{proof}
Consider the MDP in figure \ref{fig:PRINCIPLEvsMOCA} which consists of an initial state $s_1$ and two sub-MDPs depending on the action taken at step $h=1$. If the learner takes action $a_1$ it receives a reward $\Delta >0$ and makes a transition to a sub-MDP $\cM_1$ for which $|\cS_1| = \log(S), |\cA_1| = \log(A), H_1 = H-1$ and where the rewards can be anything. On the other hand, if it takes action $a_2$ the learner will receive zero reward and make a transition to a sub-MDP $\cM_2$ for which $|\cS_2| = S-\log(S), |\cA_2| = A, H_2 = H-1$, the rewards are equal to zero everywhere and the transitions are deterministic, i.e. $p(s'|s,a) \in \{0,1\}$ for all $(s,a) \in \cS_2\times\cA_2$.
\begin{figure}[H]\label{fig:PRINCIPLEvsMOCA}
    \includegraphics[width = 0.8\linewidth]{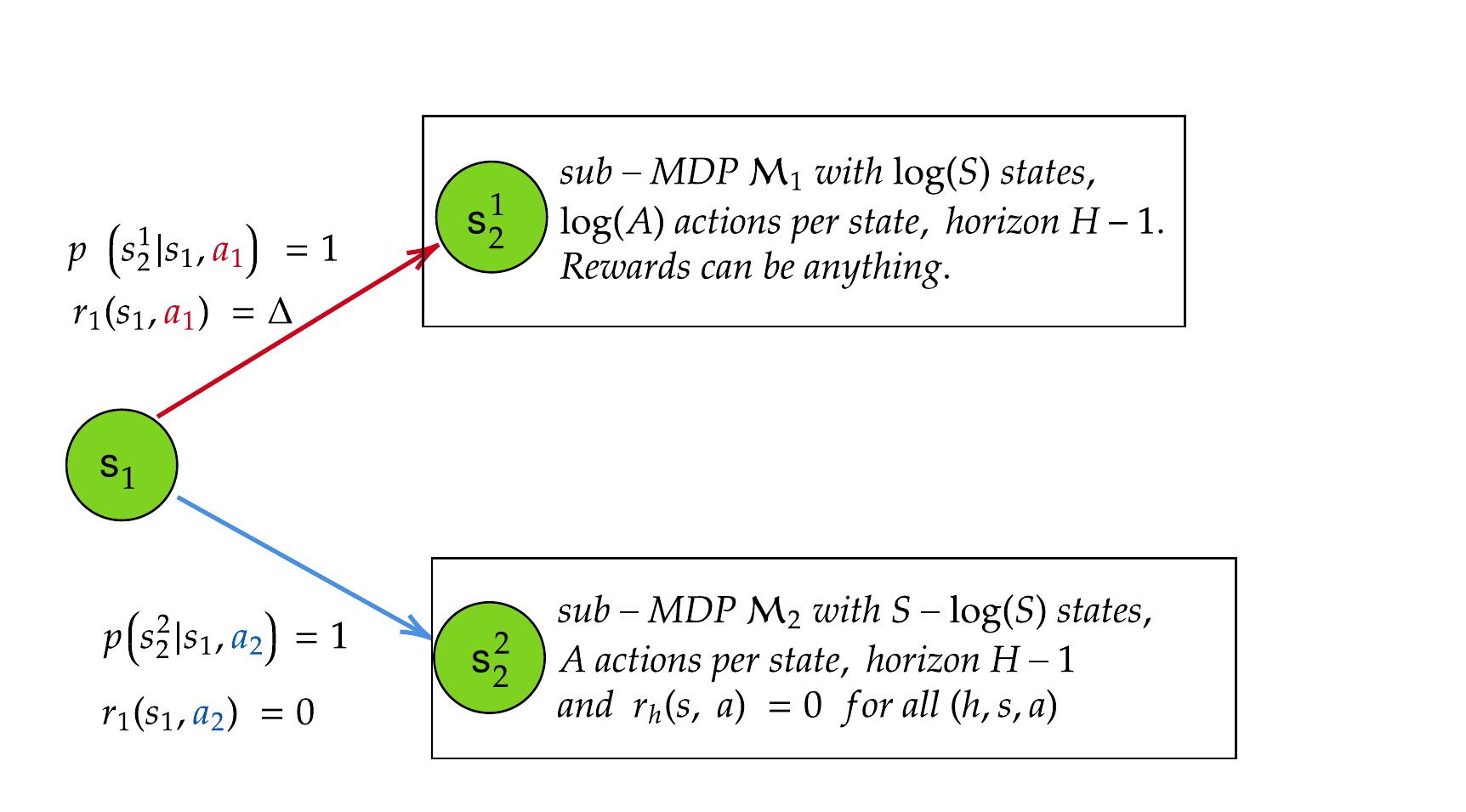}
    \caption{MDP instance with large policy gaps and small value gaps.}
\end{figure}
Note that in this example $\widetilde{\Delta}_h(s,a) = 0$ for all $(h,s,a) \in \cM_2$. Therefore
\begin{align}\label{ineq:MOCA-example}
    \textrm{MOCA}(\cM, \epsilon) &\geq \frac{H^4 \big|(h,s,a):\ \widetilde{\Delta}_h(s,a)\leq 3\epsilon/W_h(s) \big|}{\epsilon^2}\nonumber, \\
    &\geq \frac{H^4 (H-1)(S-\log(S))A}{\epsilon^2}.
\end{align}
On the other hand for all triplets $(h,s,a)$ in the sub-MDP $\cM_2$ we have
\begin{align}\label{ineq:PRINCIPLE-gaps-bound-1}
    \sup_{\pi\in \PiS} \frac{p^\pi_h(s,a)}{\max(\epsilon, \Delta(\pi) )^2} \le \sup_{\pi\in \PiS} \frac{4\pi_1(a_2|s_1)}{(\epsilon + \Delta(\pi) )^2},
\end{align}
where we used that $p^\pi_h(s,a) \le \pi_1(a_2|s_1)$ (since the only path to reach $(h,s,a)$ is by playing action $a_2$ at $s_1$) and that $\max(a,b) \geq (a+b)/2$. Now, by the performance-difference lemma we have
\begin{align*}
    \Delta(\pi) &= \sum_{h,s,a} p_h^{\pi}(s,a)[V_h^\star(s) - Q_h^\star(s,a)] \\
    &\geq p_1^{\pi}(s_1,a_2)[V_1^\star(s_1) - Q_1^\star(s_1,a_2)]  = \pi_1(a_2|s_1)\Delta.
\end{align*}

Plugging this back into (\ref{ineq:PRINCIPLE-gaps-bound-1}), we get
\begin{align*}
    \sup_{\pi\in \PiS} \frac{p^\pi_h(s,a)}{\max(\epsilon, \Delta(\pi) )^2}&\leq \sup_{\pi\in \PiS} \frac{4\pi_1(a_2|s_1)}{(\epsilon + \pi_1(a_2|s_1)\Delta )^2}\\
    &= \sup_{x\in [0,1]} \frac{4x}{(\epsilon + x\Delta )^2} = \frac{1}{\epsilon\Delta} 
\end{align*}
 For triplets $(h,s,a)$ outside of  $\cM_2$ (i.e. either at $s_1$ or in the sub-MDP $\cM_1$ ) we use the crude bound
\begin{align*}
    \sup_{\pi\in \PiS} \frac{p^\pi_h(s,a)}{\max(\epsilon, \Delta(\pi) )^2} \le \frac{ \sup_{\pi\in \PiS} p^\pi_h(s,a)}{\epsilon^2}.
\end{align*}
Therefore
\begin{align}\label{ineq:PRINCIPLE-example}
    &\textrm{PRINCIPLE}(\cM, \epsilon) = H^3\varphi^\star\left(\left[\sup_{\pi\in \PiS} \frac{p^\pi_h(s,a)}{\max(\epsilon, \Delta(\pi) )^2} \right]_{h,s,a} \right)\nonumber\\
    &=  H^3\varphi^\star\left(\left[\sup_{\pi\in \PiS} \frac{p^\pi_h(s,a)}{\max(\epsilon, \Delta(\pi) )^2} (\indi{(h,s,a) \in \cM_2}+ \indi{(h,s,a) \notin \cM_2})\right]_{h,s,a} \right)\nonumber\\
    &\stackrel{(a)}{\leq} H^3\varphi^\star\left(\left[\sup_{\pi\in \PiS} \frac{p^\pi_h(s,a)}{\max(\epsilon, \Delta(\pi) )^2} \indi{(h,s,a) \in \cM_2}\right]_{h,s,a} \right)\nonumber\\
    &\quad + H^3\varphi^\star\left(\left[\sup_{\pi\in \PiS} \frac{p^\pi_h(s,a)}{\max(\epsilon, \Delta(\pi) )^2} \indi{(h,s,a) \notin \cM_2}\right]_{h,s,a} \right)\nonumber\\
    &\leq H^3\varphi^\star\left(\left[\frac{\indi{(h,s,a) \in \cM_2}}{\epsilon \Delta} \right]_{h,s,a} \right) + H^3\varphi^\star\left(\left[\frac{\indi{(h,s,a) \notin \cM_2}\sup_{\pi\in \PiS} p^\pi_h(s,a)}{\epsilon^2} \right]_{h,s,a} \right) \nonumber\\
    &\stackrel{(b)}{\leq} H^3\sum_{(h,s,a) \in \cM_2} \frac{1}{\epsilon \Delta\sup_{\pi\in \PiS} p^\pi_h(s,a) }+ H^3\sum_{(h,s,a) \notin \cM_2} \frac{1}{\epsilon^2}\nonumber \\
    &\stackrel{(c)}{=} \frac{H^3(H-1)(S-\log(S))A}{\epsilon \Delta} + \frac{H^3(H-1)\log(S)\log(A)}{\epsilon^2}
\end{align}
where (a) uses the sub-linearity of the flow from Lemma \ref{lem:flow-linear}, (b) uses the bound on $\phi^\star$ from Lemma \ref{lem:bound-flow-simple} and (c) uses that the sub-MDP $\cM_2$ has deterministic transitions. Combining (\ref{ineq:MOCA-example}) and (\ref{ineq:PRINCIPLE-example}) finishes the proof.
\end{proof}

%% file: appendix/app_visitations.tex
\section{Estimating State Reachability} \label{app:visitations}

Let $\cA^\Pi$ be a regret minimizer that has a small regret for a (fixed) reward function $r$. If we set this reward function to $r_{h'}^{(h,s)}(s',a') = \ind((s' = s,h'=h))$ for a target pair $(h,s)$ intuitively the regret minimizer will visit as much as possible state $s$ in step $h$ and the total reward collected by the algorithm, $n_h^{t}(s)=\sum_{a \in \cA} n_h^{t}(s,a)$, will be close to $t \times W_h(s)$, where the maximum visitation probability $W_h(s)=\max_{\pi} p_h^{\pi}(s)$ is actually the optimal value function in the MDP with reward function $r^{(h,s)}$. The empirical number of visitations can thus be used to estimate the unknown visitation probability. 

This idea is already at the heart of the initialization phase of the MOCA algorithm, which relies on repeatedly running the Euler algorithm. We propose a slightly simpler version below, that doesn't need any restart and relies on a generic algorithm $\cA^{\Pi}$ satisfying some first-order regret bound scaling with a quantity $\cR_{\delta}^{\Pi}(T)$, as specified in the following theorem. \visitalg($(h,s);\varepsilon_0,\delta)$ outputs a valid confidence interval $[\underline{W}_h(s), \overline{W}_h(s)]$ on the value of $W_h(s)$, which can be further used to eliminate all $(h,s)$ whose maximum visitation probability is (slightly) smaller than a target $\varepsilon_0$.


\begin{algorithm}[!ht]
\caption{\visitalg$((h,s) ; \varepsilon_0,\delta)$}\label{alg:visitations}
\begin{algorithmic}[1]
\STATE \textbf{Input:} Step $h$, state $s$, threshold $\varepsilon_0 > 0$, failure probability $\delta\in(0,1)$, regret minimizer $\cA^{\Pi}$ 
\STATE \textbf{Output:}  An interval $[\underline{W}_h(s), \overline{W}_h(s)]$
\STATE Compute $T = T(\varepsilon_0,\delta) =  \inf \left\{T \in \N : 4\cR^{\Pi}_{\delta/2}(T) + 6\log\left(\frac{4}{\delta}\right) \leq \tfrac{\varepsilon_0}{4} T\right\}$
\STATE Collect $T$ episodes $\{(s_1^{t},a_1^{t},\dots,s_H^{t},a_H^{t})\}_{t \leq T}$ using $\cA^{\Pi}$ with reward $\widetilde{r}_{h'}(s',a') = \ind((s' = s,h'=h))$ and confidence $1-\delta/2$
\STATE Let $n_h^{T}(s) = \sum_{t=1}^{T}\ind(s_h^{t} = s)$ be the number of visits of $(h,s)$ 
\STATE Define $\underline{W}_h(s) = \left(\frac{n_h^{T}(s)}{2T} - \frac{\varepsilon_0}{16}\right) \vee 0$ and $\overline{W}_h(s) =  \left(\frac{2n_h^{T}(s)}{T} + \frac{\epsilon_0}{4}\right) \wedge 1$
\end{algorithmic}
\end{algorithm}


\begin{theorem}\label{thm:estimate-visitations}
Assume that, for all $(h,s)$, when $\cA^{\Pi}$ is run for the reward function $r=r^{(h,s)}$ and confidence $1-\delta$ up to some horizon $T\in \N$, with probability larger than $1-\delta$, 
\begin{align}\label{eq:regret-not-anytime}
      \ \sum_{t=1}^{T} V_1^{\star}\left(s_1 ; r\right) - \sum_{t=1}^{T}V_1^{\pi^{t}}\left(s_1 ; r\right) \leq \sqrt{\cR^{\Pi}_\delta(T) TV^\star(s_1;r)} + \cR^{\Pi}_\delta(T).
\end{align}
For all $(h,s)$, let $[\underline{W}_h(s),\overline{W}_h(s)]$ be the output of \visitalg$(h,s;\varepsilon_0,\delta/(SH))$ and define \[\widehat{\cX} = \left\{(h,s) : \underline{W}_h(s) \geq \frac{\varepsilon_0}{8} \right\}.\] 
With probability $1-\delta$, the following holds: 
\begin{itemize}
 \item For all $(h,s)$, $W_h(s) \in \left[\underline{W}_h(s),\overline{W}_h(s)\right]$
 \item $\left\{(h,s) : W_h(s) \geq \epsilon_0\right\} \subseteq \widehat{\cX} \subseteq \left\{(h,s) : W_h(s) \geq \frac{\varepsilon_0}{8}\right\}$
 \item For all $(h,s) \in \widehat{\cX}$, \ $\overline{W}_h(s) \leq 36 W_h(s)$. 
\end{itemize}
Moreover, the (deterministic) sample complexity necessary to construct $\widehat{\cX}$ is
\[T_{\varepsilon_0}(\delta) := SH \times \inf\left\{T \in \N^\star : T \in \N : 4\cR^{\Pi}_{\delta/(2SH)}(T) + 6\log\left(\frac{4}{\delta}\right) \leq \frac{\varepsilon_0}{4} T\right\}.\]
In particular, using UCBVI as the regret minimizer, we have $T_{\varepsilon_0}(\delta) = \widetilde{\cO}\left(\frac{S^2AH^2\left(\log\left(\frac{SAH}{\delta}\right) + S\right)}{\varepsilon_0}\right)$.
\end{theorem}

\begin{proof} Let $T=T(\varepsilon_0,\delta)$ be the (deterministic) number of episodes of \visitalg($(h,s);\varepsilon_0,\delta$), which satisfies 
\begin{equation}4\cR^{\Pi}_{\delta/2}(T) + 6\log\left(\frac{4}{\delta}\right) \leq \alpha\varepsilon_0 T \ \ \ \text{ for } \ \ \ \alpha := \frac{1}{4}\label{eq:cdt-T}.\end{equation}
The analysis relies on the first-order bound on the regret of $\cA^{\Pi}$ assumed in \eqref{eq:regret-not-anytime} and on a tight control of the martingale 
\[M_T = \sum_{t=1}^{T} \left[\ind(s_h^{t} = s) - p_h^{\pi_{t}}(s)\right]\]
where $p_h^{\pi}(s) = p_h^{\pi}(s,\pi(s))$ is the probability to reach $s$ under policy $\pi$. Observing that the increment of this martingale is bounded in $[-1,1]$ and that its variance is upper bounded by $W_h(s)$, we can use Bernstein's inequality to get that 
\[\bP\left(|M_T| \leq \sqrt{2TW_h(s)\log\left(\frac{4}{\delta}\right)} + \frac{2}{3}\log\left(\frac{4}{\delta}\right)\right) \geq 1 - \frac{\delta}{2}.\]
Remarking that the regret of $\cA^{\Pi}$ for the reward function $r=r^{(h,s)}$ can be written 
\[ \sum_{t=1}^{T} V_1^{\star}\left(s_1 ; r\right) - \sum_{t=1}^{T}V_1^{\pi^{t}}\left(s_1 ; r\right) = TW_h(s) - \sum_{t=1}^{T}p_h^{\pi^{t}}(s) = TW_h(s) - n_h^{T}(s) + M_T\]
and that $n_h^{T}(s) \leq TW_h(s) + M_T$, we obtain that with probability larger than $1-\delta$, the following two inequalities hold: {\small
\begin{eqnarray*}
 n_h^{T}(s) &\geq& T W_h(s) - \left[\sqrt{\cR_{\delta/2}(T) T W_h(s)} + \cR_{\delta/2}(T) + \sqrt{2\log\left(\frac{4}{\delta}\right) T W_h(s)} + \frac{2}{3}\log\left(\frac{4}{\delta}\right)\right] \\
TW_h(s) & \geq & n_h^{T}(s) - \left[\sqrt{2\log\left(\frac{4}{\delta}\right) T W_h(s)} + \frac{2}{3}\log\left(\frac{4}{\delta}\right)\right]
 \end{eqnarray*}
}
Using the AM-GM inequality above, this first yields
\begin{align*}
    n_h^T(s)/2 - g(\delta) \leq T W_h(s) \leq 2 n_h^T(s) + f(T,\delta),
\end{align*}
where $f(T,\delta) := 4\cR_{\delta/2}(T) + \frac{16}{3}\log\left(\frac{4}{\delta}\right)$ and $g(\delta) := \frac{7}{6}\log\left(\frac{4}{\delta}\right)$. Observing that $g(\delta) \leq \tfrac{1}{4}f(T,\delta)$ and $f(T,\delta) \leq \alpha\varepsilon_0T$ by inequality \eqref{eq:cdt-T}, we get 
\[ \frac{n_h^T(s)}{2T} - \frac{\alpha\varepsilon_0}{4} \leq W_h(s) \leq \frac{2n_h^T(s)}{T} + \alpha\epsilon_0,
\]
which also implies
\begin{align*}
    \frac{W_h(s)}{2} - \frac{\alpha\epsilon_0}{2} \leq \frac{n_h^T(s)}{T} \leq 2W_h(s) + \frac{\alpha\epsilon_0}{2}.
\end{align*}
As the output of \visitalg($(h,s);\varepsilon_0,\delta$) can be written \[\left[\underline{W}_h(s)=\left(\frac{n_h^T(s)}{2T} - \frac{\alpha\varepsilon_0}{4}\right) \vee 0,\overline{W}_h(s) =  \left(\frac{2n_h^T(s)}{T} + \alpha\epsilon_0\right) \wedge 1\right]\] and we get that with probability larger than $1-\delta$:
\begin{enumerate}
    \item For any value of $W_h(s)$,
    \[ \frac{W_h(s)}{4} - \frac{\alpha\epsilon_0}{2} \leq \underline{W}_h(s) \leq W_h(s) \leq \overline{W}_h(s) \leq 4W_h(s) + 2\alpha\epsilon_0.\]
    \item If $W_h(s) \geq \varepsilon_0$, then $W_h(s) \in [\underline{W}_h(s),\overline{W}_h(s)] \in [\frac{1-2\alpha}{4}{W}_h(s), \left(4+2\alpha\right) W_h(s)]$.
    \item If $W_h(s) < \varepsilon_0$, then $W_h(s) \in [\underline{W}_h(s),\overline{W}_h(s)] \in \left[0 , \left(4+2\alpha\right) \epsilon_0\right]$.
\end{enumerate} 
Now if $[\underline{W}_h(s),\overline{W}_h(s)]$ is the output of \visitalg($(h,s);\varepsilon,\delta/SH)$ and \[\widehat{\cX} = \left\{ (h,s) : \underline{W}_h(s) \geq \frac{1-2\alpha}{4}\epsilon_0\right\}\] we deduce that, with probability $1-\delta$:
\begin{itemize}
    \item $(h,s)$ with $W_h(s) \geq \varepsilon_0$ are all in $\widehat{\cX}$.
    \item Since $\underline{W}_h(s) \leq W_h(s)$, any $(h,s)$ with $W_h(s) < \frac{1-2\alpha}{4}\epsilon_0$ does not belong to $\widehat{\cX}$.
\end{itemize}
This proves that $\{ (h,s) : W_h(s) \geq \varepsilon_0 \} \subseteq \hat{\cX} \subseteq \{ (h,s) : W_h(s) \geq \frac{1-2\alpha}{4}\epsilon_0 \}$.
To prove the last statement we remark that for $(h,s) \in \widehat{\cX}$, if $W_h(s)\geq \varepsilon_0$, we have by 2. that 
$\overline{W}_h(s) \leq \left(4 + 2\alpha\right)W_h(s)$ while if $W_h(s) \in \left[\frac{1-2\alpha}{4}\epsilon_0, \varepsilon_0\right)$ we have by 3. that 
\[\overline{W}_h(s) \leq \left(4 + 2\alpha\right)\varepsilon_0 \leq 4\frac{4 + 2\alpha}{1-2\alpha}W_h(s)\]
Plugging the value $\alpha=1/4$ yields $\overline{W}_h(s) \leq 36W_h(s)$ in both cases. 

To get an upper bound on the number of episodes used by an instance of \visitalg, we need to find a  $T$ that satisfies 
\begin{equation}T-1 \leq \frac{16}{\varepsilon_0}\cR^{\Pi}_{\delta/(2SH)}(T) + \frac{24}{\varepsilon_0}\log\left(\frac{SAH}{\delta}\right).\label{eq:final_crude_bound}\end{equation}
For UCBVI, Theorem \ref{th:ucbvi-cr} yields $\cR_{\delta}(T) = 256^2 SAH\left(\log\left(\frac{2SAH}{\delta}\right)+6S\right)\log^2(T+1)$. Using the bound $\log^{2}(x) \leq 4 \sqrt{x}$ we get a first crude upper bound on $T$ by solving a quadratic equation which gives the final scaling by plugging back this crude bound in \eqref{eq:final_crude_bound}. 
\end{proof}